\documentclass[12pt]{article}
\usepackage{geometry}
\usepackage{amsmath}
\numberwithin{equation}{section}
\usepackage{amsthm}
\usepackage{amssymb}
\usepackage{graphicx} 
\usepackage{mlmath} 
\usepackage{authblk}
\usepackage{hyperref}

\newtheorem{innercustomgeneric}{\customgenericname}
\providecommand{\customgenericname}{}
\newcommand{\newcustomtheorem}[2]{%
  \newenvironment{#1}[1]
  {%
   \renewcommand\customgenericname{#2}%
   \renewcommand\theinnercustomgeneric{##1}%
   \innercustomgeneric
  }
  {\endinnercustomgeneric}
}

\providecommand{\keywords}[1]
{
  \small	
  \textbf{\textit{Keywords---}} #1
} 
\newtheorem{lemma}{Lemma}
\newtheorem{definition}{Definition}
\newtheorem{remark}{Remark}
\newtheorem{assumption}{Assumption}

\newtheorem{proposition}{Proposition}
\newtheorem{theorem}{Theorem}
\newtheorem{corollary}{Corollary}
\newcustomtheorem{customthm}{Theorem}
\begin{document}
\numberwithin{equation}{section} 
\title{Demystifying Lazy Training of Neural Networks from a Macroscopic Viewpoint}
\author{  
Yuqing Li\textsuperscript{\rm 1,2} \thanks{Corresponding author: liyuqing\underline{~}551@sjtu.edu.cn.},
Tao Luo\textsuperscript{\rm 1,2,3,4}, 
Qixuan Zhou\textsuperscript{\rm 1,2}\\
\textsuperscript{\rm 1}  School of Mathematical Sciences, Shanghai Jiao Tong University \\
\textsuperscript{\rm 2}  CMA-Shanghai, Shanghai Jiao Tong University\\
\textsuperscript{\rm 3} Institute of Natural Sciences, MOE-LSC, Shanghai Jiao Tong University \\
\textsuperscript{\rm 4}   Shanghai Artificial Intelligence Laboratory

\{liyuqing\underline{~}551, 
zhouqixuan, luotao41\}@sjtu.edu.cn.
}
\date{\today}
\maketitle
\begin{abstract}


In this paper, we advance the understanding of neural network training dynamics by examining the intricate interplay of various factors introduced by weight parameters in the initialization process.  Motivated by  the foundational work of Luo et al.~(J. Mach. Learn. Res., Vol. 22, Iss. 1,   No. 71, pp 3327-3373), we explore the gradient descent dynamics of  neural networks through the lens of macroscopic limits, where  we analyze its  behavior as  width   $m$ tends to infinity.    Our study presents a unified approach with refined techniques   designed for multi-layer fully connected neural networks, which can be readily extended to other neural network architectures. Our investigation reveals that  gradient descent  can rapidly drive deep neural networks to zero training loss, irrespective of the specific initialization schemes employed by weight parameters, provided that the initial scale of the output function $\kappa$ surpasses a certain threshold. This regime, characterized as  the theta-lazy area, accentuates the predominant influence of the initial scale $\kappa$ over other   factors on   the training behavior of  neural networks. Furthermore,  our approach   draws inspiration from the Neural Tangent Kernel (NTK) paradigm, and  we expand its applicability. While   NTK   typically assumes that $\lim_{m\to\infty}\frac{\log \kappa}{\log m}=\frac{1}{2}$, and   imposes   each   weight parameters to scale by the factor $\frac{1}{\sqrt{m}}$,  in our theta-lazy regime, we discard the factor   and relax the  conditions to $\lim_{m\to\infty}\frac{\log \kappa}{\log m}>0$. Similar to   NTK,  the behavior of overparameterized neural networks within the theta-lazy regime trained by gradient descent can be effectively described by a specific kernel. Through rigorous analysis, our investigation illuminates the pivotal role of $\kappa$ in governing the training dynamics of neural networks.


\end{abstract}
\allowdisplaybreaks

\section{Introduction}
One intriguing observation  in deep learning  pertains to    the influence of initialization scales on the dynamical behaviors exhibited by Neural Networks (NNs)~\cite{zhang2021rethink}. 
Within a specific regime, NNs trained with gradient descent can be  interpreted as a kernel regression predictor, termed   the Neural Tangent Kernel (NTK)~\cite{Jacot2018Neural,gu2020generalized,Du2018Gradient}, moreover,  Chizat et al.~\cite{chizat2019lazy} identify this regime as the lazy training regime, in which the parameters of  NNs hardly vary.
However, alternative scaling regimes engender highly nonlinear characteristics in NNs trained via gradient descent,  prompting the establishment of    mean-field analysis techniques~\cite{mei2018mean,rotskoff2018parameters,chizat2018global,sirignano2020mean}    as a means to explore the behavior of infinitely wide two-layer networks under such initialization scales.     Additionally, small initialization has been empirically demonstrated to induce a phenomenon known as condensation~\cite{maennel2018gradient,luo2021phase},   wherein the weight vectors of NNs concentrate on isolated orientations throughout the training process. 
This phenomenon is   significant as NNs with condensed weight vectors effectively resemble ``smaller" NNs with reduced parameterization, thus diminishing the complexity of the output functions they represent. Since generalization error can be bounded in terms of network complexity \cite{bartlett2002rademacher}, NNs featuring condensed parameters tend to exhibit superior generalization abilities. 
In light of these observations, the identification of distinct initialization regimes for NNs represents a crucial step towards unraveling the underlying mechanisms governing the training dynamics of neural networks.

Philosophically, our approach advocated here   bears a lot of similarity to that of  molecular dynamics~\cite{frenkel2023understanding}.  In molecular dynamics, the system is represented by a collection of discrete particles, each embodying certain properties or characteristics of the system. By simulating the behavior of these particles over time, the molecular dynamics  endeavors to capture the collective behavior and dynamical intricacies inherent in the system. 
Similarly,   our unified approach has to deal with the myriad parameters  in   NNs,  where several challenges stem directly from the system,  including the interactions between weight parameters across different layers,  and the intricate dependencies between  weight parameters and the output functions. Our goal is  to   capture the lazy training phenomena occurring at various scales simultaneously. Specifically,  as detailed in Section \ref{sub-theta}, our  approach focus on the statistical properties of the `particles', such as the relative distance  between weight parameters across individual layers, and the   initial scale of the output function $\kappa$,  to  scrutinize the macroscopic behavior of parameters.

Moreover,   we draw parallels from other fields to enhance our understanding of the training dynamics. 
For instance, within the realm of continuum mechanics,   the Cauchy-Born rule~\cite{ericksen2008cauchy} is derived  assuming that materials can be described as continuous media.  Hence, the Cauchy-Born rule serves as an example of the macroscopic limit, relating the macroscopic deformation of a material to the underlying microscopic arrangement of atoms in a crystal lattice. 
Analogously, from the perspective of continuum mechanics, our investigation sheds light on  the macroscopic limit of the output function of NNs  by considering the width $m$ approaching infinity. 
In kinetic theory,     under the Boltzmann-Grad scaling  $N\eps^{2}\equiv 1$,  whereas the number
of particles $N$ goes to infinity and the characteristic length  of interaction  $\eps$ simultaneously goes to zero,  the Boltzmann equation  can  be rigorously derived as the mesoscopic limit of
systems of three-dimensional hard spheres~\cite{gallagher2013newton}.
In a parallel manner,    Mei et al.~\cite{mei2018mean}  derived 
the mean-field limit of two-layer NNs to analyze the behavior of NNs  trained under stochastic gradient descent in the limit of infinitely many neurons,  with its initial scale satisfying $\lim_{m\to\infty}\frac{\log \kappa}{\log m}=0$. Essentially,  the mean-field description  approximates the evolution of the   weight parameters   by an evolution in the space of probability distributions, and this evolution can be defined through a partial differential equation. 
We conjecture that regardless of the initial scale $\kappa$, the evolution of the weight parameters can always be approximated by an evolution in the space of probability distributions as $m$ tends to infinity. However,  due to technical constraints, achieving tractable mathematical descriptions of such evolution is  feasible when $\lim_{m\to\infty}\frac{\log \kappa}{\log m}=0$, and discussions on these techniques are beyond the scope of this paper.
Furthermore, we emphasize that the mean-field scaling serves as a critical regime in the phase diagram of NNs. It is noteworthy that the $\vtheta$-lazy regime is sub-critical, in that it resides in the entire half-plane of the regime separation line $\lim_{m\to\infty}\frac{\log \kappa}{\log m}=0$.


Technically speaking, by dissecting the initial scale $\kappa$ into its constituent elements determined by the initialization scheme applied to the weight parameters of each individual layer, we uncover  the mechanisms underpinning the   lazy training phenomenon. This analytical process enables us to disentangle the contributions of each initialization factor to the overall training dynamics, providing clarity on how each element influences the training behavior of neural networks.   Specifically, the  condition  $\lim_{m\to\infty}\frac{\log \kappa}{\log m}>0,~\kappa >1,$~(where $m$ represents the width) serves  as a key ingredient in our analysis,  which enables us to interpret the dynamics of the weight parameters belonging to each individual layer as a kernel regression predictor. It is worth noting that  our techniques draw upon  prior research~\cite{ma2020comparative}, and  we acknowledge that our approach is also  inspired by  the NTK.     While the NTK   conventionally assumes that $\lim_{m\to\infty}\frac{\log \kappa}{\log m}=\frac{1}{2}$, and it involves scaling the   weight parameters  by a factor $\frac{1}{\sqrt{m}}$,  our approach extends its applicability by  discarding the 
 factor $\frac{1}{\sqrt{m}}$, and relaxing the  condition to $\lim_{m\to\infty}\frac{\log \kappa}{\log m}>0$. Moreover, we propose that this formulation can be readily extended to explore training dynamics across various NN architectures. We postulate that the initial scale $\kappa$ also plays a pivotal role in governing the persistence of weight parameters in $L$-layer Convolutional Neural Networks~(CNNs) during training.
In summary, through this refined  analysis, we offer a more comprehensive understanding of the intricate interplay between initialization strategies and training dynamics in deep learning models.


This paper is the fourth  paper in our series of works on the phase diagram of NNs. Our first paper~\cite{luo2021phase}  established the phase diagram   for the two-layer ReLU
neural network at the infinite-width limit, 
thus providing a comprehensive characterization of its dynamical regimes and  its dependence on the hyperparameters related to initialization.
Within this phase diagram, we identified three distinct regimes: the linear regime, critical regime, and condensed regime. 
Building upon this groundwork, our subsequent study~\cite{chen2023phase} elucidated the phase diagram of initial condensation for two-layer neural networks equipped  with a wide class of smooth activation functions. Herein, we  reveal the mechanism of initial condensation for two-layer NNs, and  we identify the directions towards which the weight parameters condense. 
Furthermore, we initiated our exploration into the  realm of  multi-layer NNs by empirically presenting the phase diagram for three-layer ReLU NNs with infinite width~\cite{zhou2022empirical}. 
Finally, 
this paper is motivated by a series of recent articles~\cite{Jacot2018Neural,gu2020generalized,gu2020gradient,Du2019Gradient,Huang2020NTH,Yuqing2022ResNet} where it is shown that overparameterized multi-layer NNs under the NTK scaling   converge linearly to zero training loss with their parameters hardly varying. 
We contend that this behavior is not peculiar to the NTK scaling, and such behavior is predominantly influenced by the scale of the output function of  NNs at initialization, rather than some specific choices of initialization schemes.
To substantiate this assertion, we introduce our approach in Section \ref{Section...Main-Results}, which illustrates  that virtually any NN model can be trained in   the  lazy regime, provided that   the initial scale of its output $\kappa$ is sufficiently large~($\lim_{m\to\infty}\frac{\log \kappa}{\log m}>0$, $m$ is the width). This finding underscores the feasibility of  fast training for NNs, although at the cost of recovering a linear method.    In a subsequent paper in the series, we will extend this 
approach to the regime of  small initialization, where the phenomenon of condensation can be observed. We aim to exploit this methodology to   identify  the underlying mechanism by which different choices of initialization schemes give rise to distinct dynamical behaviors in NNs.  

The organization of the paper is listed  as follows. In Section~\ref{Section...Related-Works}, we  discuss some related works. In Section~\ref{Section...Preliminaries},  we give  some preliminary introduction to our problems.  
 In Section~\ref{Section...Main-Results}, we state some proof techniques and   give out the outline of proofs for our main results, and conclusions are drawn in Section~\ref{section....Conclusion}. All the details of the proof are deferred to the Appendix.

\section{Related Works}\label{Section...Related-Works}

There has been a rich literature on   the choice of initialization schemes    to
facilitate neural network training~\cite{glorot2010understanding,He2016Deep,mei2018mean,sirignano2020mean},  while most of the work identified   width $m$ as a hyperparameter, where the lazy regime is reached when the width grows towards infinity~\cite{Jacot2018Neural,Du2019Gradient},   Chizat et al.~\cite{chizat2019lazy} advocate for the consideration of the initialization scale as the primary hyperparameter of interest,  rather than the width parameter $m$. Subsequent investigations by    Woodworth et al.~\cite{woodworth2020kernel}   focus on
the role of initialization scale as a pivotal determinant governing the transition  between two  different
regimes, namely the kernel regime and the rich regime, within the context of  matrix factorization problems. Additionally, Williams et al.~\cite{williams2019gradient}  studied the implicit bias of gradient descent in the approximation of univariate functions
using single-hidden layer ReLU networks. Their findings   highlights the importance of judiciously selecting initialization strategies to effectively guide the training process of NNs.   
Furthermore, Mehta et al.~\cite{mehta2021extreme}  conducted an in-depth investigation into the effects of   initialization scales on the generalization performance of NNs trained using Stochastic Gradient Descent~(SGD). Their study elucidates how augmenting the initialization scale can detrimentally affect the generalization capabilities of NNs, 
 emphasizing the intricate balance required in initializing neural network parameters to achieve better generalization performance.
In summary, the selection of appropriate initialization scales emerges as a critical factor in sculpting the training dynamics and generalization performance of NNs, thereby underscoring its significance in the domain of neural network research and application.

\section{Preliminaries}\label{Section...Preliminaries}
\subsection{Notations}\label{Subsection...Notations} 
In this work, $n$ is  the number of input samples, and $m$ is the width of the neural network. 
The set $[n]=\{1,2, \ldots, n\}$ is introduced, and the standard Big-O and Big-Omega notations are respectively represented by $\fO(\cdot)$ and $\Omega(\cdot)$.    The notation $\fN(\vmu, \Sigma)$ specifies the normal distribution characterized by mean $\vmu$ and covariance $\Sigma$.
The norms are defined as follows: the vector $L^2$ norm by $\Norm{\cdot}_2$, the vector or function $L^{\infty}$ norm by $\Norm{\cdot}_{\infty}$, the matrix operator norm by $\Norm{\cdot}_{2\to 2}$, and the matrix Frobenius norm by $\Norm{\cdot}_{\mathrm{F}}$. For any matrix $\mA$, its smallest eigenvalue is denoted by $\lambda_{\min}(\mA)$. The tensor product of two vectors is represented by $\otimes$, and the Hadamard product of two matrices by $\odot$.
Additionally, the set of analytic functions $f(\cdot):\sR\to\sR$ is denoted by $\fC^{\omega}(\sR)$, and the standard inner product is indicated by $\left<\cdot,\cdot\right>$.  Finally,  we  say that an event holds
with high probability~(see \cite{Huang2020NTH}), if it holds with probability at least $1-\exp\left(- m^{\epsilon}\right)$ for some  $\epsilon>0$.  This terminology and notation  set the foundation for the subsequent discussion and analysis within the paper.
\subsection{Problem Setup}\label{Subsection...Problem-Setup}
We consider a NN with $L$ hidden layers, where for any $l\in[L]$,
\begin{equation}\label{eq...text...Prelim...L-Layer-NN}
\begin{aligned} 
\vx^{[l]}&=  \sigma(\mW^{[l]}\vx^{[l-1]}),  
\end{aligned} 
\end{equation}
where we identify $\vx^{[0]}:=\vx\in\sR^d$,  $\mW^{[1]}\in \sR^{m\times d}$  and $\mW^{[l]}\in \sR^{m\times m}$  are the weight matrices, and  $\sigma(\cdot)$ is the activation function.   The output of the   $L$-layer NN reads
\begin{equation}\label{eq...text...Prelim...NN-Final-Output} 
      f_{\vtheta}(\vx):=f (\vx,\vtheta) =\va^\T \vx^{[L]},
\end{equation}
where $\va\in \sR^m$.
We denote the vector containing all parameters by 
\begin{equation}\label{eq...text...Prelim...NN-Parameter} 
\vtheta:= 
 \mathrm{vec} \left(\va,  \mW^{[1]} ,  \mW^{[2]} ,\dots,\mW^{[L]}\right),
\end{equation}
and  for any $l\in[L]$, we identify  
\begin{equation}\label{vec}
{\vtheta}_{\mW^{[l]}}:=\mathrm{vec} \left( {{\mW}^{[l]}}\right),~~{\vtheta}_{\va}:={{\va}},
\end{equation}
and    $\sigma\left(\vW^{[l]}\vx^{[l-1]}\right)$ as $\sigma_{[l]}(\vx)$, and the diagonal matrix generated by the  $r$-th derivatives of $\sigma(\cdot)$   applied coordinate-wisely to $\mW^{[l]}\vx^{[l-1]}$., i.e., $\mathrm{diag} \left(\sigma^{(r)}(\mW^{[l]}\vx^{[l-1]})\right),$ by $\vsigma^{(r)}_{[l]}(\vx),$ where $r\geq 1,$ 
and the empirical risk reads
\begin{equation}\label{eq...text...Prelim...Empirical-Loss}
    R_{\fS}(\vtheta)=\frac{1}{2n}\sum_{i=1}^n {(f_{\vtheta}(\vx_i)-y_i)}^2.
\end{equation}
As  we denote hereafter that  for all $i\in[n]$, 
\[
e_{i}  :=e_i(\vtheta  ) := \vf_{\vtheta }(\vx_i ) - y_i,  
\]
then  the empirical risk also reads 
$R_{\fS}(\vtheta)=\frac{1}{2n}\sum_{i=1}^n e_i^2.$ 
To write out the training dynamics based on gradient descent~(GD) at the continuous limit, for any $l\in[L-1]$,  we   define  
\begin{equation}\label{eq...text...Prelim...Special-Matrices}
    \mE^{[l]}(\vx):= \vsigma^{(1)}_{[l]}(\vx)\left(\mW^{[l+1]}\right)^\T,
\end{equation}
then the training dynamics read: For any time $t\geq 0$,
\begin{equation}\label{eqgroup...text...Prelim...GD-Dynamics} 
\left\{ 
\begin{aligned} 
\frac{\mathrm{d} \boldsymbol{W}^{[l]}}{\mathrm{d} t}  &=-\frac{\mathrm{d} R_{\fS}(\boldsymbol{\theta})}{\mathrm{d} \boldsymbol{W}^{[l]}} =-\frac{1}{n} \sum_{i=1}^n e_i\prod_{k=l}^{L-1}\mE^{[k]}(\vx_i)
 \vsigma_{[L]}^{(1)}\left(\vx_i\right)\va   \otimes  \vx_i^{[l-1]},~~l\in[L],\\
\frac{\mathrm{d} \va }{\mathrm{d} t}  &=-\frac{\mathrm{d} R_{\fS}(\boldsymbol{\theta})}{\mathrm{d} \va} =-\frac{1}{n} \sum_{i=1}^{n}e_i \vx_{i}^{[L]}.
\end{aligned}
\right.  
\end{equation}
We initialize the parameters $\vtheta$  following: For any   $l\in[L]$,
\begin{equation}\label{eq...text...Prelim...Initialization-Scheme}
    \mW^{[l]}_{i,j}(0)\sim \fN(\vzero, \beta_l^2), \quad\va_k(0)\sim \fN(0, \beta_{L+1}^2),
\end{equation}
where $\{\beta_l\}_{l=1}^{L+1}$ are positive scaling factors, and  the parameters can be normalized into  
\begin{equation}\label{eq...text...Prelim...Normalization-on-Parameters}
     \frac{\overline{\mW}^{[l]}}{\sqrt{m}}=\frac{1}{\sqrt{m}\beta_l} {\mW}^{[l]}, \quad\frac{\Bar{\va}}{\sqrt{m}} =\frac{1}{\sqrt{m}\beta_{L+1}}\va.
\end{equation}
Throughout this paper, we   refer to  all the  bar-parameters   as the \textbf{normalized} parameters, and we denote further that the vector containing all normalized  parameters by 
\begin{equation}\label{eq...text...Vectorized-Normalized-Parameters}
\Bar{\vtheta}:= \mathrm{vec} \left(\frac{\bar{\va}}{\sqrt{m}}, \frac{\overline{\mW}^{[1]}}{\sqrt{m}}, \frac{\overline{\mW}^{[2]}}{\sqrt{m}}, \dots,\frac{\overline{\mW}^{[L]}}{\sqrt{m}}
 \right),
\end{equation}
Finally,    the loss dynamics  of the empirical risk reads
\begin{equation*}
\begin{aligned}
\frac{\D}{\D t} R_\fS(\vtheta)& =\frac{1}{2n}\sum_{i=1}^n\frac{\D e_i^2}{\D t}=\frac{1}{ n}\sum_{i=1}^ne_i\frac{\D e_i }{\D t}\\
    &=-\frac{1}{ n}\sum_{i,j=1}^ne_i\left< \nabla_{\vtheta} f_{\vtheta}(\vx_i),  \nabla_{\vtheta} f_{\vtheta}(\vx_j)  \right>e_j\\
    &\leq -\frac{2 }{n}\lambda_{\min}\left(\mG(\vtheta)\right)R_\fS(\vtheta),
\end{aligned}
\end{equation*}
where $\mG(\vtheta)$ is the Gram matrix, and we observe that   decay rate of  the empirical risk is determined by the least eigenvalue of $\mG(\vtheta)$, whose components read 
\[
\mG(\vtheta):=[\mG_{ij}]_{n\times n}:=[\mG(\vx_i, \vx_j)]_{n\times n}:=[\left< \nabla_{\vtheta} f_{\vtheta}(\vx_i),  \nabla_{\vtheta} f_{\vtheta}(\vx_j)  \right>]_{n\times n}.
\]
We remark that 
$\mG(\vtheta):=\sum_{l=1}^{L+1} \mG^{[l]} (\vtheta),$
whereas for $l=L+1$,
\begin{equation} \label{eq...text...Prelim...Gram-Order-L+1}
\begin{aligned}
 \mG^{[L+1]}   (\vtheta)&:=\left[\mG_{ij}^{[L+1]}\right]_{n\times n}:=\left[\mG^{[L+1]}(\vx_i, \vx_j)\right]_{n\times n}\\
 &:=[\left< \nabla_{\va} f_{\vtheta}(\vx_i),  \nabla_{\va} f_{\vtheta}(\vx_j)  \right>]_{n\times n} =\left[\left< \vx_{i}^{[L]},  \vx_{j}^{[L]}  \right>\right]_{n\times n},
 \end{aligned}
\end{equation}
and for $l=L$, 
\begin{equation} \label{eq...text...Prelim...Gram-Order-L}
\begin{aligned}
\mG^{[L]}   (\vtheta)&:=\left[\mG_{ij}^{[L]}\right]_{n\times n}:=\left[\mG^{[L]}(\vx_i, \vx_j)\right]_{n\times n}\\
&:=[\left< \nabla_{\mW^{[L]}} f_{\vtheta}(\vx_i),  \nabla_{\mW^{[L]}} f_{\vtheta}(\vx_j)  \right>]_{n\times n} \\
&=\left[
\left< 
 \vsigma_{[L]}^{(1)}\left(\vx_i\right)\va ,  
 \vsigma_{[L]}^{(1)}\left(\vx_j\right)\va \right>\left< \vx_{i}^{[L-1]},  \vx_{j}^{[L-1]}  \right>\right]_{n\times n},
 \end{aligned}
\end{equation}
and finally,   for $l\in[L-1]$,
\begin{equation} \label{eq...text...Prelim...Gram-Order-l<L}
\begin{aligned}
 \mG^{[l]}   (\vtheta)&:=\left[\mG_{ij}^{[l]}\right]_{n\times n}:=\left[\mG^{[l]}(\vx_i, \vx_j)\right]_{n\times n}\\
 &:=[\left< \nabla_{\mW^{[l]}} f_{\vtheta}(\vx_i),  \nabla_{\mW^{[l]}} f_{\vtheta}(\vx_j)  \right>]_{n\times n} \\
 &~=\left[
  \left< 
\prod_{k=l}^{L-1}\mE^{[k]}(\vx_i)
 \vsigma_{[L]}^{(1)}\left(\vx_i\right)\va ,  
\prod_{k=l}^{L-1}\mE^{[k]}(\vx_j)
 \vsigma_{[L]}^{(1)}\left(\vx_j\right)\va \right>\left< \vx_{i}^{[l-1]},  \vx_{j}^{[l-1]}  \right>\right]_{n\times n}.
 \end{aligned}
\end{equation}
\subsection{Activation Functions and Input Samples}\label{Subsection...Activation-Functions-Input-Samples}
We shall impose some   technical conditions  on activation, samples and scaling factors.
\begin{assumption}\label{Assumption....Activation-Function}
We assume that the activation function $\sigma\in\fC^{\omega}(\sR)$ and is not a polynomial function, and  its function value at $0$ satisfy  
     ${\sigma(0)}=0.$
Moreover,  there exists a universal constant $C>0$,   such that  its first  and second derivatives  satisfy 
\begin{equation}\label{eq...assump...Activation-Function...Bounded-First-Second-Derivative}
{\sigma^{(1)}(0)}=1,\quad\Norm{\sigma^{(1)}(\cdot)}_{\infty}\leq C,\quad\Norm{\sigma^{(2)}(\cdot)}_{\infty}\leq C, 
\end{equation}
and
\begin{equation}\label{eq...assump...Activation-Function...Left-Limit-Right-Limit}
\lim_{z\to-\infty}{\sigma^{(1)}(z)}=a,\quad\lim_{z\to+\infty}{\sigma^{(1)}(z)}=b, 
\end{equation}
where $a\neq b$.
\end{assumption}
\begin{remark}
We remark that Assumption \ref{Assumption....Activation-Function} can be satisfied by by using the scaled SiLU activation: $$\sigma(x)=\frac{2x}{1+\exp(-x)},$$
where $a=0$ and $b=2$.
 
Some other functions also satisfy this assumption, for instance, the modified scaled softplus activation: $$\sigma(x)=2\left(\log(1+\exp(x))-\log 2\right),$$where $a=0$ and $b=2$.
\end{remark}
\begin{assumption}\label{Assumption...Data}
We assume  that  for all $i\in[n]$,  there exists   constant $c>0$, such that   the training inputs and labels  $\fS =\{(\vx_i,y_i)\}_{i=1}^n$  satisfy  
\[\frac{1}{c}\leq\Norm{\vx_{i}}_2, \quad\Abs{y_{i}}\leq c,\] 
and  all training inputs are non-parallel with each other.
\end{assumption}

Assumption \ref{Assumption...Data} guarantees that   the  normalized  Gram matrices defined in Section \ref{Subsection...Normalized-Outputs-and-Gram-Matrices} are strictly positive definite.
\begin{assumption}\label{Assumption...Limit-Existence}
We assume that for all $l\in[L+1]$, the following limit exists 
\begin{equation} \label{eq...assump...Scaling-Limit-Existence}
{\gamma}_l:=\lim_{m\to\infty} -\frac{\log \beta_l}{\log m}.
\end{equation}
\end{assumption}
\begin{remark}
It is noteworthy that in the context of a fully connected layer with $n_{\textrm{in}}$ input units and $n_{\textrm{out}}$ output units, Xavier initialization~\cite{glorot2010understanding} and He initialization~\cite{he2015delving} are commonly employed for initializing weights. Xavier initialization initializes weights using a Gaussian distribution with zero mean and variance $\frac{2}{n_{\textrm{in}} + n_{\textrm{out}}}$, while He initialization utilizes a Gaussian distribution with zero mean and variance $\frac{2}{n_{\textrm{in}}}$.  In both Xavier and He initialization,   choice on the initialization scale is adjusted based on the width $m$.  Therefore, in the context of our investigation into the behavior of overparameterized NNs, the assumption of the existence of ${\gamma_l}_{l=1}^{L+1}$ as $m$ tends to infinity is a natural extension.
\end{remark}
\section{Technique Overview and Main Results}\label{Section...Main-Results}
In this part, we   describe some technical tools and present the sketch of proofs for our theorem. The statement of our theorem can be found in Section \ref{Subsection...Statement-of-Thm}.
Before we proceed, several updated notations and definitions are required.
\subsection{Normalized  Outputs and    Gram Matrices}\label{Subsection...Normalized-Outputs-and-Gram-Matrices}
We start by a $L$-layer \textbf{normalized NN model}. 
\begin{definition}[Normalized NN]\label{Definition...Normalized-NNs}
    Given a $L$-layer NN, then the \textbf{normalized NN} reads:
\begin{equation}\label{eq...Definition...Technique...Normalized-NN} 
\left\{ 
\begin{aligned}
\Bar{\vx}^{[l]}&=\frac{1}{\sqrt{m}}\frac{\sigma\left(\left((\sqrt{m})^{l-1}\prod_{k=1}^l\beta_k\right)\overline{\mW}^{[l]}\Bar{\vx}^{[l-1]}\right)}{(\sqrt{m})^{l-1}\prod_{k=1}^l\beta_k},~~l\in[L], \\ 
\Bar{f}_{\vtheta}(\vx)&=\left(\frac{\bar{\va}}{\sqrt{m}}\right)^\T \Bar{\vx}^{[L]}.
\end{aligned} 
\right.
\end{equation}
\end{definition}

It is noteworthy that  the condition  on the activation function $\sigma(\cdot)$ imposed in  Assumption \ref{Assumption....Activation-Function} is crucial in that   for any $l\in[L]$,  it guarantees  $\Norm{\bar{\vx}^{[l]}}_2\sim \fO(1)$. 

Given  new  scaling factors $\{\alpha_l\}_{l=1}^{L+1}$ with 
\begin{equation}\label{eq...text...Technique...Alpha-and-Beta}
 \alpha_l:=\sqrt{m}\beta_l,    
\end{equation}
and 
\begin{equation} \label{eq...text...Technique...Kappa}
\kappa:=\prod_{k=1}^{L+1}\alpha_k,  \end{equation} 
thus we have  the scaling relations between   $\left\{\left\{{\vx}^{[l]}\right\}_{l=1}^L,{f}_{\vtheta}(\vx)\right\}$ and $\left\{\left\{\Bar{\vx}^{[l]}\right\}_{l=1}^L,\Bar{f}_{\vtheta}(\vx)\right\}$,  
\begin{equation}  \label{eq...text...Technique...Relation-between-Normalized-Outputs-and-Outputs}
\left\{
\begin{aligned}
{\vx}^{[l]}&=\left(\prod_{k=1}^l\alpha_k\right)\Bar{\vx}^{[l]},~~l\in[L], \\
{f}_{\vtheta}(\vx)&=\left(\prod_{k=1}^{L+1}\alpha_k\right)\Bar{f}_{\vtheta}(\vx)=\kappa\Bar{f}_{\vtheta}(\vx).   
\end{aligned}
\right.
\end{equation}
Consequently,  to write out the   normalized  Gram matrices $\left\{\overline{\mG}^{[l]}(\vtheta)\right\}_{l=1}^{L+1}$,   for any $l \in [L-1]$,  we firstly normalize   $\mE^{[l]}(\vx)$ by
\begin{equation}\label{eq...text...Technique...Normalized-Special-Matrices} 
    \overline{\mE}^{[l]}(\vx) :=\vsigma^{(1)}_{[l]}(\vx)\left(\frac{\overline{\mW}^{[l+1]}}{\sqrt{m}}\right)^\T, 
\end{equation}
then we proceed to define the auxiliary matrices $\left\{\overline{\mH}^{[l]}(\vtheta)\right\}_{l=1}^{L}$.
\begin{definition}\label{Definition...Auxiliary-Matrices-H(L)}
    Given sample $\fS =\{(\vx_i,y_i)\}_{i=1}^n$,  $\left\{\overline{\mH}^{[l]}(\vtheta)\right\}_{l=1}^{L}$ are defined as follows:
\begin{equation} \label{eq...definition...Technique...Normalized-H-Matrix}
\left\{
\begin{aligned}
\overline{\mH}_{ij}^{[L]}   (\vtheta)&:=\left< 
 \vsigma_{[L]}^{(1)}\left(\vx_i\right)\frac{\Bar{\va}}{\sqrt{m}},   
 \vsigma_{[L]}^{(1)}\left(\vx_j\right)\frac{\Bar{\va}}{\sqrt{m}}\right>,\\
\overline{\mH}_{ij}^{[l]}   (\vtheta) &:= \left<
 \prod_{k=l}^{L-1}\overline{\mE}^{[k]}(\vx_i)
 \vsigma_{[L]}^{(1)}\left(\vx_i\right)\frac{\Bar{\va}}{\sqrt{m}},  
\prod_{k=l}^{L-1}\overline{\mE}^{[k]}(\vx_j)
 \vsigma_{[L]}^{(1)}\left(\vx_j\right)\frac{\Bar{\va}}{\sqrt{m}}\right>,~~l\in[L-1].
\end{aligned}
\right.
\end{equation}
\end{definition} 
 
Hence on the basis of \eqref{eq...text...Prelim...Gram-Order-L+1}, \eqref{eq...text...Prelim...Gram-Order-L}, and \eqref{eq...text...Prelim...Gram-Order-l<L},   we obtain that  
\begin{definition}[Normalized Gram Matrices]\label{Definition...Normalized-Gram-Matrices-G(L)}
    Given sample $\fS =\{(\vx_i,y_i)\}_{i=1}^n$ and  $\left\{\overline{\mH}^{[l]}(\vtheta)\right\}_{l=1}^{L}$,  the  normalized Gram matrices $\left\{\overline{\mG}^{[l]}(\vtheta)\right\}_{l=1}^{L+1}$ are defined as follows:  
\begin{equation} \label{eq...definition...Technique...Normalized-Gram-Matrices}
\left\{
\begin{aligned}
\overline{\mG}_{ij}^{[L+1]}   (\vtheta) &:= \left< \bar{\vx}_{i}^{[L]},  \bar{\vx}_{j}^{[L]}  \right>,\\
\overline{\mG}_{ij}^{[l]}   (\vtheta) &:=  \overline{\mH}_{ij}^{[l]}   (\vtheta)\left< \bar{\vx}_{i}^{[l-1]},  \bar{\vx}_{j}^{[l-1]}  \right>,~~l\in[L].
\end{aligned}
\right.
\end{equation}
\end{definition} 
 
Most importantly,   the scaling relations between the     Gram matrices $\left\{ {\mG}^{[l]}(\vtheta)\right\}_{l=1}^{L+1}$ and the   normalized  Gram matrices $\left\{\overline{\mG}^{[l]}(\vtheta)\right\}_{l=1}^{L+1}$ read
\begin{equation}\label{eq...text...Scaling-Relation-between-Normalized-and-Unnormalized-Gram-Matrix}
\overline{\mG}^{[l]}   (\vtheta)= \frac{\alpha_{l}^2}{\kappa^2}{\mG}^{[l]}(\vtheta),~~l\in[L+1].
\end{equation}
\subsection{Normalized Limiting Gram Matrices}\label{Subsection...Normalized-Limiting-Gram-Matrices}
As $m\to \infty$, we define  the normalized limiting Gram matrices $\left\{\mK^{[l]}\right\}_{l=1}^{L+1}$ to characterize the limiting behavior of  $\left\{\overline{\mG}^{[l]}(\vtheta(t)))\right\}_{l=1}^{L+1}$ at    time $t=0$, i.e.,  $\left\{\overline{\mG}^{[l]}(\vtheta^0)\right\}_{l=1}^{L+1}$.  

Firstly, we  remark  that definition of the limiting Gram matrix $\mK^{[L+1]}$ depends on the  auxiliary matrices   $\left\{\widetilde{\mK}^{[l]} \right\}_{l=1}^L$ and $\left\{\widetilde{\mA}^{[l]} \right\}_{l=1}^{L}$.
\begin{definition}\label{Definition...Auxiliary-Matrices-K+A}
Given sample $\fS =\{(\vx_i,y_i)\}_{i=1}^n$, $\left\{\widetilde{\mK}^{[l]} \right\}_{l=1}^L$ and $\left\{\widetilde{\mA}^{[l]} \right\}_{l=1}^{L}$ are  recursively defined as follows:
\begin{equation}\label{eqgroup...definition...Bootstramp-Matrices-for-Normalized-L+1-th-Gram-Matrix}
\left\{
\begin{aligned}
\widetilde{\mK}^{[0]}_{ij}&:=\left<\vx_i^{[0]},\vx_j^{[0]}\right>:=\left<\vx_i,\vx_j\right>,\\
\widetilde{\mA}^{[l]}_{ij}&:=\begin{pmatrix}\widetilde{\mK}_{ii}^{[l-1]}&\widetilde{\mK}_{ij}^{[l-1]}\\
\widetilde{\mK}_{ji}^{[l-1]}&\widetilde{\mK}_{jj}^{[l-1]}\end{pmatrix},~~l\in[L],\\
\widetilde{\mK}^{[l]}_{ij}&:= \Exp_{(u,v)^{\T}\sim \fN\left(\vzero, \widetilde{\mA}^{[l]}_{ij}\right)}  \frac{\sigma\left(\left(\frac{1}{\sqrt{m}}\prod_{k=1}^l\alpha_k\right)u\right)}{\frac{1}{\sqrt{m}}\prod_{k=1}^l\alpha_k}\frac{\sigma\left(\left(\frac{1}{\sqrt{m}}\prod_{k=1}^l\alpha_k\right)v\right)}{\frac{1}{\sqrt{m}}\prod_{k=1}^l\alpha_k},~~l\in[L].
\end{aligned}
\right.
\end{equation}
\end{definition}
\noindent
For  definition of the rest of the limiting Gram matrices $\left\{\mK^{[l]}\right\}_{l=1}^{L}$, we define   the  auxiliary matrices $\left\{ {\widetilde{\mI}}^{[l]} \right\}_{l=1}^L$.
\begin{definition}\label{Definition...Auxiliary-Matrices-I}
Given sample $\fS =\{(\vx_i,y_i)\}_{i=1}^n$  and $\left\{\widetilde{\mA}^{[l]} \right\}_{l=1}^{L}$, $\left\{ {\widetilde{\mI}}^{[l]} \right\}_{l=1}^L$ are defined as follows:
\begin{equation}\label{eqgroup...definition...Bootstramp-Matrices-for-Normalized-l<L+1-th-Gram-Matrix}
\begin{aligned} 
\widetilde{\mI}^{[l]}_{ij}&:= \Exp_{(u,v)^{\T}\sim \fN\left(\vzero, \widetilde{\mA}^{[l]}_{ij}\right)}   {\sigma^{(1)}\left(\left(\frac{1}{\sqrt{m}}\prod_{k=1}^l\alpha_k\right)u\right)}  {\sigma^{(1)}\left(\left(\frac{1}{\sqrt{m}}\prod_{k=1}^l\alpha_k\right)v\right)},~~l\in[L].
\end{aligned}
\end{equation}
\end{definition}

Therefore, we obtain that 
\begin{definition}[Normalized Limiting Gram Matrices]\label{Definition...Limiting-Gram-Matrices}
Given sample $\fS =\{(\vx_i,y_i)\}_{i=1}^n$, $\left\{\widetilde{\mK}^{[l]} \right\}_{l=1}^L$,  $\left\{\widetilde{\mA}^{[l]} \right\}_{l=1}^{L}$, and     $\left\{ {\widetilde{\mI}}^{[l]} \right\}_{l=1}^L$, $\left\{ {\mK}^{[l]} \right\}_{l=1}^{L+1}$ are defined as follows:
\begin{equation}\label{eqgroup...definition...Normalized-Limiting-Gram-Matrix}
\left\{
\begin{aligned} 
\mK_{ij}^{[L+1]}&:= \widetilde{\mK}_{ij}^{[L]},\\
  \mK_{ij}^{[l]}&:=\widetilde{\mK}_{ij}^{[l-1]}\prod_{k=l}^L\widetilde{\mI}_{ij}^{[k]},~~l\in[L].
\end{aligned}
\right.
\end{equation}
\end{definition}
\noindent
As for the   positive-definiteness of    $\left\{\mK^{[l]}\right\}_{l=1}^{L+1}$, we have
\begin{proposition}\label{Proposition...Positive-Definiteness-of-Limiting-Gram-Matrix}
Suppose $\sigma(\cdot)$ satisfies conditions  in Assumption \ref{Assumption....Activation-Function}, and $\fS$     satisfies conditions  in  Assumption 
 \ref{Assumption...Data}, then     for any $i\in [n]$ and $l\in[L]$,  there exist some positive  constants $\mu_1, \mu_2>0$, such that 
\begin{equation}\label{eq...Proposition....Estimates-on-Diagonal-Elements}
\left\{
\begin{aligned}
\mu_1^L\leq {\mu_1^l}&\leq \widetilde{\mK}^{[l]}_{ii} \leq \mu_2^l\leq \mu_2^L,\\
\mu_1   &\leq \widetilde{\mI}^{[l]}_{ii} \leq \mu_2.
 \end{aligned} 
 \right.
 \end{equation}
Moreover, as we denote $\lambda_\fS:=\min_{l\in[L+1]}    \lambda_{\min}\left(\mK^{[l]}\right),$ then $\lambda_\fS>0$.
\end{proposition} 
\begin{proof}
We shall prove relation \eqref{eq...Proposition....Estimates-on-Diagonal-Elements} by induction. The case where $l=1$ is exactly relation \eqref{A-Lemma...eq...Second-Moment-Bound} in Lemma \ref{A-Lemma...Second-Moment-Boundedness}.  Then,  assume that  we already have 
\[
\mu_1^L\leq {\mu_1^{l-1}} \leq \widetilde{\mK}^{[l-1]}_{ii} \leq \mu_2^{l-1}\leq \mu_2^L,
\]
and we observe that 
\begin{equation*}
    \widetilde{\mK}^{[l]}_{ii} = \Exp_{u\sim \fN\left(\vzero, \widetilde{\mK}^{[l-1]}_{ii}\right)}  
   \left[ \frac{\sigma\left(\left(\frac{1}{\sqrt{m}}\prod_{k=1}^l\alpha_k\right)u\right)}{\frac{1}{\sqrt{m}}\prod_{k=1}^l\alpha_k}\right]^2,
\end{equation*}
consequently,  as we set $\eps:=\frac{1}{{\sqrt{m}}}\prod_{k=1}^l\alpha_k$, then based on  Lemma \ref{A-Lemma...Second-Moment-Boundedness}, we obtain that 
\begin{equation}\label{112}
{\mu_1^l} \leq 
\mu_1  \widetilde{\mK}_{ii}^{[l-1]}\leq \widetilde{\mK}^{[l]}_{ii} \leq \mu_2 \widetilde{\mK}_{ii}^{[l-1]} 
 \leq \mu_2^l,
\end{equation}
therefore, based on Lemma \ref{A-Lemma...Gram-Matrices-without-Derivative},  for any $l\in[L]$, $\lambda_{\min}\left(\widetilde{\mK}^{[l]}\right)>0$, hence $\lambda_{\min}\left({\mK}^{[L+1]}\right)>0$.  As for $\mI^{[l]}$, since
\begin{align*}
\widetilde{\mI}_{ii}^{[l]}=\Exp_{u\sim \fN\left(\vzero, \widetilde{\mK}_{ii}^{[l-1]}\right)}   \left[{\sigma^{(1)}\left(\left(\frac{1}{\sqrt{m}}\prod_{k=1}^l\alpha_k\right)u\right)}\right]^2,
\end{align*}
as we set $\eps:=\frac{1}{{\sqrt{m}}}\prod_{k=1}^l\alpha_k$, and \eqref{112} guarantees  validity of    relation \eqref{1122} imposed in  Lemma \ref{A-Lemma...Second-Moment-Boundedness}, we finish the proof for relation  \eqref{eq...Proposition....Estimates-on-Diagonal-Elements}. As we recall that for any $l\in[L]$,
\[\mK^{[l]}= \widetilde{\mK}^{[l-1]}\odot \widetilde{\mI}^{[l]}\odot\widetilde{\mI}^{[l+1]}\cdots\odot\widetilde{\mI}^{[L]},\] then based on Lemma \ref{A-Lemma...Hadamard-Product-is-Positive-Definite},
\begin{align*}
 \lambda_{\min}\left(\mK^{[l]}\right)&\geq \left(\left(\frac{n-1}{n}\right)^{\frac{n-1}{2}}\right)^{L-l+1} \left(\prod_{i=1}^n {\widetilde{\mI}_{ii}^{[L]}} \right)\left(\prod_{i=1}^n {\widetilde{\mI}_{ii}^{[L-1]}} \right)\cdots\left(\prod_{i=1}^n {\widetilde{\mI}_{ii}^{[l]}} \right)\mathrm{det}\left(\widetilde{\mK}^{[l-1]}\right)\\
&\geq  \left(\left(\frac{n-1}{n}\right)^{\frac{n-1}{2}} \mu_1^n\right)^{L-l+1}\mathrm{det}\left(\widetilde{\mK}^{[l-1]}\right)>0.
 \end{align*}
\end{proof}
\subsection{Least Eigenvalue of Normalized Gram Matrices at Initial Stage}\label{Subsection...Least-Eigenvalue-of-Normalized-Gram-Matrices}
Our next two propositions serves to demonstrate that  the  normalized Gram matrix $\overline{\mG}^{[l]}(\vtheta^0)$ is close to the normalized limiting Gram matrix $\mK^{[l]}$. For notational simplicity, we define  $\left\{\widetilde{\mB}^{[l]}(\vtheta^0) \right\}_{l=1}^{L-1}$,
\begin{equation}\label{A-text...eqgroup...Bootstramp-Matrices-for-Matrix-B}
\begin{aligned} 
\widetilde{\mB}^{[l]}_{ij}(\vtheta^0)&:=\begin{pmatrix}\left<\Bar{\vx}_i^{[l]}, \Bar{\vx}_i^{[l]}\right> & \left<\Bar{\vx}_i^{[l]}, \Bar{\vx}_j^{[l]}\right>\\
\left<\Bar{\vx}_j^{[l]}, \Bar{\vx}_i^{[l]}\right> & \left<\Bar{\vx}_j^{[l]}, \Bar{\vx}_j^{[l]}\right> \end{pmatrix}. 
\end{aligned}
\end{equation}
\begin{proposition}\label{A-prop...Concentration-on-L+1-th-Gram-Matrix}
Suppose $\sigma(\cdot)$ satisfies conditions  in Assumption \ref{Assumption....Activation-Function}, and $\fS$     satisfies conditions  in  Assumption 
 \ref{Assumption...Data}, then     for any $i, j\in [n]$ and $l\in[L]$, the following holds
\begin{equation}\label{A-Prop...eq...C-L+1-G-M...Concentration-on-l-th-Gram-Matrix}
 \Prob\left(\Abs{\left< \bar{\vx}_{i}^{[l]},  \bar{\vx}_{j}^{[l]}  \right>-\widetilde{\mK}^{[l]}_{ij} }\geq \eta\right)\leq 2\exp\left(-C_0 m \eta^2\right),  
\end{equation}
for some   constant $C_0>0$ depending on $L$.
\end{proposition}
\begin{proof}
We shall prove this by induction. For $l=1$,  the inner product  $\left< \bar{\vx}_{i}^{[1]},  \bar{\vx}_{j}^{[1]}  \right>$ reads
\begin{align*}
 \left< \bar{\vx}_{i}^{[1]},  \bar{\vx}_{j}^{[1]}  \right>&=\frac{1}{m} \sum_{k'=1}^m   \frac{\sigma\left(\frac{\alpha_1}{\sqrt{m}}\vw_{k'}^\T{\vx}_i^{[0]}\right)}{\frac{\alpha_1}{\sqrt{m}}}\frac{\left(\frac{\alpha_1}{\sqrt{m}}\vw_{k'}^\T{\vx}_j^{[0]}\right)}{\frac{\alpha_1}{\sqrt{m}}},
\end{align*}
 since for any $k'\in[m]$,  $\vw_{k'}\sim\fN(\vzero,\mI_d)$, then
\begin{align*}
  \Norm{\frac{\sigma\left(\frac{\alpha_1}{\sqrt{m}}\vw_{k'}^\T{\vx}_i^{[0]}\right)}{\frac{\alpha_1}{\sqrt{m}}}\frac{\left(\frac{\alpha_1}{\sqrt{m}}\vw_{k'}^\T{\vx}_j^{[0]}\right)}{\frac{\alpha_1}{\sqrt{m}}}}_{\psi}&\leq c^2 \Norm{\norm{\vw_{k'}}_2^2}_{\psi}\leq c^2   C_{\psi, d},
\end{align*}
is a sub-exponential random variable, and as we notice that 
\begin{align*}
\Exp_{\vw\sim\fN(\vzero,\mI_d)}\frac{\sigma\left(\beta_1\vw^\T{\vx}_i^{[0]}\right)}{\beta_1}\frac{\left(\beta_1\vw^\T{\vx}_j^{[0]}\right)}{\beta_1}=\Exp_{(u,v)^{\T}\sim \fN\left(\vzero, \widetilde{\mA}^{[l]}_{ij}\right)} \frac{\sigma\left(\frac{\alpha_1}{\sqrt{m}}u\right)}{\frac{\alpha_1}{\sqrt{m}}}
\frac{\sigma\left(\frac{\alpha_1}{\sqrt{m}}v\right)}{\frac{\alpha_1}{\sqrt{m}}},
\end{align*}
hence by application of  Theorem \ref{A-Thm...Bernstein-Inequality}, for some absolute constant $C_0>0$,
\begin{equation}
\Prob\left(\Abs{\left< \bar{\vx}_{i}^{[1]},  \bar{\vx}_{j}^{[1]}  \right>-\widetilde{\mK}^{[1]}_{ij} }\geq \eta\right)\leq 2\exp\left(-\frac{C_0}{c^2C_{\psi, d}^2} m {\eta^2}\right).
\end{equation}

We assume that    relation \eqref{A-Prop...eq...C-L+1-G-M...Concentration-on-l-th-Gram-Matrix} holds for   $l'< l$,  and we proceed to demonstrate that it also  holds true  for $l'=l$.  As the inner product  $\left< \bar{\vx}_{i}^{[l]},  \bar{\vx}_{j}^{[l]}  \right>$ reads
\begin{align*}
 \left< \bar{\vx}_{i}^{[l]},  \bar{\vx}_{j}^{[l]}  \right>&=\frac{1}{m} \sum_{k'=1}^m   \frac{\sigma\left(\left(\frac{1}{\sqrt{m}}\prod_{k=1}^l\alpha_k\right)\vw_{k'}^\T{\vx}_i^{[l-1]}\right)}{\frac{1}{\sqrt{m}}\prod_{k=1}^l\alpha_k}\frac{\left(\left(\frac{1}{\sqrt{m}}\prod_{k=1}^l\alpha_k\right)\vw_{k'}^\T{\vx}_j^{[l-1]}\right)}{\frac{1}{\sqrt{m}}\prod_{k=1}^l\alpha_k},
\end{align*}
as we set $\eps_l:=\frac{1}{\sqrt{m}}\prod_{k=1}^l\alpha_k,$  and since for any $k'\in[m]$,  $\vw_{k'}\sim\fN(\vzero,\mI_m)$, then
\begin{align*}
  \Norm{\frac{\sigma\left(\eps_l\vw_{k'}^\T{\vx}_i^{[l-1]}\right)}{\eps_l}\frac{\sigma\left(\eps_l\vw_{k'}^\T{\vx}_j^{[l-1]}\right)}{\eps_l}}_{\psi}&\leq c^2 \Norm{\norm{\vw_{k'}}_2^2}_{\psi}\leq c^2   C_{\psi, m},
\end{align*}
is also a sub-exponential random variable, and   as we notice that 
\begin{align*}
 \Exp_{\vw\sim\fN(\vzero,\mI_m)} \frac{\sigma\left(\eps_l\vw^\T\Bar{\vx}_i^{[l-1]}\right)}{\eps_l} \frac{\sigma\left(\eps_l\vw^\T\Bar{\vx}_j^{[l-1]}\right)}{\eps_l} 
= \Exp_{(u,v)^{\T}\sim \fN\left(\vzero, \widetilde{\mB}^{[l-1]}_{ij}(\vtheta^0) \right)}  \frac{\sigma\left(  \eps_lu\right)}{\eps_l}
\frac{\sigma\left(  \eps_lv\right)}{\eps_l},
\end{align*}
hence by application of  Theorem \ref{A-Thm...Bernstein-Inequality}, for some absolute constant $C_0>0$,
\begin{equation*}
\Prob\left(\Abs{\left< \bar{\vx}_{i}^{[l]},  \bar{\vx}_{j}^{[l]}  \right>-\Exp_{(u,v)^{\T}\sim \fN\left(\vzero, \widetilde{\mB}^{[l-1]}_{ij}(\vtheta^0) \right)}  \frac{\sigma\left(  \eps_lu\right)}{\eps_l}
\frac{\sigma\left(  \eps_lv\right)}{\eps_l}}\geq \eta\right)\leq 2\exp\left(-\frac{C_0}{c^2C_{\psi, m}^2} m {\eta^2}\right),
\end{equation*}
then by application of Lemma \ref{A-Lemma...Matrix-Norm-and-Entry-Norm}, we obtain that 
\begin{align*}
&\Abs{F\left(\widetilde{\mB}^{[l-1]}_{ij}(\vtheta^0)\right) -F\left(\widetilde{\mA}^{[l]}_{ij}\right)}\leq  C\Norm{\widetilde{\mB}^{[l-1]}_{ij}(\vtheta^0)-\widetilde{\mA}^{[l]}_{ij}}_{\infty}, 
\end{align*}
thus we have 
\begin{align*}
&\Prob\left(\Abs{F\left(\widetilde{\mB}^{[l-1]}_{ij}(\vtheta^0)\right) -F\left(\widetilde{\mA}^{[l]}_{ij}\right)}\geq \eta\right)\\
\leq &\Prob\left(\Abs{\left<\bar{\vx}_{i}^{[l-1]},  \bar{\vx}_{i}^{[l-1]}\right>-\widetilde{\mK}^{[l-1]}_{ii}}\geq  \frac{\eta}{C}\right)+ \Prob\left(\Abs{\left<\bar{\vx}_{i}^{[l-1]},  \bar{\vx}_{j}^{[l-1]}\right>-\widetilde{\mK}^{[l-1]}_{ij}}\geq \frac{\eta}{C}\right)\\
 &+\Prob\left(\Abs{\left<\bar{\vx}_{j}^{[l-1]},  \bar{\vx}_{i}^{[l-1]}\right>-\widetilde{\mK}^{[l-1]}_{ji}}\geq \frac{\eta}{C}\right)+\Prob\left(\Abs{\left<\bar{\vx}_{j}^{[l-1]},  \bar{\vx}_{j}^{[l-1]}\right>-\widetilde{\mK}^{[l-1]}_{jj}}\geq \frac{\eta}{C}\right),
\end{align*}
and from our induction hypothesis,  as the following holds
\[\Prob\left(\Abs{\left<\bar{\vx}_{i}^{[l-1]},  \bar{\vx}_{j}^{[l-1]}\right>-\widetilde{\mK}^{[l-1]}_{ij}}\geq \frac{\eta}{C}\right)\leq 2\exp\left(-\frac{C_0}{C^2} m \eta^2\right),\]
then we obtain that 
\begin{equation}
\begin{aligned}
 &\Prob\left(\Abs{\left<\bar{\vx}_{i}^{[l]},  \bar{\vx}_{j}^{[l]}\right>-\widetilde{\mK}^{[l]}_{ij}}\geq \eta\right)\\ 
 \leq & \Prob\left(\Abs{\left< \bar{\vx}_{i}^{[l]},  \bar{\vx}_{j}^{[l]}  \right>-\Exp_{(u,v)^{\T}\sim \fN\left(\vzero, \widetilde{\mB}^{[l-1]}_{ij}(\vtheta^0) \right)}  \frac{\sigma\left(  \eps_lu\right)}{\eps_l}
\frac{\sigma\left(  \eps_lv\right)}{\eps_l}}\geq \frac{\eta}{2}\right)\\
&+\Prob\left(\Abs{F\left(\widetilde{\mB}^{[l-1]}_{ij}(\vtheta^0)\right) -F\left(\widetilde{\mA}^{[l]}_{ij}\right)}\geq \frac{\eta}{2}\right)\\
\leq &2\exp\left(-C_0 m \frac{\eta^2}{36}\right)+8\exp\left(-\frac{C_0}{C^2} m \frac{\eta^2}{4}\right)=2\exp\left(-C_0 m {\eta^2}\right).
\end{aligned}
\end{equation}
\end{proof}
\begin{proposition}\label{A-prop...Concentration-on-l<L-th-Gram-Matrix}
Suppose $\sigma(\cdot)$ satisfies conditions  in Assumption \ref{Assumption....Activation-Function}, and $\fS$     satisfies conditions  in  Assumption 
 \ref{Assumption...Data}, then     for any $i, j\in [n]$ and $l\in[L]$, the following holds
\begin{equation}\label{A-Prop...eq...C-l<L-G-M...Concentration-on-l<L-th-Gram-Matrix}
 \Prob\left(\Abs{\overline{\mH}_{ij}^{[l]}   (\vtheta^0)-\prod_{k=l}^L\widetilde{\mI}_{ij}^{[k]}}\geq \eta\right)\leq 2\exp\left(-C_0 m \eta^2\right),  
\end{equation}
for some   constant $C_0>0$ depending on $L$.
\end{proposition}
\begin{proof}
We shall prove this by induction. For $l=L$,  as we set $\eps_L :=\frac{1}{\sqrt{m}}\prod_{k=1}^L\alpha_k$,  
\begin{align*}
 \overline{\mH}_{ij}^{[L]}   (\vtheta^0)&=\left<  
 \vsigma_{[L]}^{(1)}\left(\vx_i\right)\frac{\Bar{\va}}{\sqrt{m}},    
 \vsigma_{[L]}^{(1)}\left(\vx_j\right)\frac{\Bar{\va}}{\sqrt{m}}\right>\\
&=\frac{1}{m}\sum_{k'=1}^m a_{k'}^2    {\sigma^{(1)}\left(\eps_L\vw_{k'}^\T\Bar{\vx}_i^{[L-1]}\right)} {\sigma^{(1)}\left(\eps_L\vw_{k'}^\T\Bar{\vx}_j^{[L-1]}\right)}, 
\end{align*}
where  $a_{k'}\sim\fN(\vzero,1) $, and $\vw_{k'}\sim\fN(\vzero,\mI_m)$, and since 
\begin{align*}
  \Norm{a_{k'}^2    {\sigma^{(1)}\left(\eps_L\vw_{k'}^\T\Bar{\vx}_i^{[L-1]}\right)} {\sigma^{(1)}\left(\eps_L\vw_{k'}^\T\Bar{\vx}_j^{[L-1]}\right)}}_{\psi}&\leq c^2 \Norm{a_{k'}^2}_{\psi}\leq c^2   C_{\psi, 1},
\end{align*}
is  a sub-exponential random variable,  and   as we notice that 
\begin{align*}
 \Exp_{(a, \vw)\sim\fN(\vzero,\mI_{m+1})} a^2    {\sigma^{(1)}\left(\eps_L\vw^\T\Bar{\vx}_i^{[L-1]}\right)} {\sigma^{(1)}\left(\eps_L\vw^\T\Bar{\vx}_j^{[L-1]}\right)} 
=&\Exp_{(u,v)^{\T}\sim \fN\left(\vzero, \widetilde{\mB}^{[L-1]}_{ij} \right)}   {\sigma^{(1)}\left(  \eps_Lu\right)} 
 {\sigma^{(1)}\left( \eps_Lv\right)},
\end{align*}
hence by application of  Theorem \ref{A-Thm...Bernstein-Inequality}, for some absolute constant $C_0>0$,
\begin{equation*}
\Prob\left(\Abs{\overline{\mH}_{ij}^{[L]}   (\vtheta^0)-\Exp_{(u,v)^{\T}\sim \fN\left(\vzero, \widetilde{\mB}^{[L-1]}_{ij} \right)}   {\sigma^{(1)}\left(  \eps_Lu\right)} 
 {\sigma^{(1)}\left( \eps_Lv\right)}}\geq \eta\right)\leq 2\exp\left(-\frac{C_0}{c^2C_{\psi, 1}^2} m {\eta^2}\right),
\end{equation*}
then by application of Lemma \ref{A-Lemma...Matrix-Norm-and-Entry-Norm}, we obtain that 
\begin{align*}
&\Abs{G\left(\widetilde{\mB}^{[L-1]}_{ij}(\vtheta^0)\right) -G\left(\widetilde{\mA}^{[L]}_{ij}\right)}\leq  C\Norm{\widetilde{\mB}^{[L-1]}_{ij}(\vtheta^0)-\widetilde{\mA}^{[L]}_{ij}}_{\infty}, 
\end{align*}
thus by similar reasoning in  Proposition \ref{A-prop...Concentration-on-L+1-th-Gram-Matrix}, we obtain that 
\begin{equation}\label{tt}
\begin{aligned}
 &\Prob\left(\Abs{\overline{\mH}_{ij}^{[L]}   (\vtheta^0)-\widetilde{\mI}^{[l]}_{ij}}\geq \eta\right)\\ 
 \leq & \Prob\left(\Abs{\overline{\mH}_{ij}^{[L]}   (\vtheta^0)-\Exp_{(u,v)^{\T}\sim \fN\left(\vzero, \widetilde{\mB}^{[L-1]}_{ij} \right)}   {\sigma^{(1)}\left(  \eps_Lu\right)} 
 {\sigma^{(1)}\left( \eps_Lv\right)}}\geq \frac{\eta}{2}\right)\\
&+\Prob\left(\Abs{G\left(\widetilde{\mB}^{[L-1]}_{ij}(\vtheta^0)\right) -G\left(\widetilde{\mA}^{[L]}_{ij}\right)}\geq \frac{\eta}{2}\right)  
\leq 2\exp\left(-C_0 m {\eta^2}\right).
\end{aligned}
\end{equation}

We assume that    relation \eqref{A-Prop...eq...C-l<L-G-M...Concentration-on-l<L-th-Gram-Matrix} holds for   $l'> l$,  and we proceed to demonstrate that it also  holds true  for $l'=l$.  As we  recall that 
\[
\overline{\mH}_{ij}^{[l+1]}   (\vtheta^0)  = \left<
 \prod_{k=l+1}^{L-1}\overline{\mE}^{[k]}(\vx_i)
 \vsigma_{[L]}^{(1)}\left(\vx_i\right)\frac{\Bar{\va}}{\sqrt{m}},  
\prod_{k=l+1}^{L-1}\overline{\mE}^{[k]}(\vx_j)
 \vsigma_{[L]}^{(1)}\left(\vx_j\right)\frac{\Bar{\va}}{\sqrt{m}}\right>,
\]
and the entries in $ \prod_{k=l+1}^{L-1}\overline{\mE}^{[k]}(\vx_i)
 \vsigma_{[L]}^{(1)}\left(\vx_i\right)\frac{\Bar{\va}}{\sqrt{m}}$ and $ \prod_{k=l+1}^{L-1}\overline{\mE}^{[k]}(\vx_j)
 \vsigma_{[L]}^{(1)}\left(\vx_j\right)\frac{\Bar{\va}}{\sqrt{m}}$ read 
\[  \prod_{k=l+1}^{L-1}\overline{\mE}^{[k]}(\vx_i)
 \vsigma_{[L]}^{(1)}\left(\vx_i\right)\frac{\Bar{\va}}{\sqrt{m}}=\left[\frac{s_r(i)}{\sqrt{m}}\right]_{m\times 1},~~\prod_{k=l+1}^{L-1}\overline{\mE}^{[k]}(\vx_j)
 \vsigma_{[L]}^{(1)}\left(\vx_j\right)\frac{\Bar{\va}}{\sqrt{m}}=\left[\frac{s_r(j)}{\sqrt{m}}\right]_{m\times 1},\]
and the entries in $\vsigma^{(1)}_{[l]}(\vx_i)$ and $\vsigma^{(1)}_{[l]}(\vx_j)$ read
\[
\vsigma^{(1)}_{[l]}(\vx_i)=\mathrm{diag}\left([\mu_p(i)]_{m\times 1}\right),~~\vsigma^{(1)}_{[l]}(\vx_j)=\mathrm{diag}\left([\mu_p(j)]_{m\times 1}\right),
\]
the entries in  $\left(\frac{\overline{\mW}^{[l+1]}}{\sqrt{m}}\right)^\T$ by 
\[
\left(\frac{\overline{\mW}^{[l+1]}}{\sqrt{m}}\right)^\T=\left[\frac{w_{p,q}}{\sqrt{m}}\right]_{m\times m},
\]
therefore, 
\begin{align*}
\overline{\mH}_{ij}^{[l]}   (\vtheta^0)  &= \left<
 \prod_{k=l}^{L-1}\overline{\mE}^{[k]}(\vx_i)
 \vsigma_{[L]}^{(1)}\left(\vx_i\right)\frac{\Bar{\va}}{\sqrt{m}},  
\prod_{k=l}^{L-1}\overline{\mE}^{[k]}(\vx_j)
 \vsigma_{[L]}^{(1)}\left(\vx_j\right)\frac{\Bar{\va}}{\sqrt{m}}\right>\\
&=\left<\mathrm{diag}\left([\mu_p(i)]\right)\left[\frac{w_{p,q}}{\sqrt{m}}\right]\left[\frac{s_q(i)}{\sqrt{m}}\right], \mathrm{diag}\left([\mu_p(j)]\right)\left[\frac{w_{p,r}}{\sqrt{m}}\right]\left[\frac{s_r(j)}{\sqrt{m}}\right]\right>\\
&=\sum_{p,q,r=1}^m\mu_p(i) \frac{w_{p,q}}{\sqrt{m}}\frac{s_{q}(i)}{\sqrt{m}}  \mu_p(j) \frac{w_{p,r}}{\sqrt{m}}\frac{s_{r}(j)}{\sqrt{m}},
\end{align*}
it shall be noticed that for the  fixed pair $(q,r)$, the coefficients of $\frac{s_{q}(i)}{\sqrt{m}}\frac{s_{r}(j)}{\sqrt{m}}$ read
\[
\sum_{p=1}^m \mu_p(i)\mu_p(j)\frac{w_{p,q}}{\sqrt{m}} \frac{w_{p,r}}{\sqrt{m}}=\frac{1}{m} \sum_{p=1}^m \mu_p(i)\mu_p(j)w_{p,q}w_{p,r},
\]
where for any fixed $p\in[m]$,
\[
\Norm{\mu_p(i)\mu_p(j)w_{p,q}w_{p,r}}_{\psi}\leq c^2 \Norm{w_{p,q}}_\psi\Norm{w_{p,r}}_\psi\leq c^2C_{\psi, 1},
\]
is a sub-exponential random variable, and   $w_{p,q}, w_{p,r}$ are independent with  $\mu_p(i), \mu_p(j)$. 

Moreover, if $q\neq r$, then $w_{p,q}, w_{p,r}$ are independent with each other, with expectation  
\[
\Exp    \mu_p(i)\mu_p(j)w_{p,q}w_{p,r} = \Exp \left[\mu_p(i)\mu_p(j)\right]\Exp [w_{p,q}]\Exp [w_{p,r}]=0,
\]
if $q=r$, then  its expectation reads
\[
\Exp   \mu_p(i)\mu_p(j)w_{p,q}^2 =  \sum_{p=1}^m \Exp \left[\mu_p(i)\mu_p(j)\right]\Exp [w_{p,q}^2]= \Exp \left[\mu_p(i)\mu_p(j)\right],
\]
and as we set  $\eps_l =\frac{1}{\sqrt{m}}\prod_{k=1}^{l}\alpha_k$, then 
\begin{align*}
 \Exp \left[\mu_p(i)\mu_p(j)\right] &=\Exp_{ \vw\sim\fN(\vzero,\mI_{m})}    {\sigma^{(1)}\left(\eps_{l}\vw^\T\Bar{\vx}_i^{[l-1]}\right)} {\sigma^{(1)}\left(\eps_l\vw^\T\Bar{\vx}_j^{[l-1]}\right)} \\
 &=\Exp_{(u,v)^{\T}\sim \fN\left(\vzero, \widetilde{\mB}^{[l-1]}_{ij} (\vtheta^0)\right)}   {\sigma^{(1)}\left(  \eps_{l}u\right)} 
 {\sigma^{(1)}\left( \eps_{l}v\right)},
\end{align*}
therefore, we   focus on the case where $q=r,$ and  the coefficients of $\frac{s_{q}(i)}{\sqrt{m}}\frac{s_{q}(i)}{\sqrt{m}}$ satisfy
\[
 \Prob\left(\Abs{\frac{1}{m} \sum_{p=1}^m \mu_p(i)\mu_p(j)w_{p,q}^2- G\left(\widetilde{\mB}^{[l-1]}_{ij}(\vtheta^0)\right)}\geq \eta\right)\leq 2\exp\left(-C_0 m \eta^2\right), 
\]
and by similar reasoning in \eqref{tt}, we obtain that $G\left(\widetilde{\mB}^{[l-1]}_{ij}(\vtheta^0)\right)$ is close to  $\widetilde{\mI}^{[l]}_{ij}$, then with high probability, the quantity  $\sum_{p,q,r=1}^m\mu_p(i) \frac{w_{p,q}}{\sqrt{m}}\frac{s_{q}(i)}{\sqrt{m}}  \mu_p(j) \frac{w_{p,r}}{\sqrt{m}}\frac{s_{r}(j)}{\sqrt{m}}$ converges to 
\[
\widetilde{\mI}^{[l]}_{ij} \left(\sum_{q=1}^m\frac{s_{q}(i)s_{q}(j)}{{m}}\right)=\widetilde{\mI}^{[l]}_{ij} \overline{\mH}_{ij}^{[l+1]}   (\vtheta^0),
\]
and based on the induction hypothesis, we finish our proof.
\end{proof}
\noindent
Based on  Proposition \ref{A-prop...Concentration-on-L+1-th-Gram-Matrix} and Proposition \ref{A-prop...Concentration-on-l<L-th-Gram-Matrix}, we state a corollary without proof.
\begin{corollary}\label{A-Cor...gram-matrix}
    Suppose $\sigma(\cdot)$ satisfies conditions  in Assumption \ref{Assumption....Activation-Function}, and $\fS$     satisfies conditions  in  Assumption 
 \ref{Assumption...Data}, then     for any $i, j\in [n]$, and $l\in[L+1]$, the following holds
\begin{equation} 
 \Prob\left(\Abs{\overline{\mG}_{ij}^{[l]}   (\vtheta^0) -{\mK}^{[l]}_{ij}}\geq \eta\right)\leq 2\exp\left(-C_0 m \eta^2\right),  
\end{equation}
for some   constant $C_0>0$ depending on $L$.
\end{corollary}

As for the least eigenvalue of normalized Gram matrices at $t=0$, i.e., $\left\{\overline{\mG}^{[l]}(\vtheta^0)\right\}_{l=1}^{L+1}$, we obtain that 
\begin{proposition} \label{A-prop...Least-Eigenvalue-on-the-NTK}
Suppose $\sigma(\cdot)$ satisfies conditions  in Assumption \ref{Assumption....Activation-Function}, and $\fS$     satisfies conditions  in  Assumption 
 \ref{Assumption...Data}, if 
\[m=\Omega\left(\frac{n^2}{\lambda_\fS^2}\log n\right),\]
where $\lambda_\fS=\min_{l\in[L+1]}    \lambda_{\min}\left(\mK^{[l]}\right)$, then with high probability, for any $l\in[L+1]$
\begin{equation}\label{A-Prop...eq...L-E-NTK...Least-Eigenvalue-lth-Gram}
\lambda_{\min}\left(\overline{\mG}^{[l]}(\vtheta^0)\right)\geq \frac{3}{4}\lambda_\fS.
\end{equation}
\end{proposition}
\begin{proof}
 For any $\eta > 0$ and $l\in[L+1]$,  and for all  $i, j\in [n]$,  we define the events
\begin{equation*}
\begin{aligned}
\Omega_{ij}^{[l]}&:=\left\{\vtheta^0 ~\middle|~   \Abs{\overline{\mG}^{[l]}_{ij}\left(\vtheta^0\right) - \mK^{[l]}_{ij}}  \leq \frac{\eta}{n} \right\},     
\end{aligned}
\end{equation*}
By application of Corollary \ref{A-Cor...gram-matrix},  we obtain that for any $\eta>0$ and $l\in[L+1]$, 
\begin{equation*}
\begin{aligned}
\Prob(\Omega^{[l]}_{ij})   & \geq 1-2\exp\left(-\frac{C_0m\eta^2}{n^2}\right), 
\end{aligned}
\end{equation*}
hence with probability at least $1-4n^2\exp\left(-\frac{C_0m\eta^2}{n^2}\right)$ over the choice of $\vtheta^0$, we have
\begin{equation*}
\begin{aligned}
\Norm{\overline{\mG}^{[l]}\left(\vtheta^0\right) - \mK^{[l]}}_\mathrm{F} &\leq n \Norm{\overline{\mG}^{[l]}\left(\vtheta^0\right) - \mK^{[l]}}_{\infty}\leq \eta.
\end{aligned}
\end{equation*}
By taking $\eta=\frac{\lambda_\fS}{4}$, we conclude that 
\begin{equation*}
\begin{aligned}
\lambda_{\min}\left(\overline{\mG}^{[l]}\left(\vtheta^0\right)\right)  
& \geq\lambda_\fS-\Norm{\overline{\mG}^{[l]}\left(\vtheta^0\right) - \mK^{[l]}}_\mathrm{F}  \geq\lambda_\fS-\frac{\lambda_\fS}{4}  =\frac{3}{4}\lambda_\fS.
\end{aligned}
\end{equation*}
\end{proof}
\subsection{A Unified Approach for Multi-layer NNs}
As the dynamics of the normalized parameters read 
\begin{equation}\label{eqgroup...text...Normalized-Dynamics}
\left\{
\begin{aligned}
\frac{\mathrm{d} \frac{\overline{\mW}^{[l]}}{\sqrt{m}}}{\mathrm{d} t} &  =-\frac{\kappa}{\alpha_l^2}\frac{1}{n}   \sum_{i=1}^n e_i\left(\prod_{k=l}^{L-1}\overline{\mE}^{[k]}(\vx_i)\right)  \vsigma_{[L]}^{(1)}\left(\vx_i\right)\frac{\bar{\va}}{\sqrt{m}}  \otimes  \bar{\vx}_i^{[l-1]},~~l\in[L-1],\\
\frac{\mathrm{d} \frac{\overline{\mW}^{[L]}}{\sqrt{m}}}{\mathrm{d} t} &  =-\frac{\kappa}{\alpha_L^2}\frac{1}{n}   \sum_{i=1}^n e_i \vsigma_{[L]}^{(1)}\left(\vx_i\right)\frac{\bar{\va}}{\sqrt{m}}  \otimes  \bar{\vx}_i^{[L-1]} ,\\
 \frac{\mathrm{d} \frac{\bar{\va}}{\sqrt{m}} }{\mathrm{d} t} &  =-\frac{\kappa}{\alpha_{L+1}^2}\frac{1}{n} \sum_{i=1}^{n}e_i  \bar{\vx}_i^{[L]}.
\end{aligned}     
\right.
\end{equation}
Based on  dynamics \eqref{eqgroup...text...Normalized-Dynamics}, by taking  norm  on both sides,  we obtain that  for any $l\in[L-1]$,       
\begin{align*}
\frac{\D \Norm{\frac{\overline{\mW}^{[l]}}{\sqrt{m}}}_{2\to 2}}{\D t} 
&\leq\frac{\sqrt{2}\kappa}{\alpha_l^2} \left(\prod_{k=1}^{l-1}\Norm{\frac{\overline{\mW}^{[k]}}{\sqrt{m}}}_{2\to 2}\right)\left(\prod_{k=l+1}^{L}\Norm{ \frac{\overline{\mW}^{[k]}}{\sqrt{m}} }_{2\to 2}\right)\Norm{\frac{\bar{\va}}{\sqrt{m}}}_2 \sqrt{R_\fS(\vtheta)},\\
\frac{\D \Norm{\frac{\overline{\mW}^{[L]}}{\sqrt{m}}}_{2\to 2}}{\D t} 
&\leq\frac{\sqrt{2}\kappa}{\alpha_L^2} \left(\prod_{k=1}^{L-1}\Norm{\frac{\overline{\mW}^{[k]}}{\sqrt{m}}}_{2\to 2}\right) \Norm{\frac{\bar{\va}}{\sqrt{m}}}_2 \sqrt{R_\fS(\vtheta)},\\
\frac{\mathrm{d} \Norm{\frac{\bar{\va}}{\sqrt{m}}}_2 }{\mathrm{d} t} &\leq  \frac{\sqrt{2}\kappa}{\alpha_{L+1}^2}\left(\prod_{k=1}^{L}\Norm{\frac{\overline{\mW}^{[k]}}{\sqrt{m}}}_{2\to 2}\right) \sqrt{R_\fS(\vtheta)}.
\end{align*}
As we denote  that for any $l\in[L]$,
\begin{equation} \label{eq...text...Definition-of-pL+pl}
\begin{aligned}
p_l(t)&:=\sup_{s\in[0,t]}\Norm{\frac{\overline{\mW}^{[l]}(s)}{\sqrt{m}} }_{2\to 2},~~p_{L+1}(t) :=\sup_{s\in[0,t]}\Norm{\frac{\bar{\va}(s)}{\sqrt{m}}}_{2},
\end{aligned}
\end{equation}
the above inequality reads 
\begin{equation}
\left\{
\begin{aligned}
\frac{\D p_l(t)}{\D t}&\leq \frac{\sqrt{2}\kappa}{\alpha_l^2}\left(\prod_{k=1}^{l-1}p_k(t)\right)\left(\prod_{k=l+1}^{L+1}p_k(t)\right)\sqrt{R_\fS(\vtheta)},~~l\in[L], \\
\frac{\D p_{L+1}(t)}{\D t}&\leq \frac{\sqrt{2}\kappa}{\alpha_{L+1}^2}\left(\prod_{k=1}^{L}p_k(t)\right) \sqrt{R_\fS(\vtheta)}.
\end{aligned}    
\right.
\end{equation}
Define the \textbf{stopping time} 
\begin{equation}\label{eq...text...Stopping-Time}
    t^\ast = \inf\{t \mid \vtheta(t)\notin \mathcal{N}\left(\vtheta^0\right)\},
\end{equation}
where the event is defined as 
\begin{equation*}
\mathcal{N}\left(\vtheta^0\right) := \left\{\vtheta \mid \Norm{\mG(\vtheta) - \mG\left(\vtheta^0\right)}_\mathrm{F}\leq  \left(\sum_{l=1}^{L+1} \frac{\kappa^2}{\alpha_l^2}\right)\frac{\lambda_\fS}{4}\right\},
\end{equation*}
and we observe immediately that 
the event $\mathcal{N}\left(\vtheta^0\right)\neq \varnothing$, since $\vtheta^0\in\mathcal{N}\left(\vtheta^0\right)$.
\begin{proposition}\label{B-prop...Loss-Initial-Decay}
Suppose $\sigma(\cdot)$ satisfies conditions  in Assumption \ref{Assumption....Activation-Function}, and $\fS$     satisfies conditions  in  Assumption 
 \ref{Assumption...Data}, if  $\sum_{k=1}^{L+1}\gamma_k<\frac{L+1}{2}$, i.e., $\lim_{m\to\infty} \frac{\log \kappa}{\log m}>0$, and
\[m=\Omega\left(\frac{n^2}{\lambda_\fS^2}\log n\right),\]
where $\lambda_\fS=\min_{l\in[L+1]}    \lambda_{\min}\left(\mK^{[l]}\right)$, then with high probability, for any time $t\in[0, t^\ast)$,
\begin{equation}\label{B-Prop...eq...Loss-Initial-Decay}
R_\fS(\vtheta(t)) \leq \exp\left(- \frac{1}{n}\left[\left(\sum_{l=1}^{L+1} \frac{\kappa^2}{\alpha_l^2}\right) {\lambda_\fS}\right]t\right)R_\fS(\vtheta(0)).
\end{equation}
\end{proposition}
\begin{proof}
As we notice that
\[
\mG(\vtheta)=\sum_{l=1}^{L+1} \mG^{[l]} (\vtheta)=\sum_{l=1}^{L+1}\frac{\kappa^2}{\alpha_l^2} \overline{\mG}^{[l]} (\vtheta),
\]
therefore, based on Proposition \ref{A-prop...Least-Eigenvalue-on-the-NTK}, with high probability
\begin{equation}
\lambda_{\min}\left(\mG(\vtheta^0)\right)\geq \sum_{l=1}^{L+1} \lambda_{\min}\left(\mG^{[l]} (\vtheta^0)\right)\geq \left(\sum_{l=1}^{L+1}\frac{\kappa^2}{\alpha_l^2}\right) \frac{3}{4}\lambda_\fS,
\end{equation}
and for any $\vtheta\in\mathcal{N}\left(\vtheta^0\right)$, 
\begin{align*}
\lambda_{\min}\left(\mG(\vtheta)\right)& \geq \lambda_{\min}\left(\mG\left(\vtheta^0\right)\right) - \norm{\mG(\vtheta) - \mG\left(\vtheta^0\right)}_\mathrm{F}\\
& \geq \left(\sum_{l=1}^{L+1} \frac{\kappa^2}{\alpha_l^2}\right)\frac{3}{4}\lambda_\fS-\left(\sum_{l=1}^{L+1} \frac{\kappa^2}{\alpha_l^2}\right)\frac{\lambda_\fS}{4}= \left(\sum_{l=1}^{L+1} \frac{\kappa^2}{\alpha_l^2}\right)\frac{\lambda_\fS}{2}.
\end{align*}
Finally, we obtain that
\begin{equation*}
\begin{aligned}
\frac{\D}{\D t}R_\fS(\vtheta)& =  \frac{\D}{\D t}\left(\frac{1}{2n}\sum_{i=1}^n e_i^2\right) =\frac{1}{n}\sum_{i=1}^n e_i\frac{\D e_i}{\D t} =-\frac{1}{n^2}\sum_{i=1}^n  e_i\left(\sum_{l=1}^{L+1}\frac{\kappa^2}{\alpha_l^2}\overline{\mG}^{[l]} (\vtheta)\right)e_j\\
&\leq -\frac{2}{n}\lambda_{\min}\left(\sum_{l=1}^{L+1}\frac{\kappa^2}{\alpha_l^2}\overline{\mG}^{[l]} \right)\frac{1}{2n}\sum_{i=1}^n  e_i^2 \leq -\frac{1}{n}\left[\left(\sum_{l=1}^{L+1} \frac{\kappa^2}{\alpha_l^2}\right) {\lambda_\fS}\right] R_\fS(\vtheta),
\end{aligned}\end{equation*}
and immediate integration yields the result.
\end{proof}
\noindent
Since relation \eqref{B-Prop...eq...Loss-Initial-Decay} holds for any time $t\in[0, t^\ast)$,  we obtain that 
\begin{equation*}
\begin{aligned}
\sqrt{R_\fS(\vtheta)}&\leq \exp\left(-\frac{1}{2n} \left[\sum_{l=1}^{L+1}\frac{\kappa^2}{\alpha_l^2}\lambda_\fS\right] t \right) \sqrt{R_\fS(\vtheta(0))}.
\end{aligned}
\end{equation*}
Therefore, we have for any $l\in[L]$ and time $t\in[0, t^\ast)$,
\begin{equation*}
\begin{aligned}
 p_l(t) &\leq  p_l(0)+\frac{\sqrt{2}\kappa}{\alpha_l^2}\left(\prod_{k=1}^{l-1}p_k(t)\right)\left(\prod_{k=l+1}^{L+1}p_k(t)\right)\int_{0}^t\sqrt{R_\fS(\vtheta(s))}\D s \\
&\leq p_l(0)+\frac{\sqrt{2}\kappa\sqrt{R_\fS(\vtheta(0))}}{\alpha_l^2}\left(\prod_{k=1}^{l-1}p_k(t)\right)\left(\prod_{k=l+1}^{L+1}p_k(t)\right)\int_{0}^{+\infty}\exp\left(-\frac{1}{2n} \left[\sum_{l=1}^{L+1}\frac{\kappa^2}{\alpha_l^2}\lambda_\fS\right] s \right)\D s\\
&\leq p_l(0)+\frac{\sqrt{2}\kappa}{\alpha_l^2} \frac{2n\sqrt{R_\fS(\vtheta(0))}}{\left(\sum_{l=1}^{L+1} \frac{\kappa^2}{\alpha_l^2}\right) {\lambda_\fS}} \left(\prod_{k=1}^{l-1}p_k(t)\right)\left(\prod_{k=l+1}^{L+1}p_k(t)\right), 
\end{aligned}    
\end{equation*}
and by similar reasoning 
\begin{equation*}
\begin{aligned}
 p_{L+1}(t) &\leq  p_{L+1}(0)+  \frac{\sqrt{2}\kappa}{\alpha_{L+1}^2} \frac{2n\sqrt{R_\fS(\vtheta(0))}}{\left(\sum_{l=1}^{L+1} \frac{\kappa^2}{\alpha_l^2}\right) {\lambda_\fS}} \left(\prod_{k=1}^{L}p_k(t)\right).
\end{aligned}    
\end{equation*}
Since $\lim_{m\to\infty}\frac{\log \kappa}{\log m}>0$,   then  $\kappa>1$,  and for any $l\in[L+1]$,
\begin{equation*}
    \frac{\sqrt{2}\kappa}{\alpha_{l}^2} \frac{2n\sqrt{R_\fS(\vtheta(0))}}{\left(\sum_{l=1}^{L+1} \frac{\kappa^2}{\alpha_l^2}\right) {\lambda_\fS}}\leq \frac{2\sqrt{2}n\sqrt{R_\fS(\vtheta(0))}}{\kappa \lambda_\fS},
\end{equation*}
therefore, we obtain that  
\begin{equation}
\left\{
\begin{aligned}
    p_l(t)&\leq p_l(0)+\frac{2\sqrt{2}n\sqrt{R_\fS(\vtheta(0))}}{\kappa \lambda_\fS} \left(\prod_{k=1}^{l-1}p_k(t)\right)\left(\prod_{k=l+1}^{L+1}p_k(t)\right),~~l\in[L],\\ 
        p_{L+1}(t)&\leq p_{L+1}(0)+\frac{2\sqrt{2}n\sqrt{R_\fS(\vtheta(0))}}{\kappa \lambda_\fS} \left(\prod_{k=1}^{L}p_k(t)\right),
\end{aligned}    
\right.
\end{equation}
and if we choose $m$ large enough, then $\frac{1}{\kappa}\to 0$, hence $p_l(t)$ hardly varies for time $t\in[0, t^\ast)$. Consequently, we state  Proposition \ref{B-prop...Parameter-no-Movement-2-Norm}, whose rigorous proof can be found in Appendix \ref{B-subsection...Proof-of-Proposition-Parameter-no-Movements}.
\begin{proposition}\label{B-prop...Parameter-no-Movement-2-Norm}
Suppose $\sigma(\cdot)$ satisfies conditions  in Assumption \ref{Assumption....Activation-Function}, and $\fS$     satisfies conditions  in  Assumption 
 \ref{Assumption...Data}, if $\sum_{k=1}^{L+1}\gamma_k<\frac{L+1}{2}$, i.e., $\lim_{m\to\infty} \frac{\log \kappa}{\log m}>0$, and
\[m=\max\left\{\Omega\left(\frac{n^2}{\lambda_\fS^2}\log n\right),  \Omega\left(\left( \frac{n}{  \lambda_\fS}\right)^{\frac{1}{\frac{L+1}{2}-\sum_{k=1}^{L+1}\gamma_k}}\right)\right\},\]
where $\lambda_\fS=\min_{l\in[L+1]}    \lambda_{\min}\left(\mK^{[l]}\right)$, then with high probability,     for any $l\in[L]$ and   $t\in[0, t^\ast)$,
\begin{equation}\label{B-Prop...eq...Parameter-no-Movement-2-Norm}
\begin{aligned}
\Norm{\frac{\bar{\va}(t)}{\sqrt{m}}-\frac{\bar{\va}(0)}{\sqrt{m}}}_2  &\leq\frac{2\sqrt{2}n\sqrt{R_\fS(\vtheta^0)}}{\kappa \lambda_\fS}4^L,\\
\Norm{\frac{\overline{\mW}^{[l]}(t)}{\sqrt{m}}-\frac{\overline{\mW}^{[l]}(0)}{\sqrt{m}}}_{2\to 2} 
&\leq \frac{2\sqrt{2}n\sqrt{R_\fS(\vtheta^0)}}{\kappa \lambda_\fS}4^L.
\end{aligned}
\end{equation}
\end{proposition}
\noindent
Finally, we remark that  since for any $l\in[L]$, the dynamics of $\frac{\overline{\mW}^{[l]}}{\sqrt{m}}$ always takes the form
\[
\frac{\mathrm{d} \frac{\overline{\mW}^{[l]}}{\sqrt{m}}}{\mathrm{d} t}    =\vu\otimes\vz,
\]
for some vector $\vu$ and $\vz$. Therefore,  the following estimates hold true regardless of the choice of operator norm or Frobenius norm,  
\begin{align*}
\frac{\D \Norm{\frac{\overline{\mW}^{[l]}}{\sqrt{m}}}_{2\to 2}}{\D t} 
&\leq   \norm{\vu}_2   \norm{\vz}_2,\\
\frac{\D \Norm{\frac{\overline{\mW}^{[l]}}{\sqrt{m}}}_{\mathrm{F}}}{\D t} 
&\leq   \norm{\vu}_2   \norm{\vz}_2,
\end{align*}
hence the variation of  $\Norm{\frac{\overline{\mW}^{[l]}(\cdot)}{\sqrt{m}}}_{\mathrm{F}}$ shares the same upper bound as the variation of  $\Norm{\frac{\overline{\mW}^{[l]}(\cdot)}{\sqrt{m}}}_{2\to 2}$. Thus, a corollary  is immediately obtained as follows.
\begin{corollary}\label{B-Cor...Parameter-no-Movement-Frobenius-Norm}
Suppose $\sigma(\cdot)$ satisfies conditions  in Assumption \ref{Assumption....Activation-Function}, and $\fS$     satisfies conditions  in  Assumption 
 \ref{Assumption...Data}, if $\sum_{k=1}^{L+1}\gamma_k<\frac{L+1}{2}$, i.e., $\lim_{m\to\infty} \frac{\log \kappa}{\log m}>0$, and
\[m=\max\left\{\Omega\left(\frac{n^2}{\lambda_\fS^2}\log n\right),  \Omega\left(\left( \frac{n}{  \lambda_\fS}\right)^{\frac{1}{\frac{L+1}{2}-\sum_{k=1}^{L+1}\gamma_k}}\right)\right\},\]
where $\lambda_\fS=\min_{l\in[L+1]}    \lambda_{\min}\left(\mK^{[l]}\right)$, then with high probability,     for any $l\in[L]$ and  $t\in[0, t^\ast)$,
\begin{equation}\label{B-Cor...eq...Parameter-no-Movement-Frobenius-Norm}
\begin{aligned} 
\Norm{\frac{\overline{\mW}^{[l]}(t)}{\sqrt{m}}-\frac{\overline{\mW}^{[l]}(0)}{\sqrt{m}}}_{\mathrm{F}} 
&\leq \frac{2\sqrt{2}n\sqrt{R_\fS(\vtheta^0)}}{\kappa \lambda_\fS}4^L.
\end{aligned}
\end{equation}
\end{corollary}
\begin{proof}
By taking  Frobenius norm into dynamics \eqref{eqgroup...text...Normalized-Dynamics}, then for any $l\in[L-1]$,  
\begin{align*}
\frac{\D \Norm{\frac{\overline{\mW}^{[l]}}{\sqrt{m}}}_{\mathrm{F}}}{\D t} &\leq\frac{\sqrt{2}\kappa}{\alpha_l^2} \left(\prod_{k=1}^{l-1}\Norm{\frac{\overline{\mW}^{[k]}}{\sqrt{m}}}_{2\to 2}\right)\left(\prod_{k=l+1}^{L}\Norm{ \frac{\overline{\mW}^{[k]}}{\sqrt{m}} }_{2\to 2}\right)\Norm{\frac{\bar{\va}}{\sqrt{m}}}_2 \sqrt{R_\fS(\vtheta)},\\
\frac{\D \Norm{\frac{\overline{\mW}^{[L]}}{\sqrt{m}}}_{\mathrm{F}}}{\D t} &\leq\frac{\sqrt{2}\kappa}{\alpha_L^2} \left(\prod_{k=1}^{L-1}\Norm{\frac{\overline{\mW}^{[k]}}{\sqrt{m}}}_{2\to 2}\right) \Norm{\frac{\bar{\va}}{\sqrt{m}}}_2 \sqrt{R_\fS(\vtheta)}, 
\end{align*}
and we obtain immediately that 
for any $l\in[L]$ and  time $t\in[0,t^*)$,
\begin{align*}
 \Norm{\frac{\overline{\mW}^{[l]}(t)}{\sqrt{m}}-\frac{\overline{\mW}^{[l]}(0)}{\sqrt{m}}}_{\mathrm{F}} &\leq \frac{2\sqrt{2}n\sqrt{R_\fS(\vtheta^0)}}{\kappa \lambda_\fS} \left(\prod_{k=1}^{l-1}p_k(t)\right)\left(\prod_{k=l+1}^{L+1}p_k(t)\right)\\
&\leq \frac{2\sqrt{2}n\sqrt{R_\fS(\vtheta^0)}}{\kappa \lambda_\fS}4^L.
\end{align*}    
\end{proof}
\subsection{Theta-lazy Regime}\label{sub-theta}

It is known that the  output function of  $L$-layer NNs  is linear with respect to $\vtheta_a$, hence for any $l\in[L]$, if the set of parameters $\vtheta_{\mW^{[l]}}$ remain stuck to its initialization throughout the whole training process, then the training dynamics of $L$-layer NNs   can be linearized around the initialization. The $\vtheta$-lazy regime area precisely corresponds to the region where the output function of $L$-layer NNs   can be well approximated by its linearized model, i.e.,  
\begin{equation}\label{eq...text...Taylor}
\begin{aligned}
 f_{\vtheta}(\vx)&\approx f\left(\vx, \vtheta(0)\right)+\left<\nabla_{\va} f\left(\vx, \vtheta(0)\right),   \vtheta_a(t)- \vtheta_a(0) \right> \\
&~~~~~~~~~~~~~~~~~~+\sum_{l=1}^L\left<\nabla_{{\mW^{[l]}}} f\left(\vx, \vtheta(0)\right),  \vtheta_{\mW^{[l]}}(t)- \vtheta_{\mW^{[l]}}(0) \right>.
\end{aligned}
\end{equation}
In general, this linear approximation  holds valid only when  $\vtheta_{\mW^{[l]}}(t)$ remains within a small neighbourhood of $ \vtheta_{\mW^{[l]}}(0)$ for every $l\in[L]$.
Since the size of this neighbourhood  scales with $\Norm{ \vtheta_{\mW^{[l]}}(0)}_{2}$,  and  as shown in relation \eqref{vec},  $\vtheta_{\mW^{[l]}}$ is the  vectorized form of  $\mW^{[l]}$,
hence it shall be noted that for any $l\in[L]$,
\[\Norm{{\vtheta}_{\mW^{[l]}}}_2=\Norm{{{\mW}^{[l]}}}_{\mathrm{F}},\]
and   the following quantity is employed to characterize how far $ \vtheta_{\mW^{[l]}}(t)$ deviates away from $ \vtheta_{\mW^{[l]}}(0)$,  
\begin{equation}\label{eq...text...MainResults...Regime-Characterization}
   \mathrm{RD}\left( {\mW^{[l]}}\right)(t):=\frac{\Norm{ \vtheta_{\mW^{[l]}}(t)- \vtheta_{\mW^{[l]}}(0)}_{2}}{\Norm{ \vtheta_{\mW^{[l]}}(0)}_{2}}=\frac{\Norm{\mW^{[l]}(t)-\mW^{[l]}(0)}_{\mathrm{F}}}{\Norm{ \mW^{[l]}(0)}_{\mathrm{F}}}.
\end{equation}
Hence    for any $l\in[L]$, we  firstly  identify the parameters 
\begin{equation}
\Bar{\vtheta}_{\mW^{[l]}}:=\mathrm{vec} \left(\frac{\overline{\mW}^{[l]}}{\sqrt{m}}\right),~~\Bar{\vtheta}_{\va}:= \frac{\bar{\va}}{\sqrt{m}},
\end{equation} 
and it also shall be noted that for any $l\in[L]$,
\[\Norm{\Bar{\vtheta}_{\mW^{[l]}}}_2=\Norm{\frac{\overline{\mW}^{[l]}}{\sqrt{m}}}_{\mathrm{F}}.\]
More importantly,   we observe that 
\begin{equation*}
 \mathrm{RD}\left( {\mW^{[l]}}\right)(t)=\frac{\Norm{\mW^{[l]}(t)-\mW^{[l]}(0)}_{\mathrm{F}}}{\Norm{ \mW^{[l]}(0)}_{\mathrm{F}}}=\frac{\Norm{\frac{\overline{\mW}^{[l]}(t)}{\sqrt{m}}-\frac{\overline{\mW}^{[l]}(0)}{\sqrt{m}}}_{\mathrm{F}}}{\Norm{\frac{\overline{\mW}^{[l]}(0)}{\sqrt{m}}}_{\mathrm{F}}}.    
\end{equation*}
We shall provide some estimates  on the upper   and lower bounds of  initial parameters $\bar{\vtheta}^0$. 
\begin{proposition}[Upper and lower bounds of initial parameters]\label{A-prop..Upper-Bound-and-Lower-Bound-Initial-Parameter}
With high probability   over the choice of $\vtheta^0$,  for any $l\in[L]$,
\begin{equation}\label{A-prop...eq...Upper-Bound-and-Lower-Bound-Initial-Parameter...2-2-Norm}
\begin{aligned}
\sqrt{\frac{1}{2}}& \leq \Norm{\frac{\bar{\va}(0)}{\sqrt{m}}}_2\leq \sqrt{\frac{3}{2}}, \\
{\frac{1}{2}} & \leq \Norm{\frac{\overline{\mW}^{[l]} (0)}{\sqrt{m}}}_{2\to 2}\leq 2. 
\end{aligned}
\end{equation}
Moreover, for any $l\in[2:L]$,
\begin{equation}\label{A-prop...eq...Upper-Bound-and-Lower-Bound-Initial-Parameter...Frobenius-Norm}
\begin{aligned}
\sqrt{\frac{md}{2}} &\leq \Norm{{\overline{\mW}^{[1]} (0)}}_{\mathrm{F}}\leq \sqrt{\frac{3md}{2}},\\
\sqrt{\frac{1}{2}}m &\leq \Norm{{\overline{\mW}^{[l]} (0)}}_{\mathrm{F}}\leq \sqrt{\frac{3}{2}}m.
\end{aligned}
\end{equation}
\end{proposition}
\begin{proof}
 Establishment of relation \eqref{A-prop...eq...Upper-Bound-and-Lower-Bound-Initial-Parameter...2-2-Norm} arises directly from Lemma \ref{A-Lemma...Operator-Norm-Random-Matrix} and Lemma \ref{A-Lemma......2-Norm-Chi-Square}.
  As for relation \eqref{A-prop...eq...Upper-Bound-and-Lower-Bound-Initial-Parameter...Frobenius-Norm}, we observe that  for $l=1$, 
since 
\[\left(\overline{\mW}^{[1]}_{1,1}(0)\right)^2,  \cdots \left(\overline{\mW}^{[l]}_{1,d}(0)\right)^2; \left(\overline{\mW}^{[l]}_{2,1}(0)\right)^2,\cdots, \left(\overline{\mW}^{[l]}_{2,d}(0)\right)^2;\cdots \left(\overline{\mW}^{[l]}_{m,d}(0)\right)^2\sim\chi^2(1),\] 
are i.i.d.\ sub-exponential  
random variables   with 
 $\Exp \left(\overline{\mW}^{[1]}_{1,1}(0)\right)^2 =1.$ 
By application of  Theorem \ref{A-Thm...Bernstein-Inequality}, we have
\begin{equation*}
\Prob\left(\Abs{\frac{1}{md}\sum_{i=1}^{m}\sum_{j=1}^{d}\left(\overline{\mW}^{[1]}_{i,j}(0)\right)^2-1}\geq \eta\right)\leq 2\exp\left(-\frac{ C_0 m^2 \eta^2}{C^2_{\psi,1}} \right),
\end{equation*}
as we set $\eta=\frac{1}{2}$, then with high probability   over the choice of $\vtheta^0$, we obtain that 
\[\sqrt{\frac{md}{2}} \leq \Norm{{\overline{\mW}^{[1]} (0)}}_{\mathrm{F}}\leq \sqrt{\frac{3md}{2}},\]
and for any $l\in[2:L]$, by similar reasoning  we obtain that 
\[\sqrt{\frac{1}{2}}m \leq \Norm{{\overline{\mW}^{[l]} (0)}}_{\mathrm{F}}\leq \sqrt{\frac{3}{2}}m.\]
\end{proof}
 
Therefore, we obtain that   
\begin{align*}
 \sup\limits_{t\in[0,t^*)}\mathrm{RD}\left( {\mW^{[1]}}\right)(t)&= \frac{\Norm{\frac{\overline{\mW}^{[1]}(t)}{\sqrt{m}}-\frac{\overline{\mW}^{[1]}(0)}{\sqrt{m}}}_{\mathrm{F}}}{\Norm{\frac{\overline{\mW}^{[1]}(0)}{\sqrt{m}}}_{\mathrm{F}}}\lesssim  \frac{1}{m^{\frac{L+1}{2}-\sum_{k=1}^{L+1}\gamma_k}}\frac{4n}{\lambda_\fS}\sqrt{\frac{R_\fS(\vtheta^0)}{d}}4^L,\end{align*}
and for $l\in[2:L]$, we have 
\begin{align*}
 \sup\limits_{t\in[0,t^*)}\mathrm{RD}\left(  {\mW^{[l]}}\right)(t)&= 
 \frac{\Norm{\frac{\overline{\mW}^{[l]}(t)}{\sqrt{m}}-\frac{\overline{\mW}^{[l]}(0)}{\sqrt{m}}}_{\mathrm{F}}}{\Norm{\frac{\overline{\mW}^{[l]}(0)}{\sqrt{m}}}_{\mathrm{F}}}\lesssim   \frac{1}{m^{\frac{L+2}{2}-\sum_{k=1}^{L+1}\gamma_k}}\frac{4n\sqrt{{R_\fS(\vtheta^0)}}}{\lambda_\fS}
 4^L,\end{align*}
and finally, we have 
\begin{align*} 
\sup\limits_{t\in[0,t^*)}\mathrm{RD}(  {\va})(t)&=\frac{\Norm{\frac{\Bar{\va}(t)}{\sqrt{m}}-\frac{\Bar{\va}(0)}{\sqrt{m}}}_2}{\Norm{\frac{\Bar{\va}(0)}{\sqrt{m}}}_2} \lesssim   \frac{1}{m^{\frac{L+1}{2}-\sum_{k=1}^{L+1}\gamma_k}}\frac{4n\sqrt{{R_\fS(\vtheta^0)}}}{\lambda_\fS}4^L.
\end{align*}
Our above analysis reveals that as $\sum_{k=1}^{L+1}\gamma_k<\frac{L+1}{2}$, i.e., $\lim_{m\to\infty}\frac{\log \kappa}{\log m}>0$, then for large enough $m$,   all entries of $\bar{\vtheta}$, the vector containing all normalized parameters, vary slightly during  the  training period $[0, t^*)$.
It is noteworthy that for $l\in[2:L]$, it is more  difficult to observe the variation of $ \bar{\vtheta}_{\mW^{[l]}}$, whose vector $2$-norm at time $t=0$ is of order $m$, i.e., \[\Norm{\bar{\vtheta}_{\mW^{[l]}}(0)}_2=\Norm{\frac{\overline{\mW}^{[l]}(0)}{\sqrt{m}}}_{\mathrm{F}}\sim \fO(m),\]    
while the vector $2$-norm of $ \bar{\vtheta}_{\mW^{[1]}}$ and $\bar{\vtheta}_{\va}(t)$ are both of order $\sqrt{m}$, i.e., 
\[\Norm{\bar{\vtheta}_{\mW^{[1]}}(0)}_2=\Norm{\frac{\overline{\mW}^{[1]}(0)}{\sqrt{m}}}_{\mathrm{F}}\sim \fO(\sqrt{m}),\quad \Norm{\bar{\vtheta}_{\va}(0)}_2={\Norm{\frac{\Bar{\va}(0)}{\sqrt{m}}}_2}\sim \fO(\sqrt{m}).\]    
Finally, Theorem \ref{Theorem...repeaet} demonstrates that the stopping time $t^*=+\infty$, thus indicating that the following holds for all time $t>0$: For any $l\in[L+1]$,
\begin{equation*}
    \lambda_{\min}\left(\overline{\mG}^{[l]}(\vtheta(t))\right)\geq \frac{1}{2}\lambda_\fS,
\end{equation*}
and 
\begin{equation*}
   R_\fS(\vtheta(t))\leq  \exp\left(-\frac{1}{n} \left[\sum_{l=1}^{L+1}\frac{\kappa^2}{\alpha_l^2}\lambda_\fS\right] t \right) R_\fS(\vtheta(0)), 
\end{equation*}
and most importantly, the vector $\bar{\vtheta}$ containing all normalized parameters  varies slightly throughout the whole training process. Therefore, we remark that  the NN training dynamics fall  into the $\vtheta$-lazy regime, and 
\begin{equation}\left\{
\begin{aligned}
  \sup\limits_{t\in[0,+\infty)} \mathrm{RD}\left(  {\mW^{[l]}}\right)(t)&= 
 \frac{\Norm{\frac{\overline{\mW}^{[l]}(t)}{\sqrt{m}}-\frac{\overline{\mW}^{[l]}(0)}{\sqrt{m}}}_{\mathrm{F}}}{\Norm{\frac{\overline{\mW}^{[l]}(0)}{\sqrt{m}}}_{\mathrm{F}}}\xrightarrow{ m \to \infty } 0,~~l\in[L],\\
\sup\limits_{t\in[0,+\infty)}\mathrm{RD}(  {\va})(t)&=\frac{\Norm{\frac{\Bar{\va}(t)}{\sqrt{m}}-\frac{\Bar{\va}(0)}{\sqrt{m}}}_2}{\Norm{\frac{\Bar{\va}(0)}{\sqrt{m}}}_2} \xrightarrow{ m \to \infty } 0.
\end{aligned}  
\right.
\end{equation}
\subsection{Statement of the Theorem}\label{Subsection...Statement-of-Thm}
Theorem \ref{Theorem...repeaet} are rigorously stated as follows, and a sketch of   proof  for Theorem \ref{Theorem...repeaet} has been provided in Figure \ref{figure}. Moreover, its detailed proof can be found in Appendix \ref{B-subsection...Proof-of-Theorem}.
\begin{figure}[ht]
\centering
    \includegraphics[width=\textwidth]{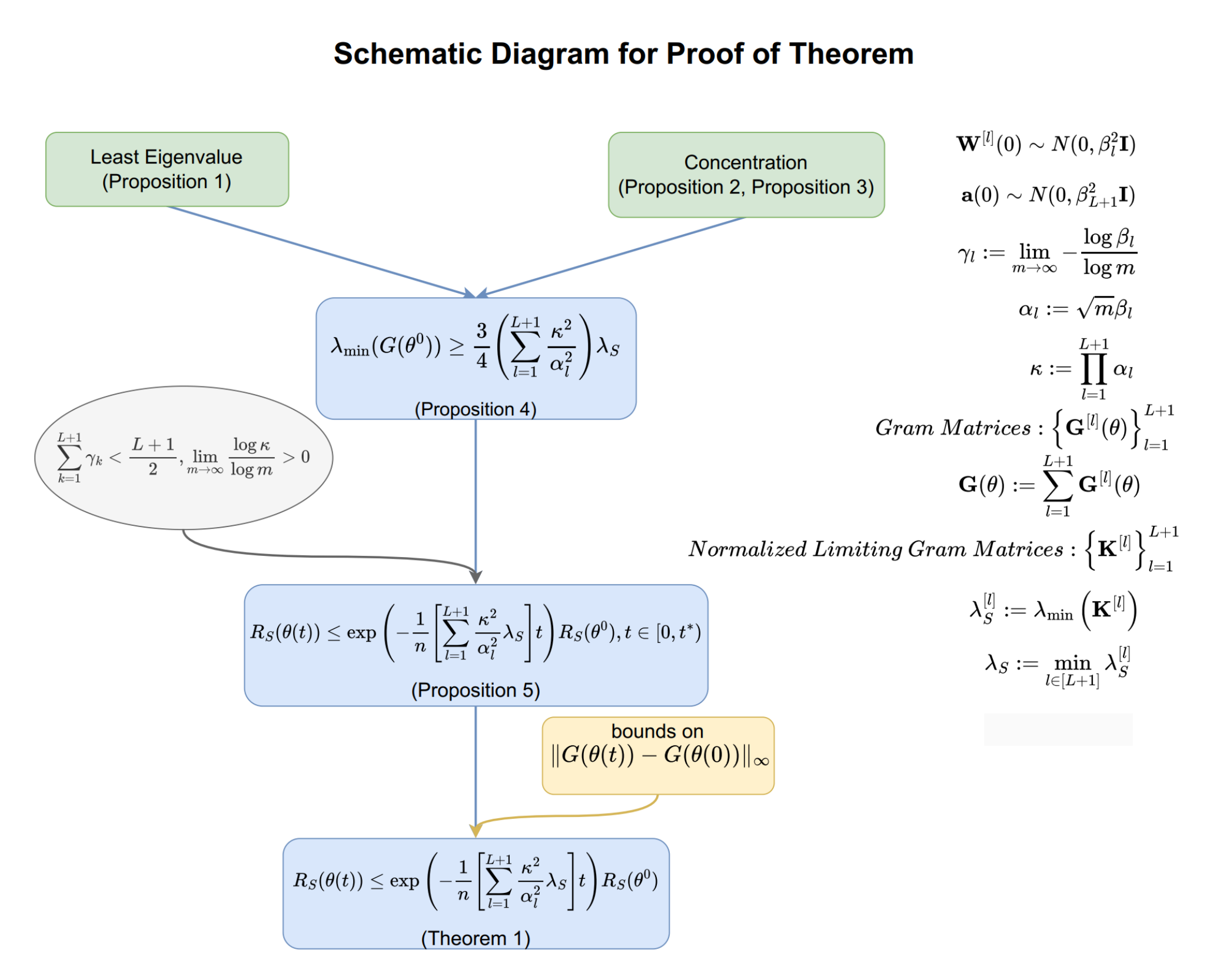}
    \caption{Sketch of proof for Theorem 1.}
    \label{figure}
\end{figure}
\begin{theorem} \label{Theorem...repeaet}
Suppose $\sigma(\cdot)$ satisfies conditions  in Assumption \ref{Assumption....Activation-Function}, and $\fS$     satisfies conditions  in  Assumption 
 \ref{Assumption...Data}, if $\sum_{k=1}^{L+1}\gamma_k<\frac{L+1}{2}$, i.e., $\lim_{m\to\infty} \frac{\log \kappa}{\log m}>0$, and
\[m=\max\left\{\left(\Omega\left(\frac{n^2}{\lambda_\fS^2}\log n\right)\right),  \Omega\left(\left( \frac{n}{  \lambda_\fS}\right)^{\frac{1}{\frac{L+1}{2}-\sum_{k=1}^{L+1}\gamma_k}}\right), \Omega\left(\left( \frac{n}{  \lambda_\fS}\right)^{\frac{2}{\frac{L+1}{2}-\sum_{k=1}^{L+1}\gamma_k}}\right)\right\},\]
where $\lambda_\fS=\min_{l\in[L+1]}    \lambda_{\min}\left(\mK^{[l]}\right)$, 
then with high probability,  for any time $t\geq 0$,
\begin{equation} 
R_\fS(\vtheta(t)) \leq \exp\left(- \frac{1}{n}\left[\left(\sum_{l=1}^{L+1} \frac{\kappa^2}{\alpha_l^2}\right) {\lambda_\fS}\right]t\right)R_\fS(\vtheta(0)).
\end{equation}
Moreover, for any  time $t\geq 0$,
\begin{equation} \label{Thm...eq...Theta-Lazy...Part-One}
\left\{
\begin{aligned} 
\lim_{m\to\infty}\frac{\Norm{\mW^{[l]}(t)-\mW^{[l]}(0)}_{\mathrm{F}}}{\Norm{ \mW^{[l]}(0)}_{\mathrm{F}}}=\frac{\Norm{\frac{\overline{\mW}^{[l]}(t)}{\sqrt{m}}-\frac{\overline{\mW}^{[l]}(0)}{\sqrt{m}}}_{\mathrm{F}}}{\Norm{\frac{\overline{\mW}^{[l]}(0)}{\sqrt{m}}}_{\mathrm{F}}}&=0,~~l\in[L],\\
\lim_{m\to\infty}\frac{\Norm{{{\va}(t)}-{{\va}(0)}}_2}{\Norm{{{\va}(0)}}_2}=\lim_{m\to\infty}\frac{\Norm{\frac{\Bar{\va}(t)}{\sqrt{m}}-\frac{\Bar{\va}(0)}{\sqrt{m}}}_2}{\Norm{\frac{\Bar{\va}(0)}{\sqrt{m}}}_2}&=0.
\end{aligned}
\right.
\end{equation}
\end{theorem}
\begin{remark}
The high probability  specifically reads
\begin{itemize}
\item Event (I): Concentration between the normalized Gram matrices $\left\{\overline{\mG}^{[l]}(\vtheta^0)\right\}_{l=1}^{L+1}$ and the    normalized limiting Gram matrices  $\left\{\mK^{[l]}\right\}_{l=1}^{L+1}$; 
\begin{equation}\label{Event+I}
\Prob(I)\leq 2n^2L \exp\left(-\frac{2^Lm\lambda_\fS^2}{16n^2}\right).
\end{equation}
\item Event (II): Initial bounds on $\left\{\Norm{\frac{\overline{\mW}^{[l]}}{\sqrt{m}}}_{2\to 2}\right\}_{l=1}^{L}$, $\left\{\Norm{\frac{\overline{\mW}^{[l]}}{\sqrt{m}}}_{\mathrm{F}}\right\}_{l=1}^{L}$ and $\Norm{\frac{\bar{\va}}{\sqrt{m}}}_2$;
\begin{equation}\label{Event+II}
\Prob(II)\leq 2(L-1) \exp\left(-\frac{m^2}{36}\right)+2\exp\left(-\frac{md}{36}\right)+2(L+1)\exp(-m).
\end{equation}
\end{itemize}
Based on \eqref{Event+I} and \eqref{Event+II}, by picking out the event I and II,  lower bound of the  high probability reads
\begin{align*}
1-  \Prob(I)-  \Prob(II)&\geq 1-2n^2L \exp\left(-\frac{2^Lm\lambda_\fS^2}{16n^2}\right)\\
&~~-\left[2(L-1) \exp\left(-\frac{m^2}{36}\right)+2\exp\left(-\frac{md}{36}\right)+2(L+1)\exp(-m)\right].
\end{align*}
We remark that the establishment of relation \eqref{Event+I} shall be traced back to Proposition \ref{A-prop...Concentration-on-L+1-th-Gram-Matrix} and Proposition \ref{A-prop...Concentration-on-l<L-th-Gram-Matrix}, and the  establishment of relation \eqref{Event+II} is the combined effort of Lemma \ref{A-Lemma...Operator-Norm-Random-Matrix}, Lemma \ref{A-Lemma......2-Norm-Chi-Square}, and Proposition \ref{A-prop..Upper-Bound-and-Lower-Bound-Initial-Parameter}.
\end{remark}
\section{Conclusions}\label{section....Conclusion}

In this paper,  we propose a unified  approach to characterize  the  theta-lazy regime  for the $L$-layer  NNs with a wide class of smooth activation functions. 
Our investigation reveals that the initial scale $\kappa$ of the output function is  the pivotal factor for  the propensity of the parameter set $\vtheta$ to  
persist in its initialization state throughout the entire training process.
This phenomenon holds true irrespective of the diverse array of initialization schemes employed, indicating the fundamental role played by the initial scale $\kappa$ in shaping the training dynamics of neural networks.  In conclusion, our investigation into the statistical properties of the numerous parameters  elucidates a macroscopic perspective  on  the study of neural networks.

One may inquire why the theta-lazy regime, as delineated in~\cite{luo2021phase,chen2023phase}, resides in the left-half plane of the line $\kappa=1$ rather than   $\kappa=0$, as observed in this paper.   It is imperative to note that our findings encompass the results in~\cite{luo2021phase,chen2023phase}.  However, in contrast to the normalization approach employed in~\cite{luo2021phase,chen2023phase}, wherein the parameters $\mathrm{vec} \left(\va, \mW\right)$ of the two-layer neural networks are normalized to $\mathrm{vec} \left(\bar{\va}, \overline{\mW}\right)$, we adopt a more general  strategy designed for multi-layer NNs  by  normalizing the parameters to $\mathrm{vec} \left(\frac{\bar{\va}}{\sqrt{m}}, \frac{\overline{\mW}}{\sqrt{m}}\right)$.  Consequently, in the scenario where $L=1$, corresponding to the case of  two-layer NNs, it is reasonable for the threshold of $\kappa$ to shift by a distance of $\frac{1}{2} + \frac{1}{2} = 1$.
We anticipate that this formalism can be readily extended to the analysis of $L$-layer  CNNs, where we hypothesize that the initial scale of the output function also serves as a critical factor for the persistence of weight parameters in their initialization state during training. However, we acknowledge the potential variability in this phenomenon for Residual Neural Networks~(ResNet), which may hinge on the specific choice of initialization schemes, owing to their distinctive skip-connection architecture.

In our next paper,  
we aim to delineate the distinct characteristics exhibited by NNs 
in the  area where $\lim_{m\to\infty} \frac{\log \kappa}{\log m}<0$,  and we  will provide a complete and detailed analysis for the transition   across the boundary. Furthermore,   we  aspire to   reveal the mechanism of initial condensation for multi-layer NNs,
and to identify the directions towards which the weight parameters condense. 
The synthesis of these two papers holds promise in furnishing a nuanced understanding of the implicit regularization effects engendered by weight initialization schemes, thereby serving as a   cornerstone upon which future works can be done to provide thorough characterization of  the dynamical behavior of general NNs  at each of the identified regime.  
 

\section*{Acknowledgments}
We would like to give special thanks to Prof. Jingwei Liang for his helpful discussions. This work is sponsored by the National Key R\&D Program of China  Grant No. 2022YFA1008200 (T. L.), the National Natural Science Foundation of China Grant No. 12101401 (T. L.), Shanghai Municipal Science and Technology Key Project No. 22JC1401500 (T. L.), Shanghai Municipal of Science and Technology Major Project No. 2021SHZDZX0102 (T. L.), and the HPC of School of Mathematical Sciences and the Student Innovation Center, and the Siyuan-1 cluster supported by the Center for High Performance Computing at Shanghai Jiao Tong University.
\bibliographystyle{plain}
\bibliography{Ref}
 \newpage
\appendix
\section{Full Rankness of   Gram Matrices}\label{Append...Section...Full-Rank-Gram-Matrices}
\subsection{Some Technical Lemmas}\label{Append...Subsection...Several-Lemmas}
\begin{lemma}\label{A-Lemma...Gram-Matrices-without-Derivative}
Suppose $\sigma(\cdot)$ satisfies conditions  in Assumption \ref{Assumption....Activation-Function},   consider   input data    
$\fZ:=\{\vz_1,\vz_2,\dots,\vz_n\}$ comprising $n$ non-parallel unit samples,  
given any $\eps>0$,    we  define
\begin{equation} 
\begin{aligned}
&\mG_\eps(\fZ):=\left[\mG_\eps(\fZ)_{ij}\right] :=\left[\Exp_{\vw\sim \fN\left(\vzero,  \mI_d \right)} \frac{\sigma(\eps\vw^{\T}\vz_i)}{\eps}\frac{\sigma(\eps\vw^{\T}\vz_j)}{\eps}\right],
    \end{aligned}
\end{equation}
then  there exists  constant $\lambda_\fZ>0$ independent of $\eps$, such that 
\begin{equation}
\lambda_{\min}\left(\mG_\eps(\fZ)\right)\geq \lambda_\fZ.
\end{equation}
\end{lemma}
\begin{proof} 
We observe that   as $\eps\to 0$, the limiting entries of $\mG_\eps(\fZ)$, denoted by $\mG_0(\fZ)$ read
\begin{align*}
\mG_0(\fZ)_{ij}&:=\lim_{\eps\to 0} \Exp_{\vw\sim \fN\left(\vzero,  \mI_d \right)} \frac{\sigma(\eps\vw^{\T}\vz_i)\sigma(\eps\vw^{\T}\vz_j)}{\eps^2} \\
&=\Exp_{\vw\sim \fN\left(\vzero,  \mI_d \right)}\left[(\vw^{\T}\vz_i)(\vw^{\T}\vz_j)\right]=\left<\vz_i, \vz_j\right>,    
\end{align*}
and  since $\vz_i$ and $\vz_j$ are non-parallel with each other, then $\lambda_0:=\lambda_{\min}\left(\mG_0(\fZ)\right)>0$, and it depends solely on the data.

In the case where $\eps\to \infty$, as the limit  reads 
\begin{align*}
   &\lim_{\eps\to 0} \frac{\sigma(\eps\vw^{\T}\vz_i)\sigma(\eps\vw^{\T}\vz_j)}{\eps^2}=(\vw^{\T}\vz_i)(\vw^{\T}\vz_j)\left(a \mathbf{1}_{\vw^{\T}\vz_i<0}+b\mathbf{1}_{\vw^{\T}\vz_i>0}\right)\left(a \mathbf{1}_{\vw^{\T}\vz_j<0}+b\mathbf{1}_{\vw^{\T}\vz_j>0}\right),
\end{align*}
then  as $\eps\to \infty$, the limiting entries of $\mG_\eps(\fZ)$, denoted by $\mG_{\infty}(\fZ)$, shares the character of the ReLU NTK. Specifically, if we choose $a=0$ and $b=\sqrt{2}$,  then
\begin{align*}
\mG_{\infty}(\fZ)_{ij}&:=\lim_{\eps\to \infty} \Exp_{\vw\sim \fN\left(\vzero,  \mI_d \right)} \frac{\sigma(\eps\vw^{\T}\vz_i)\sigma(\eps\vw^{\T}\vz_j)}{\eps^2} \\    
   &=\frac{1}{\pi} \left(\left<\vz_i, \vz_j\right>(\pi-\arccos{\left<\vz_i, \vz_j\right>})+\sqrt{1-\left<\vz_i, \vz_j\right>^2}\right),
\end{align*}
and  since $\vz_i$ and $\vz_j$ are non-parallel with each other, then $\lambda_\infty:=\lambda_{\min}\left(\mG_\infty(\fZ)\right)>0$,  and it depends solely on the data and activation function.

Moreover, as  $\lambda_{\min}\left(\mG_\eps(\fZ)\right)$ is a continuous function with respect to the entries of $\mG_\eps(\fZ)$, we conclude that there exists $\eps_1<1$, such that for any $\eps\in(0,\eps_1]$,
$ 
\lambda_{\min}\left(\mG_\eps(\fZ)\right)\geq \frac{1}{2}\lambda_0.
$ 
Similarly, there also exists $M_1>1$, such that for any $\eps\in[M_1, \infty)$,
$\lambda_{\min}\left(\mG_\eps(\fZ)\right)\geq \frac{1}{2}\lambda_\infty.$
Finally,  Lemma $\mathrm{F.1.}$   in Du et al.~\cite{Du2018Gradient} guarantees that: For any $\eps\in[\eps_1, M_1]$,
$\lambda_{\min}\left(\mG_\eps(\fZ)\right)>0,$  and  by choosing
\[
\lambda_\fZ:=\min\left\{\frac{1}{2}\lambda_0, \frac{1}{2}\lambda_\infty, \min_{\eps\in[\eps_1, M_1]}\lambda_{\min}\left(\mG_\eps(\fZ)\right) \right\},
\]
we finish our proof.
\end{proof}
 \begin{lemma}\label{A-Lemma...Hadamard-Product-is-Positive-Definite}
  If $\mA:=[a_{ij}]_{n\times n}$ is semi-positive definite, and $\mB:=[b_{ij}]_{n\times n}$ is  positive definite,  then   
\begin{equation}\label{A-Lemma...H-P-is-P-D...eq...Least-Eigenvalue-Lower-Bound}
    \lambda_{\min}(\mA \odot \mB) \geq \left(\frac{n-1}{n}\right)^{\frac{n-1}{2}} \left(\prod_{i=1}^n {a_{ii}} \right)\mathrm{det}( \mB).
\end{equation}
\end{lemma}
\begin{proof}
By application of  Oppenheim inequality~\cite[Theorem 7.8.16]{horn2012matrix}, the following holds
\[
\mathrm{det}(\mA \odot \mB)\geq \left(\prod_{i=1}^n {a_{ii}} \right)\mathrm{det}( \mB).
\]
Moreover, Hong and Pan~\cite{hong1992lower} provided a lower bound for the least eigenvalue, wherein  for a semi-positive definite matrix $\mC=[c_{ij}]_{n\times n}$, its least eigenvalue yields the following estimate
\[
\lambda_{\min}(\mC) \geq \left(\frac{n-1}{n}\right)^{\frac{n-1}{2}} \mathrm{det}(\mC),
\]
thus we obtain that 
\[
\lambda_{\min}(\mA \odot \mB) \geq \left(\frac{n-1}{n}\right)^{\frac{n-1}{2}} \left(\prod_{i=1}^n {a_{ii}} \right)\mathrm{det}( \mB).
\]
\end{proof}
\begin{lemma}\label{A-Lemma...Matrix-Norm-and-Entry-Norm}
Suppose $\sigma(\cdot)$ satisfies conditions  in Assumption \ref{Assumption....Activation-Function},  and we assume   that there exists some constant $c>0$ and $\gamma>0$, such that 
\[
\mA_1:=\left[\begin{array}{cc}a_1^2 & \rho_1 a_1 b_1 \\ \rho_1 a_1 b_1 & b_1^2\end{array}\right],~~\mA_2:=\left[\begin{array}{cc}a_2^2 & \rho_2 a_2 b_2\\ \rho_2 a_2 b_2 & b_2^2\end{array}\right],
\]
where for any $l\in[2]$,
\begin{equation}\label{A-Lemma...eq...Diagonal}
\frac{1}{c} \leq \min \{a_l, b_l\},~~\max \{a_l, b_l\} \leq c,
\end{equation}
and 
\begin{equation}\label{A-Lemma...M-N-and-E-N...eq...Covariance}
-1+\gamma\leq \rho_l\leq  1-\gamma,
\end{equation}
then for some fixed  $\eps>0$, as we define
\begin{align*}
F(\mA)&:=\mathbb{E}_{(u, v) \sim \fN(\mathbf{0}, \mA)}\left[\frac{\sigma(\eps u)}{\eps}\frac{\sigma(\eps v)}{\eps}\right], \\
G(\mA)&:=\mathbb{E}_{(u, v) \sim \fN(\mathbf{0}, \mA)}\left[ {\sigma^{(1)}(\eps u)} {\sigma^{(1)}(\eps v)} \right],
\end{align*}
 the following holds
\begin{equation}\label{A-lemma...eq...Matrix-Entry-is-Lipschitz}
\max\left\{\Abs{F(\mA_1)-F(\mA_2)},\Abs{G(\mA_1)-G(\mA_2)}\right\} \leq  C\Norm{\mA_1-\mA_2}_{\infty},  
\end{equation}
for some constant $C>0$ that depends sorely on $c$, $\gamma$,  and $\sigma(\cdot)$, and independent of $\eps$.
\end{lemma}
\begin{proof}
Let 
\[
\mA =\left[\begin{array}{cc}a^2 & \rho ab \\ \rho a b & b^2\end{array}\right]
\]
with 
\[
\min \{a_1, a_2\}\leq a \leq \max \{a_1, a_2\},~~\min \{b_1, b_2\}\leq b \leq \max \{b_1, b_2\}
\]
and
\[
-1+\gamma\leq \rho\leq 1-\gamma,
\]
then 
\begin{align*}
F(\mA)&=\mathbb{E}_{(z_1, z_2) \sim \fN(\mathbf{0}, \mB)}\left[\frac{\sigma(\eps az_1)}{\eps}\frac{\sigma(\eps bz_2)}{\eps}\right],\\
    G(\mA)&=\mathbb{E}_{(z_1, z_2) \sim \fN(\mathbf{0}, \mB)}\left[ {\sigma^{(1)}(\eps az_1)}  {\sigma^{(1)}(\eps bz_2)} \right],
\end{align*}
with 
\[
\mB:= \left[\begin{array}{cc}1 & \rho   \\ \rho     & 1\end{array}\right].
\]
Then, we obtain that 
\begin{align*}
\Abs{\frac{\partial F(\mA)}{\partial a}}&=   \Abs{\mathbb{E}_{(z_1, z_2) \sim \fN(\mathbf{0}, \mB)}\left[ \sigma^{(1)}(\eps az_1)z_1 \frac{\sigma(\eps bz_2)}{\eps}\right]}\\ 
&\leq \left[\Exp _{z_1\sim \fN(0, 1)} \left(\sigma^{(1)}(\eps az_1)\right)^2z_1^2\right]^{1/2} \left[\Exp _{z_2\sim \fN(0, 1)} \left(\frac{\sigma(\eps bz_2)}{\eps}\right)^2\right]^{1/2}\\
&\leq C(\Exp _{z_1\sim \fN(0, 1)}  z_1^2)^{1/2} \left[\Exp _{z_2\sim \fN(0, 1)} \left( \sigma^{(1)}(\eps\theta bz_2)\right)^2b^2z_2^2  \right]^{1/2} \leq C^2b \leq C^2c,
\end{align*}
and by similar reasoning,   $\Abs{\frac{\partial F(\mA)}{\partial b}}\leq C^2c$. As for the estimate of $\Abs{\frac{\partial F(\mA)}{\partial \rho}}$, we observe that   
\begin{equation}\label{temp-1...pf}
\begin{aligned}
 \Abs{F(\mA)}&=\Abs{\mathbb{E}_{(z_1, z_2) \sim \fN(\mathbf{0}, \mB)}\left[\frac{\sigma(\eps az_1)}{\eps}\frac{\sigma(\eps bz_2)}{\eps}\right]}   \\
&\leq \left[\Exp _{z_1\sim \fN(0, 1)} \left(\frac{\sigma(\eps az_1)}{\eps}\right)^2\right]^{1/2}\left[\Exp _{z_2\sim \fN(0, 1)} \left(\frac{\sigma(\eps az_2)}{\eps}\right)^2\right]^{1/2}\leq C^2c^2,
\end{aligned}
\end{equation}
and for any $\eps>0$, $\frac{\sigma(\eps az)}{\eps}$ and $\frac{\sigma(\eps bz)}{\eps}$ belong to the Gaussian function space 
\begin{equation}\label{A-proof...eq...Gaussian-Function-Space}
\fL^2\left(\sR, \exp\left(-\frac{z^2}{2}\right)\right):=\left\{f ~\middle|~\int_{\sR} f(z)\exp\left(-\frac{z^2}{2}\right)\D z<\infty\right\},
\end{equation}
with   the Hermite expansion of $\frac{\sigma(\eps az_1)}{\eps}$ reads
\[
\frac{\sigma(\eps az_1)}{\eps}:=\sum_{k=0}^{\infty}\alpha_kH_k(z_1),
\]
while the  Hermite expansion of $\frac{\sigma(\eps bz_2)}{\eps}$ reads
\[
\frac{\sigma(\eps bz_2)}{\eps}:=\sum_{k=0}^{\infty}\beta_kH_k(z_2),
\]
where $\{H_k(\cdot)\}_{k=0}^{\infty}$ are the probabilist's Hermite polynomials, and $\{\alpha_k\}_{k=0}^{\infty}$ and $\{\beta_k\}_{k=0}^{\infty}$ depend on the choice of $\eps$. Then, we obtain that 
\begin{align*}
F(\mA)&=    \mathbb{E}_{(z_1, z_2) \sim \fN(\mathbf{0}, \mB)}\left[\frac{\sigma(\eps az_1)}{\eps}\frac{\sigma(\eps bz_2)}{\eps}\right]\\
&=\Exp\left[ \left(\sum_{k=0}^{\infty}\alpha_kH_k(z_1)\right)\left(\sum_{l=0}^{\infty}\beta_lH_l(z_2)\right)\right]\\
&=\Exp\left[  \sum_{k,l=0}^{\infty}\alpha_k \beta_l H_k(z_1)  H_l(z_2) \right] = \sum_{k,l=0}^{\infty}\alpha_k \beta_l k!\left(\Exp\left[z_1z_2\right]\right)^k\delta_k^l =\sum_{k=0}^{\infty}\alpha_k \beta_k k!\rho^k.
\end{align*}
As is shown by relation \eqref{temp-1...pf},   the power series $F(\mA)=\sum_{k=0}^{\infty}\alpha_k \beta_k k!\rho^k$ is absolute convergent for any $\rho\in(-1,1)$, hence by differentiation, the power series 
\[
\sum_{k=0}^{\infty}\alpha_{k+1} \beta_{k+1} (k+1)!(k+1)\rho^k
\]
is absolute convergent for any $\rho\in(-1,1)$, and it converges uniformly to  $\frac{\partial F(\mA)}{\partial \rho}$, thus  we obtain that  for any $\rho\in[-1+\gamma,1-\gamma]$, $\Abs{\frac{\partial F(\mA)}{\partial \rho}}\leq C.$

As for $\mG(\mA)$,  we remark that for any $\eps>0$,  
\begin{equation} 
\Abs{G(\mA)}=\Abs{\mathbb{E}_{(z_1, z_2) \sim \fN(\mathbf{0}, \mB)}\left[ {\sigma^{(1)}(\eps az_1)}  {\sigma^{(1)}(\eps bz_2)}\right]}\leq C^2,
\end{equation} 
as we send $\eps\to 0$, 
\[
\lim_{\eps\to 0} G(\mA)=\lim_{\eps\to 0} \mathbb{E}_{(z_1, z_2) \sim \fN(\mathbf{0}, \mB)}\left[ {\sigma^{(1)}(\eps az_1)}  {\sigma^{(1)}(\eps bz_2)}\right]= \mathbb{E}_{(z_1, z_2) \sim \fN(\mathbf{0}, \mB)}\left[ {\sigma^{(1)}(0)}  \right]^2=1,
\]
and as we send $\eps\to \infty$, 
\begin{align*}
\lim_{\eps\to \infty} G(\mA)&=\lim_{\eps\to \infty} \mathbb{E}_{(z_1, z_2) \sim \fN(\mathbf{0}, \mB)}\left[ {\sigma^{(1)}(\eps az_1)}  {\sigma^{(1)}(\eps bz_2)}\right]\\
&=\mathbb{E}_{(z_1, z_2) \sim \fN(\mathbf{0}, \mB)}\left[\left(a \mathbf{1}_{z_1<0}+b\mathbf{1}_{z_1>0}\right)\left(a \mathbf{1}_{z_2<0}+b\mathbf{1}_{z_2>0}\right)\right]\\
&=a^2\mathbb{E}_{(z_1, z_2) \sim \fN(\mathbf{0}, \mB)}\mathbf{1}_{z_1<0, z_2<0}+ab\mathbb{E}_{(z_1, z_2) \sim \fN(\mathbf{0}, \mB)}\mathbf{1}_{z_1<0, z_2>0}\\
&~~+ab\mathbb{E}_{(z_1, z_2) \sim \fN(\mathbf{0}, \mB)}\mathbf{1}_{z_1>0, z_2<0}+b^2\mathbb{E}_{(z_1, z_2) \sim \fN(\mathbf{0}, \mB)}\mathbf{1}_{z_1>0, z_2>0}\\
&=\frac{a^2+b^2}{2\pi}\left(\pi-\arccos{\rho}\right)+\frac{ab}{\pi}\arccos{\rho},
\end{align*}
hence by dominated convergence theorem, as $-1+\gamma\leq \rho\leq 1-\gamma,$  we obtain that 
\begin{align*}
\lim_{\eps\to 0 } \Abs{G(\mA_1)-G(\mA_2)}&=0,\\
\lim_{\eps\to \infty} \Abs{G(\mA_1)-G(\mA_2)}&\leq C\Abs{\rho_1-\rho_2}\leq C\Norm{\mA_1-\mA_2}_{\infty}.
\end{align*}
Since $G(\mA)$ is continuous in $\eps$, then there exists $\eps_1<1$, such that for any $\eps\in(0,\eps_1]$,
\begin{equation}\label{A-proof...eq...temp1}
\Abs{G(\mA_1)-G(\mA_2)}\leq \Norm{\mA_1-\mA_2}_{\infty},\end{equation}
 by similar reasoning,  there also exists $M_1>1$, such that for any $\eps\in[M_1, \infty)$,
\begin{equation}\label{A-proof...eq...temp2}\Abs{G(\mA_1)-G(\mA_2)}\leq 2C\Norm{\mA_1-\mA_2}_{\infty}.\end{equation}
Finally, for any $\eps\in[\eps_1, M_1]$, 
\begin{align*}
\Abs{\frac{\partial G(\mA)}{\partial a}}&=\Abs{\mathbb{E}_{(z_1, z_2) \sim \fN(\mathbf{0}, \mB)}\left[ {\sigma^{(2)}(\eps az_1)} \eps z_1 {\sigma^{(1)}(\eps bz_2)} \right]}\\
&\leq C\Abs{\mathbb{E}_{z_1 \sim \fN(0, 1)}\eps z_1}\leq 2CM_1,
\end{align*}
and 
\begin{align*}
\Abs{\frac{\partial G(\mA)}{\partial b}}&=\Abs{\mathbb{E}_{(z_1, z_2) \sim \fN(\mathbf{0}, \mB)}\left[ {\sigma^{(1)}(\eps az_1)}  {\sigma^{(1)}(\eps bz_2)\eps z_1} \right]}\\
&\leq C\Abs{\mathbb{E}_{z_2 \sim \fN(0, 1)}\eps z_2}\leq 2CM_1.
\end{align*}
Moreover, for any $\eps>0$, ${\sigma^{(1)}(\eps az_1)} $ and ${\sigma^{(1)}(\eps bz_2)} $ belong to the Gaussian function space demonstrated in \eqref{A-proof...eq...Gaussian-Function-Space}, then by similar reasoning,  as we expand $G(\mA)$ in the power series of $\rho$, since the power is absolute convergent for any $\rho\in(-1,1)$, then its  differentiation 
is also absolute convergent for any $\rho\in(-1,1)$, and it converges uniformly to  $\frac{\partial G(\mA)}{\partial \rho}$. Thus,  we obtain that  for any $\rho\in[-1+\gamma,1-\gamma]$, $\Abs{\frac{\partial G(\mA)}{\partial \rho}}\leq C.$   

We remark that as relation \eqref{A-proof...eq...temp1} and \eqref{A-proof...eq...temp2} partially finish the proof of relation \eqref{A-lemma...eq...Matrix-Entry-is-Lipschitz}, therefore, we focus on   $F(\mA)$ for any $\eps>0$, and $G(\mA)$ for any $\eps\in[\eps_1, M_1]$.   There exists $C>0$  independent of  $a$, $b$ and $\rho$, such that   
\begin{align*}
\Abs{\frac{\partial F(\mA)}{\partial a}}&\leq C,~~\Abs{\frac{\partial F(\mA)}{\partial b}}\leq C,~~\Abs{\frac{\partial F(\mA)}{\partial \rho}}\leq C, \\
\Abs{\frac{\partial G(\mA)}{\partial a}}&\leq C,~~\Abs{\frac{\partial G(\mA)}{\partial b}}\leq C,~~\Abs{\frac{\partial G(\mA)}{\partial \rho}}\leq C, 
\end{align*}
  and we proceed to bound $\nabla_{\mA} F(\mA)$ and $\nabla_{\mA} G(\mA)$.  Since   for the set 
\[
\fM:=\left\{ (a,b,\rho) :  \frac{1}{c}\leq a \leq c,~~\frac{1}{c}\leq b \leq c,~~-1+\gamma\leq \rho\leq 1-\gamma \right\},
\]
there exists a   $\fC^{\infty}$-mapping    $\vphi(\cdot):\fM \to \sR^3$,     where \[\vphi(a,b,\rho)=[a^2, b^2, \rho ab]^\T,\]
as the Jacobian of $\vphi$ at any point $\vz:=(a, b, \rho)\in\fM$ reads
\[
J_{\vphi}(\vz)=\mathrm{det} \left[\frac{\partial\vphi(\vz)}{\partial (a,b,\rho)}\right]  =\mathrm{det} \left[\begin{array}{ccc}2a & 0&0 \\ 
0&2b&0\\   \rho b&\rho a &ab \end{array}\right]=4a^2b^2>0,
\]
indicating that  $\vphi(\cdot)$ is   a $1$-$1$ mapping on the domain $\fM$. Moreover, since we have 
\[
\nabla_{\mA} F(\mA)=\left[\frac{\partial\vphi(\vz)}{\partial (a,b,\rho)}\right] ^{-1}\nabla_{\vz} F(\mA),~~\nabla_{\mA} G(\mA)=\left[\frac{\partial\vphi(\vz)}{\partial (a,b,\rho)}\right] ^{-1}\nabla_{\vz} G(\mA),
\]
and as $\max\left\{\Norm{\nabla_{\vz} F(\mA)}_{\infty}, \Norm{\nabla_{\vz} G(\mA)}_{\infty}\right\}\leq  C$,  then the following holds 
\begin{equation*}
\begin{aligned}
 \max\left\{\Norm{\nabla_{\mA} F(\mA)}_2, \Norm{\nabla_{\mA} G(\mA)}_2\right\}&\leq   \Norm{\left[\frac{\partial\vphi(\vz)}{\partial (a,b,\rho)}\right]^{-1}}_{2\to 2} \max\left\{\Norm{\nabla_{\vz} F(\mA)}_{2}, \Norm{\nabla_{\vz} G(\mA)}_{2}\right\}\\
&\leq \max\left\{\frac{1}{2a},~~\frac{1}{2b},~~\frac{1}{ab}\right\} \sqrt{3} C \leq cC,
\end{aligned}
\end{equation*}
and as all norms are equivalent in finite dimensional vector spaces, we finish the proof.
\end{proof}
\noindent
Our next lemma aims to  demonstrate   validity of    relation \eqref{A-Lemma...eq...Diagonal} imposed in Lemma \ref{A-Lemma...Matrix-Norm-and-Entry-Norm}.

\begin{lemma}\label{A-Lemma...Second-Moment-Boundedness}
   Suppose $\sigma(\cdot)$ satisfies conditions  in Assumption \ref{Assumption....Activation-Function},  and given 
   \begin{equation}\label{1122}
      \frac{1}{c}\leq x_0\leq c,  
   \end{equation}
   for some constant $c>0$,
   then for any $\eps>0$, there exist some   constants $\mu_1, \mu_2>0$ satisfying
\begin{equation}\label{A-Lemma...eq...Second-Moment-Bound}
\left\{
\begin{aligned}
 \mu_1  x_0^{2}&\leq  \Exp_{u\sim\fN(0,x_0^2)}\left[\frac{\sigma(\eps u)}{\eps}\right]^{2}\leq \mu_2 x_0^{2},\\
 \mu_1   &\leq  \Exp_{u\sim\fN(0,x_0^2)}\left[ \sigma^{(1)}(\eps u) \right]^{2}\leq \mu_2, 
 \end{aligned}
 \right.
 \end{equation}
where $\mu_1$  and $\mu_2$  are  independent of $\eps$. 
\end{lemma} 
\begin{proof}
Directly from Assumption \ref{Assumption....Activation-Function}, we obtain that 
\[
\frac{\sigma(\eps u)}{\eps}\leq Cu,~~\sigma^{(1)}(\eps u)\leq C,
\]
where $C$ is the Lipschitz constant in  Assumption \ref{Assumption....Activation-Function},
  by taking expectations 
\begin{align*}
 \Exp_{u\sim\fN(0,x_0^2)}\left[\frac{\sigma(\eps u)}{\eps}\right]^{2}&\leq C^2 \Exp_{u\sim\fN(0,x_0^2)}u^2=C^2x_0^2,\\
 \Exp_{u\sim\fN(0,x_0^2)}\left[\sigma^{(1)}(\eps u)\right]^{2}&\leq C^2, 
\end{align*}
as we set $\mu_2:=C$, we partially finish the proof for relation \eqref{A-Lemma...eq...Second-Moment-Bound}.  

 As for $\mu_1$,
we  define   auxiliary functions $F(\cdot;x_0):[0, +\infty)\to\sR$, 
\begin{equation}
    F(\eps;x_0):=\left\{\begin{array}{ll}\Exp_{u\sim\fN(0,x_0^2)}\left[\frac{\sigma(\eps u)}{\eps}\right]^2, & \eps\neq 0,  \\
    \Exp_{u\sim\fN(0,x_0^2)}u^2,& \eps=0,\end{array}\right.
\end{equation}
where  $x_0>0$ is a fixed constant. For function $F(\cdot;x_0)$,  as $\eps \to 0^+$, 
\[
\lim_{\eps \to 0^+}\frac{\sigma(\eps u)}{\eps}=\lim_{\eps \to 0^+} \frac{\left[\sigma(0)+\sigma^{(1)}(0) \eps u+o(\eps)\right]}{\eps}=u,
\]
  by taking expectation on both sides, we obtain that 
\begin{equation}
   \lim_{\eps \to 0^+} \Exp_{u\sim\fN(0,x_0^2)}\left[\frac{\sigma(\eps u)}{\eps}\right]^2=\Exp_{u\sim\fN(0,x_0^2)}u^2=x_0^2.
\end{equation}
 $F(\cdot;x_0)$ is continuous on $[0, +\infty)$, and there exists $\eps_1<1$, such that for any $\eps\in(0,\eps_1]$,
\begin{equation}
\begin{aligned}
\Exp_{u\sim\fN(0,x_0^2)}\left[\frac{\sigma(\eps u)}{\eps}\right]^2\geq \frac{1}{2}x_0^2.
\end{aligned}
\end{equation}
Moreover, as $\eps \to \infty$, 
\[
\lim_{\eps \to+ \infty}\frac{\sigma(\eps u)}{\eps}=\left\{\begin{array}{ll}a, & u<0,  \\ 0, &u=0,\\
   b,& u>0,\end{array}\right.
\]
then by taking expectation on both sides, we obtain that 
\begin{equation}
\begin{aligned}
   \lim_{\eps \to +\infty} \Exp_{u\sim\fN(0,x_0^2)}\left[\frac{\sigma(\eps u)}{\eps}\right]^2&=\Exp_{u\sim\fN(0,x_0^2)}\left[a^2u^2\mathbf{1}_{u<0}\right]+\Exp_{u\sim\fN(0,x_0^2)}\left[b^2u^2\mathbf{1}_{u>0}\right]\\
&=\frac{a^2 x_0^2}{2}+\frac{b^2 x_0^2}{2}=\left(\frac{a^2+b^2}{2}\right)x_0^2.
\end{aligned}
\end{equation}
By similar reasoning,  there also exists $M_1>1$, such that for any $\eps\in[M_1, \infty)$,
\begin{equation}
\begin{aligned}
\Exp_{u\sim\fN(0,x_0^2)}\left[\frac{\sigma(\eps u)}{\eps}\right]^2\geq \left(\frac{a^2+b^2}{4}\right)x_0^2.
\end{aligned}
\end{equation}
Finally,   as $\sigma^{(1)}(0)=1$ and $\sigma^{(1)}(\cdot)$ is continuous on $\sR$,  there exists $\widetilde{M}>0$, such that for any $s\in[-\widetilde{M}, \widetilde{M}]$, $\sigma^{(1)}(s)\geq \frac{1}{2}.$
Then, for any $u\in\left[-\frac{\widetilde{M}}{M_1}, \frac{\widetilde{M}}{M_1}\right]$ and  $\eps\in[\eps_1, M_1]$,
\begin{equation}
\begin{aligned}
\Exp_{u\sim\fN(0,x_0^2)}\left[\frac{\sigma(\eps u)}{\eps}\right]^2&\geq \Exp_{u\sim\fN(0,x_0^2)}\left[\frac{\sigma(\eps u)}{\eps}\mathbf{1}_{-\frac{\widetilde{M}}{M_1}\leq u\leq  \frac{\widetilde{M}}{M_1}}\right] ^2\\
&\geq \frac{1}{4}\Exp_{u\sim\fN(0,x_0^2)}u^2\mathbf{1}_{-\frac{\widetilde{M}}{M_1}\leq u\leq  \frac{\widetilde{M}}{M_1}}\\
& \geq  \left(\frac{1}{4}\Exp_{z\sim\fN(0,1)} z^2\mathbf{1}_{-\frac{\widetilde{M}}{cM_1}\leq z\leq  \frac{\widetilde{M}}{cM_1}}\right)x_0^2.
\end{aligned}
\end{equation}
We define another auxiliary function $G(\cdot;x_0):[0, +\infty)\to\sR$,
\begin{equation}
    G(\eps;x_0):= \left\{\begin{array}{ll}\Exp_{u\sim\fN(0,x_0^2)}\left[ \sigma^{(1)}(\eps u) \right]^{2}, & \eps\neq 0,  \\
    1,& \eps=0,\end{array}\right.
\end{equation}
where  $x_0>0$ is a fixed constant. 
For function $G(\cdot;x_0)$,  as $\eps \to 0^+$, 
\[
\lim_{\eps \to 0^+}\sigma^{(1)}(\eps u)=\lim_{\eps \to 0^+} \sigma^{(1)}(0)=1,
\]
  by taking expectation on both sides, we obtain that 
\begin{equation}
   \lim_{\eps \to 0^+} \Exp_{u\sim\fN(0,x_0^2)}\left[ \sigma^{(1)}(\eps u) \right]^{2}=\Exp_{u\sim\fN(0,x_0^2)}1=1.
\end{equation}
 $G(\cdot;x_0)$ is continuous on $[0, +\infty)$, and there exists $\eps_1<1$, such that for any $\eps\in(0,\eps_1]$,
\begin{equation}
\begin{aligned}
\Exp_{u\sim\fN(0,x_0^2)}\left[ \sigma^{(1)}(\eps u) \right]^{2}\geq \frac{1}{2}.
\end{aligned}
\end{equation}
Moreover, as $\eps \to \infty$, 
\[
\lim_{\eps \to+ \infty}\sigma^{(1)}(\eps u)=\left\{\begin{array}{ll}a, & u<0,  \\ 1, &u=0,\\
   b,& u>0,\end{array}\right.
\]
then by taking expectation on both sides, we obtain that 
\begin{equation}
\begin{aligned}
   \lim_{\eps \to +\infty} \Exp_{u\sim\fN(0,x_0^2)}\left[\sigma^{(1)}(\eps u)\right]^2&=\Exp_{u\sim\fN(0,x_0^2)}\left[a^2\mathbf{1}_{u<0}\right]+\Exp_{u\sim\fN(0,x_0^2)}\left[b^2\mathbf{1}_{u>0}\right]\\
&=\frac{a^2+b^2}{2}.
\end{aligned}
\end{equation}
Then, by similar reasoning,  there also exists $M_1>1$, such that for any $\eps\in[M_1, \infty)$,
\begin{equation}
\begin{aligned}
\Exp_{u\sim\fN(0,x_0^2)}\left[ \sigma^{(1)}(\eps u) \right]^{2}\geq   \frac{a^2+b^2}{4}.
\end{aligned}
\end{equation}
Finally,   as $\sigma^{(1)}(0)=1$ and $\sigma^{(1)}(\cdot)$ is continuous on $\sR$,  there exists $\widetilde{M}>0$, such that for any $s\in[-\widetilde{M}, \widetilde{M}]$, $\sigma^{(1)}(s)\geq \frac{1}{2}.$
Then, for any $u\in\left[-\frac{\widetilde{M}}{M_1}, \frac{\widetilde{M}}{M_1}\right]$ and  $\eps\in[\eps_1, M_1]$,
\begin{equation}
\begin{aligned}
\Exp_{u\sim\fN(0,x_0^2)}\left[ \sigma^{(1)}(\eps u) \right]^{2}&\geq \Exp_{u\sim\fN(0,x_0^2)}\left[\sigma^{(1)}(\eps u)\mathbf{1}_{-\frac{\widetilde{M}}{M_1}\leq u\leq  \frac{\widetilde{M}}{M_1}}\right] ^2\\
&\geq \frac{1}{4}\Exp_{u\sim\fN(0,x_0^2)} \mathbf{1}_{-\frac{\widetilde{M}}{M_1}\leq u\leq  \frac{\widetilde{M}}{M_1}}\\
&\geq \left(\frac{1}{4}\Exp_{z\sim\fN(0,1)}  \mathbf{1}_{-\frac{\widetilde{M}}{cM_1}\leq z\leq  \frac{\widetilde{M}}{cM_1}}\right),
\end{aligned}
\end{equation}
Therefore,  we set
\[\mu_1:=\min\left\{\frac{1}{2}, \frac{a^2+b^2}{4}, \left(\frac{1}{4}\Exp_{z\sim\fN(0,1)} z^2\mathbf{1}_{-\frac{\widetilde{M}}{cM_1}\leq z\leq  \frac{\widetilde{M}}{cM_1}}\right)\right\}.\]  
\end{proof}
\noindent
We state  two lemmas concerning  the operator norm of a random matrix, and   the vector $2$-norm of a chi-square distribution, whose proofs can be found in~\cite{Du2018Gradient,Yuqing2022ResNet}.
\begin{lemma}\label{A-Lemma...Operator-Norm-Random-Matrix}
Given   $\mW\in\sR^{m\times cm}$ with i.i.d. entry $\mW_{i,j}\sim \fN(0, 1),$ then  for any $t>0$, 
\begin{equation}
\Prob\left(\Norm{\mW}_{2\to 2}\geq (1+\sqrt{c})\sqrt{m}+t\right) \leq 2\exp\left(-C_0t^2\right),
\end{equation}
for some absolute constant  $C_0>0$.
\end{lemma} 
\begin{lemma}\label{A-Lemma......2-Norm-Chi-Square}
Given   $\va\in \sR^{m}$ with i.i.d. entry $\va_{i}\sim \fN(0, 1),$
then, for any $t>0$, 
\begin{equation} 
    \Prob\left(\Norm{\va}_2 \geq \sqrt{m} +t\right) \leq \exp({-C_0t^2}),
\end{equation}
for some absolute constant  $C_0>0$.
\end{lemma}
\noindent
We  shall introduce the sub-exponential norm~\cite{Vershynin2010Introduction}  and the  Bernstein's Inequality~\cite{vershynin2018high}.
\begin{definition}
The sub-exponential norm of a random variable $\rX$ is defined as
\begin{equation}
\Norm{\rX}_{\psi} := \inf\left\{s>0 \mid \Exp_{\rX}\left[\exp\left(\frac{\abs{\rX}}{s}\right)\right]\leq 2\right\}.
\end{equation}
\end{definition}   
\noindent
    In particular, we denote $\rY:=\chi^2(d)$ as   a   chi-square distribution with  $d$ degrees of freedom, and its 
   sub-exponential norm by  $C_{\psi,d}:=\norm{\rY}_{\psi},$  and we remark that 
\[\frac{2}{1-2^{-\frac{2}{d}}}\leq C_{\psi,d}<3.\]
\begin{theorem}\label{A-Thm...Bernstein-Inequality}
Let $\{\rX_k\}_{k=1}^m$ be i.i.d.\ sub-exponential random variables satisfying  $\Exp\rX_1=\mu,$ then for any $\eta> 0$, we have
\begin{equation*}
\Prob\left(\Abs{\frac{1}{m}\sum_{k=1}^m\rX_k-\mu}\geq \eta\right)\leq 2\exp\left(-C_0 m \min\left(\frac{\eta^2}{\norm{\rX_1}^2_{\psi}},\frac{\eta}{\norm{\rX_1}_{\psi}}\right)\right),
\end{equation*}
for some absolute constant $C_0$.    
\end{theorem}
\section{Detailed Proofs on Several  Propositions}\label{Appendix-B}
\subsection{Proof of Proposition \ref{B-prop...Parameter-no-Movement-2-Norm}}\label{B-subsection...Proof-of-Proposition-Parameter-no-Movements}

Based on  dynamics \eqref{eqgroup...text...Normalized-Dynamics}, by taking  norm  on both sides,  then for any $l\in[L-1]$,    
\begin{align*}
\frac{\D \Norm{\frac{\overline{\mW}^{[l]}}{\sqrt{m}}}_{2\to 2}}{\D t}&\leq \frac{\kappa}{\alpha_l^2}\Norm{\left(\prod_{k=l}^{L-1}\overline{\mE}^{[k]}(\vx_i)\right)  \vsigma_{[L]}^{(1)}\left(\vx_i\right)\frac{\bar{\va}}{\sqrt{m}}}_2 \Norm{\bar{\vx}_i^{[l-1]}}_2\sqrt{\frac{\sum_{i=1}^n e_i^2}{n}}\\
&\leq\frac{\sqrt{2}\kappa}{\alpha_l^2} \left(\prod_{k=1}^{l-1}\Norm{\frac{\overline{\mW}^{[k]}}{\sqrt{m}}}_{2\to 2}\right)\left(\prod_{k=l+1}^{L}\Norm{ \frac{\overline{\mW}^{[k]}}{\sqrt{m}} }_{2\to 2}\right)\Norm{\frac{\bar{\va}}{\sqrt{m}}}_2 \sqrt{R_\fS(\vtheta)},\\
\frac{\D \Norm{\frac{\overline{\mW}^{[L]}}{\sqrt{m}}}_{2\to 2}}{\D t}&\leq \frac{\kappa}{\alpha_L^2}\Norm{\vsigma_{[L]}^{(1)}\left(\vx_i\right)\frac{\bar{\va}}{\sqrt{m}}}_2 \Norm{\bar{\vx}_i^{[L-1]}}_2\sqrt{\frac{\sum_{i=1}^n e_i^2}{n}}\\
&\leq\frac{\sqrt{2}\kappa}{\alpha_L^2} \left(\prod_{k=1}^{L-1}\Norm{\frac{\overline{\mW}^{[k]}}{\sqrt{m}}}_{2\to 2}\right) \Norm{\frac{\bar{\va}}{\sqrt{m}}}_2 \sqrt{R_\fS(\vtheta)},\\
\frac{\mathrm{d} \Norm{\frac{\bar{\va}}{\sqrt{m}}}_2 }{\mathrm{d} t} &\leq \frac{\kappa}{\alpha_{L+1}^2}    \Norm{\bar{\vx}_i^{[L]}}_2\sqrt{\frac{\sum_{i=1}^n e_i^2}{n}} \leq \frac{\sqrt{2}\kappa}{\alpha_{L+1}^2}\left(\prod_{k=1}^{L}\Norm{\frac{\overline{\mW}^{[k]}}{\sqrt{m}}}_{2\to 2}\right) \sqrt{R_\fS(\vtheta)}. 
\end{align*}
Based on Proposition \ref{B-prop...Loss-Initial-Decay}, we remark that for any time $t\in[0,t^*)$,
\begin{align*}
\int_{0}^t\sqrt{R_\fS(\vtheta(s))}\D s&\leq \int_{0}^t\exp\left(- \frac{1}{2n}\left[\left(\sum_{l=1}^{L+1} \frac{\kappa^2}{\alpha_l^2}\right) {\lambda_\fS}\right]s\right)\sqrt{R_\fS(\vtheta^0)} \D s\\
&\leq \int_{0}^{\infty}\exp\left(- \frac{1}{2n}\left[\left(\sum_{l=1}^{L+1} \frac{\kappa^2}{\alpha_l^2}\right) {\lambda_\fS}\right]s\right)\sqrt{R_\fS(\vtheta^0)} \D s=\frac{2n}{\left(\sum_{l=1}^{L+1} \frac{\kappa^2}{\alpha_l^2}\right) {\lambda_\fS}}.
\end{align*}
Thus,  we have for any $l\in[L]$ and  time $t\in[0,t^*)$,
\begin{align*}
 \Norm{\frac{\overline{\mW}^{[l]}(t)}{\sqrt{m}}-\frac{\overline{\mW}^{[l]}(0)}{\sqrt{m}}}_{2\to 2} 
&\leq     \frac{\sqrt{2}\kappa}{\alpha_l^2} \left(\prod_{k=1}^{l-1}p_k(t)\right)\left(\prod_{k=l+1}^{L+1}p_k(t)\right) \int_{0}^t\sqrt{R_\fS(\vtheta(s))}\D s\\
& \leq \frac{\sqrt{2}\kappa}{\alpha_l^2}\sqrt{R_\fS(\vtheta^0)} \left(\prod_{k=1}^{l-1}p_k(t)\right)\left(\prod_{k=l+1}^{L+1}p_k(t)\right) \frac{2n}{\left(\sum_{l=1}^{L+1} \frac{\kappa^2}{\alpha_l^2}\right) {\lambda_\fS}}\\
&\leq \frac{2\sqrt{2}n\sqrt{R_\fS(\vtheta^0)}}{\kappa \lambda_\fS} \left(\prod_{k=1}^{l-1}p_k(t)\right)\left(\prod_{k=l+1}^{L+1}p_k(t)\right),
\end{align*}
then by similar reasoning,
\begin{align*}
\Norm{\frac{\bar{\va}(t)}{\sqrt{m}}-\frac{\bar{\va}(0)}{\sqrt{m}}}_{2}
&\leq      \frac{2\sqrt{2}n\sqrt{R_\fS(\vtheta^0)}}{\kappa \lambda_\fS}\left(\prod_{k=1}^{L}p_k(t)\right).
\end{align*}
Directly from  Proposition \ref{A-prop..Upper-Bound-and-Lower-Bound-Initial-Parameter},   for any $l\in[L+1]$,
\[p_l(0)\leq 2.\]
Moreover, for any $l\in[L]$ and   time $t\in[0,t^*)$,
\begin{align*}
p_l(t)&\leq 2+  \frac{2\sqrt{2}n\sqrt{R_\fS(\vtheta^0)}}{\kappa \lambda_\fS} \left(\prod_{k=1}^{l-1}p_k(t)\right)\left(\prod_{k=l+1}^{L+1}p_k(t)\right),  \\
p_{L+1}(t)&\leq 2+ \frac{2\sqrt{2}n\sqrt{R_\fS(\vtheta^0)}}{\kappa \lambda_\fS}\left(\prod_{k=1}^{L}p_k(t)\right),
\end{align*}
hence if we choose $m$ large enough, such that 
\[
m\geq \left(4^L \frac{2\sqrt{2}n\sqrt{R_\fS(\vtheta^0)}}{  \lambda_\fS}\right)^{\frac{1}{\frac{L+1}{2}-\sum_{k=1}^{L+1}\gamma_k}},
\]
then the following holds
\[
\frac{2\sqrt{2}n\sqrt{R_\fS(\vtheta^0)}}{\kappa \lambda_\fS}\leq \frac{1}{4^L}.
\]
As we denote  
\begin{equation}
    p(t):=\max_{l\in[L+1]}p_l(t),
\end{equation}
and we  define the time 
\begin{equation}
    t^{**} = \inf\{t\in[0,t^*) \mid p(t)>3\},
\end{equation}
as $p(0)\leq 2$, hence $t^{**}$ is non-empty.
Suppose we have $t^{**}<t^*$, then as   $t\to t^{**}$,   
\begin{equation}
    p(t^{**})\leq 2+\frac{3^L}{4^L}\leq \frac{11}{4}<3,
\end{equation}
which leads to  contradiction with the definition of $t^{**}$.  Therefore $t^{**}=t^*$,  and    for any $l\in[L+1]$ and  $t\in[0,t^*)$, 
\[ 
p_l(t)\leq 4.
\] 
\subsection{Proof of Theorem \ref{Theorem...repeaet}}\label{B-subsection...Proof-of-Theorem}
It suffices to show that $t^*=\infty$.    \\ 
\noindent (i). Firstly, we  demonstrate that for any $i\in[n]$ and    time $t\in[0, t^*)$,
\begin{equation}\label{B-Thm-proof...eq...l<L+1-th-2-Norm}
\begin{aligned}
\Norm{\bar{\vx}_{i}^{[l]}(t)}_2&\leq4^l,\\
\Norm{\bar{\vx}_{i}^{[l]}(t)-\bar{\vx}_{i}^{[l]}(0)}_2&\leq \frac{2\sqrt{2}n\sqrt{R_\fS(\vtheta^0)}}{\kappa \lambda_\fS}\left(8^L\right)^l.
\end{aligned}
\end{equation}
For $l=1$, as we set  $\eps_1:= \alpha_1$, then for any time $t\in[0, t^*)$,  
\begin{align*}
 \Norm{ \bar{\vx}_{i}^{[1]}(t)}_2&=\Norm{\frac{\sigma\left(\eps_1\frac{\overline{\mW}^{[1]}(t)}{\sqrt{m}}\Bar{\vx}_i^{[0]}\right)}{\eps_1}}_2 \\
&\leq \Norm{\frac{\overline{\mW}^{[1]}(t)}{\sqrt{m}}}_{2\to 2}\Norm{\Bar{\vx}_j^{[0]}}_2 \leq p_1(t)\leq 4^1,\end{align*} 
and 
\begin{align*}
\Norm{ \bar{\vx}_{i}^{[1]}(t)-\bar{\vx}_{i}^{[1]}(0)}_2&=\Norm{\frac{\sigma\left(\eps_1\frac{\overline{\mW}^{[1]}(t)}{\sqrt{m}}\Bar{\vx}_i^{[0]}\right)}{\eps_1}-\frac{\sigma\left(\eps_1\frac{\overline{\mW}^{[1]}(0)}{\sqrt{m}}\Bar{\vx}_i^{[0]}\right)}{\eps_1}}_2\\
&\leq \Norm{\frac{\overline{\mW}^{[1]}(t)}{\sqrt{m}}-\frac{\overline{\mW}^{[1]}(0)}{\sqrt{m}}}_{2\to 2}\Norm{\Bar{\vx}_i^{[0]}}_2  \leq \frac{2\sqrt{2}n\sqrt{R_\fS(\vtheta^0)}}{\kappa \lambda_\fS}4^L.
\end{align*} 
We assume that   \eqref{B-Thm-proof...eq...l<L+1-th-2-Norm} holds for $l-1$,  as we set $\eps_l:= \prod_{k=1}^l\alpha_k$, then for any time $t\in[0, t^*)$,  
\begin{align*}
 \Norm{ \bar{\vx}_{i}^{[l]}(t)}_2&=\Norm{\frac{\sigma\left(\eps_l\frac{\overline{\mW}^{[l]}(t)}{\sqrt{m}}\Bar{\vx}_i^{[l-1]}\right)}{\eps_l}}_2 \\
&\leq \Norm{\frac{\overline{\mW}^{[l]}(t)}{\sqrt{m}}}_{2\to 2}\Norm{\Bar{\vx}_j^{[l-1]}}_2 \leq p_l(t)4^{l-1}\leq 4^l,\end{align*}
and  
\begin{align*}
\Norm{ \bar{\vx}_{i}^{[l]}(t)-\bar{\vx}_{i}^{[l]}(0)}_2&=\Norm{\frac{\sigma\left(\eps_l\frac{\overline{\mW}^{[l]}(t)}{\sqrt{m}}\Bar{\vx}_i^{[l-1]}(t)\right)}{\eps_l}-\frac{\sigma\left(\eps_l\frac{\overline{\mW}^{[l]}(0)}{\sqrt{m}}\Bar{\vx}_i^{[l-1]}(0)\right)}{\eps_l}}_2\\
&\leq \Norm{\frac{\overline{\mW}^{[l]}(t)}{\sqrt{m}}-\frac{\overline{\mW}^{[l]}(0)}{\sqrt{m}}}_{2\to 2}\Norm{\Bar{\vx}_i^{[l-1]}(t)}_2 \\
&~~+ \Norm{\frac{\overline{\mW}^{[l]}(0)}{\sqrt{m}}}_{2\to 2}\Norm{\Bar{\vx}_i^{[l-1]}(t)-\Bar{\vx}_i^{[l-1]}(0)}_2\\
&\leq \frac{2\sqrt{2}n\sqrt{R_\fS(\vtheta^0)}}{\kappa \lambda_\fS}4^L4^{l-1}+2\frac{2\sqrt{2}n\sqrt{R_\fS(\vtheta^0)}}{\kappa \lambda_\fS}\left(8^L\right)^{l-1} \\
&\leq \frac{2\sqrt{2}n\sqrt{R_\fS(\vtheta^0)}}{\kappa \lambda_\fS}\left(8^L\right)^{l}.
\end{align*}
(ii).  As we recall that
\begin{equation} 
\begin{aligned}
\overline{\mH}_{ij}^{[L]}   (\vtheta)&:=\left< 
 \vsigma_{[L]}^{(1)}\left(\vx_i\right)\frac{\Bar{\va}}{\sqrt{m}},   
 \vsigma_{[L]}^{(1)}\left(\vx_j\right)\frac{\Bar{\va}}{\sqrt{m}}\right>,\\
\overline{\mH}_{ij}^{[l]}   (\vtheta) &:= \left<
 \prod_{k=l}^{L-1}\overline{\mE}^{[k]}(\vx_i)
 \vsigma_{[L]}^{(1)}\left(\vx_i\right)\frac{\Bar{\va}}{\sqrt{m}},  
\prod_{k=l}^{L-1}\overline{\mE}^{[k]}(\vx_j)
 \vsigma_{[L]}^{(1)}\left(\vx_j\right)\frac{\Bar{\va}}{\sqrt{m}}\right>,
\end{aligned}
\end{equation}
and for any $i\in[n]$ and  $l\in[L]$, $\left\{\vlambda_i^{[l]}(t)\right\}_{l=1}^L$ are inductively defined as follows
\begin{equation}
\begin{aligned}
\vlambda_i^{[L]}(t)&=\vsigma_{[L]}^{(1)}\left(\vx_i(t)\right)\frac{\Bar{\va}(t)}{\sqrt{m}}, \\
\vlambda_i^{[l]}(t)&=\overline{\mE}^{[l]}(\vx_i) \vlambda_i^{[l+1]}(t)=\vsigma_{[l]}^{(1)}\left({\vx}_{i}(t)\right)\left(\frac{\overline{\mW}^{[l+1]}(t)}{\sqrt{m}}\right)^\T\vlambda_i^{[l+1]}(t),
\end{aligned}    
\end{equation}
therefore, we have that  for any $l\in[L]$ 
\begin{align*}
\overline{\mH}_{ij}^{[l]}(\vtheta(t))& = \left< 
  \vlambda_i^{[l]}(t),  
  \vlambda_j^{[l]}(t)\right>,   \\
\overline{\mG}_{ij}^{[l]}   (\vtheta(t))&=  \overline{\mH}_{ij}^{[l]}(\vtheta(t))\left< \bar{\vx}_{i}^{[l-1]}(t),  \bar{\vx}_{j}^{[l-1]}(t)  \right>. 
\end{align*}
 and for any $i\in[n]$, $l\in[L]$, and  any time $t\in[0, t^*)$,
\begin{align*}
\Norm{\vlambda_i^{[l]}(t)}_2&\leq \prod_{k=l+1}^{L+1}p_{k}(t)\leq  4^{L-l+1},  
\end{align*}
and we   demonstrate that 
for any $i,j\in[n]$,  $l\in[L]$, and   any time $t\in[0, t^*)$,
\begin{equation}\label{B-Thm-proof...eq...l<L-th-Infinity-Norm}
\begin{aligned}
&\Abs{\overline{\mH}_{ij}^{[l]}   (\vtheta(t))-\overline{\mH}_{ij}^{[l]}   (\vtheta(0))}\leq \frac{2\sqrt{2}n\sqrt{R_\fS(\vtheta^0)}}{\kappa \lambda_\fS}\left(1024^{L}\right)^{2L-l}.
\end{aligned}
\end{equation}
For $l=L$,  we demonstrate that   \eqref{B-Thm-proof...eq...l<L-th-Infinity-Norm} holds for $L$.
Since  we have
\begin{align*}
&\Abs{\overline{\mH}_{ij}^{[L]}   (\vtheta(t))-\overline{\mH}_{ij}^{[L]}   (\vtheta(0))}\\
=&\Abs{\left< 
 \vsigma_{[L]}^{(1)}\left(\vx_i(t)\right)\frac{\Bar{\va}(t)}{\sqrt{m}},  
 \vsigma_{[L]}^{(1)}\left(\vx_j(t)\right)\frac{\Bar{\va}(t)}{\sqrt{m}}\right>-\left< 
 \vsigma_{[L]}^{(1)}\left(\vx_i(0)\right)\frac{\Bar{\va}(0)}{\sqrt{m}},  
 \vsigma_{[L]}^{(1)}\left(\vx_j(0)\right)\frac{\Bar{\va}(0)}{\sqrt{m}}\right>}\\
\leq &
\Abs{
\left< 
 \vsigma_{[L]}^{(1)}\left(\vx_i(t)\right)\frac{\Bar{\va}(t)}{\sqrt{m}},  
 \vsigma_{[L]}^{(1)}\left(\vx_j(t)\right)\frac{\Bar{\va}(t)}{\sqrt{m}}\right>
 -
 \left< 
 \vsigma_{[L]}^{(1)}\left(\vx_i(t)\right)\frac{\Bar{\va}(0)}{\sqrt{m}},  
 \vsigma_{[L]}^{(1)}\left(\vx_j(t)\right)\frac{\Bar{\va}(0)}{\sqrt{m}}\right>
 }
\\
&+\Abs{
 \left< 
 \vsigma_{[L]}^{(1)}\left(\vx_i(t)\right)\frac{\Bar{\va}(0)}{\sqrt{m}},  
 \vsigma_{[L]}^{(1)}\left(\vx_j(t)\right)\frac{\Bar{\va}(0)}{\sqrt{m}}\right>
-
 \left< 
 \vsigma_{[L]}^{(1)}\left(\vx_i(0)\right)\frac{\Bar{\va}(0)}{\sqrt{m}},  
 \vsigma_{[L]}^{(1)}\left(\vx_j(0)\right)\frac{\Bar{\va}(0)}{\sqrt{m}}\right>
 },
\end{align*}
estimate  on the first term reads
\begin{align*}
&\Abs{
\left< 
 \vsigma_{[L]}^{(1)}\left(\vx_i(t)\right)\frac{\Bar{\va}(t)}{\sqrt{m}},  
 \vsigma_{[L]}^{(1)}\left(\vx_j(t)\right)\frac{\Bar{\va}(t)}{\sqrt{m}}\right>
 -
 \left< 
 \vsigma_{[L]}^{(1)}\left(\vx_i(t)\right)\frac{\Bar{\va}(0)}{\sqrt{m}},  
 \vsigma_{[L]}^{(1)}\left(\vx_j(t)\right)\frac{\Bar{\va}(0)}{\sqrt{m}}\right>
 }  \\
\leq & \Abs{
\left< 
 \frac{\Bar{\va}(t)}{\sqrt{m}},  
 \frac{\Bar{\va}(t)}{\sqrt{m}}\right>
-
\left< 
 \frac{\Bar{\va}(0)}{\sqrt{m}},  
 \frac{\Bar{\va}(0)}{\sqrt{m}}\right>
 }\leq \Abs{\Norm{\frac{\Bar{\va}(t)}{\sqrt{m}}}_2-\Norm{\frac{\Bar{\va}(0)}{\sqrt{m}}}_2}\left(\Norm{\frac{\Bar{\va}(t)}{\sqrt{m}}}_2+\Norm{\frac{\Bar{\va}(0)}{\sqrt{m}}}_2\right)\\
\leq &2p_{L+1}(t)\Abs{\Norm{\frac{\Bar{\va}(t)}{\sqrt{m}}}_2-\Norm{\frac{\Bar{\va}(0)}{\sqrt{m}}}_2}\leq 8\Norm{\frac{\Bar{\va}(t)}{\sqrt{m}}-\frac{\Bar{\va}(0)}{\sqrt{m}}}_2\leq \frac{16\sqrt{2}n\sqrt{R_\fS(\vtheta^0)}}{\kappa \lambda_\fS}4^L,
\end{align*}
and  for  the second  term,  as we set $\eps_L:= \prod_{k=1}^L\alpha_k$,  we only need to focus on the case where $\eps_L>1$,   and since  the entries in $\vsigma^{(1)}_{[L]}(\vx_i(t))$ and $\vsigma^{(1)}_{[L]}(\vx_j(t))$ read
\[
\vsigma^{(1)}_{[L]}(\vx_i(t))=\mathrm{diag}\left([\mu_p^i(t)]_{m\times 1}\right),~~\vsigma^{(1)}_{[L]}(\vx_j(t))=\mathrm{diag}\left([\mu_p^j(t)]_{m\times 1}\right),
\]
we obtain that the second term reads
\begin{align*}
  &\Abs{
 \left< 
 \vsigma_{[L]}^{(1)}\left(\vx_i(t)\right)\frac{\Bar{\va}(0)}{\sqrt{m}},  
 \vsigma_{[L]}^{(1)}\left(\vx_j(t)\right)\frac{\Bar{\va}(0)}{\sqrt{m}}\right>
-
 \left< 
 \vsigma_{[L]}^{(1)}\left(\vx_i(0)\right)\frac{\Bar{\va}(0)}{\sqrt{m}},  
 \vsigma_{[L]}^{(1)}\left(\vx_j(0)\right)\frac{\Bar{\va}(0)}{\sqrt{m}}\right>
 }  \\
=&\Abs{\frac{1}{m}\sum_{p=1}^m \Bar{\va}_p^2(0)\left(\mu_p^i(t)\mu_p^j(t)-\mu_p^i(0)\mu_p^j(0)\right)}\\
\leq &\Abs{\frac{1}{m}\sum_{p=1}^m \Bar{\va}_p^2(0)\left(\mu_p^i(t)-\mu_p^i(0) \right)}\max_{p\in [m]}\Abs{\mu_p^j(t)}+\Abs{\frac{1}{m}\sum_{p=1}^m \Bar{\va}_p^2(0)\left(\mu_p^j(t)-\mu_p^j(0) \right)}\max_{p\in [m]}\Abs{\mu_p^i(0)}\\
\leq &\Abs{\frac{1}{m}\sum_{p=1}^m \Bar{\va}_p^2(0)\left(\mu_p^i(t)-\mu_p^i(0) \right)} +\Abs{\frac{1}{m}\sum_{p=1}^m \Bar{\va}_p^2(0)\left(\mu_p^j(t)-\mu_p^j(0) \right)}, 
\end{align*}
as  we   write  $\frac{\overline{\mW}^{[L]}(t)}{\sqrt{m}}$ into
\[\frac{\overline{\mW}^{[L]}(t)}{\sqrt{m}}:=\begin{pmatrix}
\left(\frac{\Bar{\vw}_{L,1}(t)}{\sqrt{m}}\right)^\T \\
\left(\frac{\Bar{\vw}_{L,2}(t)}{\sqrt{m}}\right)^\T \\
\vdots\\
\left(\frac{\Bar{\vw}_{L,m}(t)}{\sqrt{m}} \right)^\T 
\end{pmatrix}, \] 
and we define the following events: For any  $i\in[n]$ and $k'\in[m]$,
\begin{equation}\label{B-Thm-proof...eq...Events-AR-BR}
\begin{aligned} 
\sA_{i,L,k'}(R_{k'}):=\Bigg\{&\Abs{\left<\frac{\Bar{\vw}_{L,k'}(t)}{\sqrt{m}}, \bar{\vx}_{i}^{[L-1]}(t)\right>-\left<\frac{\Bar{\vw}_{L,k'}(0)}{\sqrt{m}},\bar{\vx}_{i}^{[L-1]}(0)\right>}=R_{k'} ,\\
&~\mathbf{1}_{\left<\frac{\Bar{\vw}_{L,k'}(t)}{\sqrt{m}}, ~\bar{\vx}_{i}^{[L-1]}(t)\right>>0}\neq \mathbf{1}_{\left<\frac{\Bar{\vw}_{L,k'}(0)}{\sqrt{m}},~\bar{\vx}_{i}^{[L-1]}(0)\right>>0}~~~~~~~~~\Bigg\}.
\end{aligned}
\end{equation}
More importantly, given  that 
\[
\Abs{\left<\frac{\Bar{\vw}_{L,k'}(0)}{\sqrt{m}}, \bar{\vx}_{i}^{[L-1]}(0)\right>}>2R_{k'},
\]
then the event $\sA_{i,L,k'}(R_{k'})$ would never happen.  Hence,   estimates on   probability $\Prob\left(\sA_{i,L,k'}(R_{k'})\right)$ reads  
\begin{equation}
\begin{aligned}
\Prob\left(\sA_{i,L,k'}(R_{k'}) 
\right)
&\leq   \Prob\left(\Abs{\left<\frac{\Bar{\vw}_{L,k'}(0)}{\sqrt{m}}, \bar{\vx}_{i}^{[L-1]}(0)\right>}\leq 2R_{k'}\right)\\ 
&\leq 2 \int_0^{\frac{2R_{k'}\sqrt{m}}{\Norm{\bar{\vx}_{i}^{[L-1]}(0)}_2}} \frac{1}{\sqrt{2\pi}}\exp\left( -\frac{y^2}{2}\right)\D y\\
&\leq \frac{2}{\sqrt{2\pi}}\frac{2R_{k'}\sqrt{m}}{\Norm{\bar{\vx}_{i}^{[L-1]}(0)}_2}\leq R_{k'}\sqrt{m}8^{L-1}.
\end{aligned}
\end{equation}
Then,   estimate  on the   second  term reads,
\begin{align*}
  &\Abs{
 \left< 
 \vsigma_{[L]}^{(1)}\left(\vx_i(t)\right)\frac{\Bar{\va}(0)}{\sqrt{m}},  
 \vsigma_{[L]}^{(1)}\left(\vx_j(t)\right)\frac{\Bar{\va}(0)}{\sqrt{m}}\right>
-
 \left< 
 \vsigma_{[L]}^{(1)}\left(\vx_i(0)\right)\frac{\Bar{\va}(0)}{\sqrt{m}},  
 \vsigma_{[L]}^{(1)}\left(\vx_j(0)\right)\frac{\Bar{\va}(0)}{\sqrt{m}}\right>
 }  \\
\leq &\Abs{\frac{1}{m}\sum_{p=1}^m \Bar{\va}_p^2(0)\left(\mu_p^i(t)-\mu_p^i(0) \right)} +\Abs{\frac{1}{m}\sum_{p=1}^m \Bar{\va}_p^2(0)\left(\mu_p^j(t)-\mu_p^j(0) \right)} \\ 
\leq &\Abs{b-a}\Abs{\frac{1}{m}\sum_{p=1}^m \Bar{\va}_p^2(0)\Abs{\mathbf{1}_{\left<\frac{\Bar{\vw}_{L,p}(t)}{\sqrt{m}}, ~\bar{\vx}_{i}^{[L-1]}(t)\right>>0}- \mathbf{1}_{\left<\frac{\Bar{\vw}_{L,p}(0)}{\sqrt{m}},~\bar{\vx}_{i}^{[L-1]}(0)\right>>0}}}\\
&+\Abs{b-a}\Abs{\frac{1}{m}\sum_{p=1}^m \Bar{\va}_p^2(0)\Abs{\mathbf{1}_{\left<\frac{\Bar{\vw}_{L,p}(t)}{\sqrt{m}}, ~\bar{\vx}_{j}^{[L-1]}(t)\right>>0}- \mathbf{1}_{\left<\frac{\Bar{\vw}_{L,p}(0)}{\sqrt{m}},~\bar{\vx}_{j}^{[L-1]}(0)\right>>0}}}\\
\leq &2\Abs{b-a}\Abs{\frac{1}{m}\sum_{p=1}^m \Bar{\va}_p^2(0)\Abs{\mathbf{1}_{\left<\frac{\Bar{\vw}_{L,p}(t)}{\sqrt{m}}, ~\bar{\vx}_{i}^{[L-1]}(t)\right>>0}- \mathbf{1}_{\left<\frac{\Bar{\vw}_{L,p}(0)}{\sqrt{m}},~\bar{\vx}_{i}^{[L-1]}(0)\right>>0}}},
\end{align*}
we omit the term $\Abs{b-a}$  for simplicity, then for any $p\in[m]$  and $i\in[n]$,  we observe that 
\[
\Norm{\Bar{\va}_p^2(0)\Abs{\mathbf{1}_{\left<\frac{\Bar{\vw}_{L,p}(t)}{\sqrt{m}}, ~\bar{\vx}_{i}^{[L-1]}(t)\right>>0}- \mathbf{1}_{\left<\frac{\Bar{\vw}_{L,p}(0)}{\sqrt{m}},~\bar{\vx}_{i}^{[L-1]}(0)\right>>0}}}_{\psi}\leq \Norm{\Bar{\va}_{p}^2(0)}_{\psi}\leq C_{\psi, 1},
\]
is a sub-exponential random variable, and as we notice that
\begin{align*}
&\Exp\left[ \frac{1}{m}\sum_{p=1}^m\Bar{\va}_p^2(0)\Abs{\mathbf{1}_{\left<\frac{\Bar{\vw}_{L,p}(t)}{\sqrt{m}}, ~\bar{\vx}_{i}^{[L-1]}(t)\right>>0}- \mathbf{1}_{\left<\frac{\Bar{\vw}_{L,p}(0)}{\sqrt{m}},~\bar{\vx}_{i}^{[L-1]}(0)\right>>0}}\right]\\
=&\frac{1}{m}\sum_{p=1}^m\Exp\left[ \Bar{\va}_p^2(0)\right]\Exp\left[\Abs{\mathbf{1}_{\left<\frac{\Bar{\vw}_{L,p}(t)}{\sqrt{m}}, ~\bar{\vx}_{i}^{[L-1]}(t)\right>>0}- \mathbf{1}_{\left<\frac{\Bar{\vw}_{L,p}(0)}{\sqrt{m}},~\bar{\vx}_{i}^{[L-1]}(0)\right>>0}}\right]\\
\leq& \frac{1}{m}\sum_{p=1}^m\Prob\left(\sA_{i,L,p}(R_{p})\right)\leq \frac{1}{m}\sum_{p=1}^mR_{p}\sqrt{m}8^{L-1},
\end{align*}
whose estimate reads,
\begin{align*}
& \frac{1}{m}\sum_{p=1}^mR_{p}\sqrt{m}=\frac{1}{m}\sum_{p=1}^m\sqrt{m}\Abs{\left<\frac{\Bar{\vw}_{L,p}(t)}{\sqrt{m}}, \bar{\vx}_{i}^{[L-1]}(t)\right>-\left<\frac{\Bar{\vw}_{L,p}(0)}{\sqrt{m}},\bar{\vx}_{i}^{[L-1]}(0)\right>}\\
\leq & \frac{1}{m}\sum_{p=1}^m \sqrt{m}\left(\Norm{\frac{\Bar{\vw}_{L,p}(t)}{\sqrt{m}}- \frac{\Bar{\vw}_{L,p}(0)}{\sqrt{m}}}_2  \Norm{\bar{\vx}_{i}^{[L-1]}(t)}_2+\Abs{ \left<\frac{\Bar{\vw}_{L,p}(0)}{\sqrt{m}}, \bar{\vx}_{i}^{[L-1]}(t)-\bar{\vx}_{i}^{[L-1]}(0)\right>}\right)\\
\leq & \sqrt{ \sum_{p=1}^m\left(2\Norm{\frac{\Bar{\vw}_{L,p}(t)}{\sqrt{m}}- \frac{\Bar{\vw}_{L,p}(0)}{\sqrt{m}}}_2^2 \Norm{\bar{\vx}_{i}^{[L-1]}(t)}_2^2+2\Abs{ \left<\frac{\Bar{\vw}_{L,p}(0)}{\sqrt{m}}, \bar{\vx}_{i}^{[L-1]}(t)-\bar{\vx}_{i}^{[L-1]}(0)\right>}_2^2\right)}\\
=& \sqrt{2\Norm{\frac{\overline{\mW}^{[L]}(t)}{\sqrt{m}}-\frac{\overline{\mW}^{[L]}(0)}{\sqrt{m}}}_{\mathrm{F}}^216^{L-1}+2\Norm{\frac{\overline{\mW}^{[L]}(0)}{\sqrt{m}}\left(\bar{\vx}_{i}^{[L-1]}(t)-\bar{\vx}_{i}^{[L-1]}(0)\right)}_2^2}\\
\leq &\sqrt{2\Norm{\frac{\overline{\mW}^{[L]}(t)}{\sqrt{m}}-\frac{\overline{\mW}^{[L]}(0)}{\sqrt{m}}}_{\mathrm{F}}^216^{L-1}+2\Norm{\frac{\overline{\mW}^{[L]}(0)}{\sqrt{m}}}_2^2\Norm{\bar{\vx}_{i}^{[L-1]}(t)-\bar{\vx}_{i}^{[L-1]}(0)}_2^2}\\
\leq & 8^{L-1} \left(\Norm{\frac{\overline{\mW}^{[L]}(t)}{\sqrt{m}}-\frac{\overline{\mW}^{[L]}(0)}{\sqrt{m}}}_{\mathrm{F}}+\Norm{\bar{\vx}_{i}^{[L-1]}(t)-\bar{\vx}_{i}^{[L-1]}(0)}_2\right)\\
\leq& 8^{L-1}\left(\frac{2\sqrt{2}n\sqrt{R_\fS(\vtheta^0)}}{\kappa \lambda_\fS}4^L+\frac{2\sqrt{2}n\sqrt{R_\fS(\vtheta^0)}}{\kappa \lambda_\fS}\left(8^L\right)^{L-1}\right)\leq \frac{2\sqrt{2}n\sqrt{R_\fS(\vtheta^0)}}{\kappa \lambda_\fS}\left(64^L\right)^{L-1},
\end{align*}
then with high probability, the following holds  
\begin{align*}
  &\Abs{
 \left< 
 \vsigma_{[L]}^{(1)}\left(\vx_i(t)\right)\frac{\Bar{\va}(0)}{\sqrt{m}},  
 \vsigma_{[L]}^{(1)}\left(\vx_j(t)\right)\frac{\Bar{\va}(0)}{\sqrt{m}}\right>
-
 \left< 
 \vsigma_{[L]}^{(1)}\left(\vx_i(0)\right)\frac{\Bar{\va}(0)}{\sqrt{m}},  
 \vsigma_{[L]}^{(1)}\left(\vx_j(0)\right)\frac{\Bar{\va}(0)}{\sqrt{m}}\right>
 }  \\
\leq &\frac{2}{m} \sum_{p=1}^mR_{p}\sqrt{m}8^{L-1} \leq \frac{4\sqrt{2}n\sqrt{R_\fS(\vtheta^0)}}{\kappa \lambda_\fS}\left(64^L\right)^{L-1}8^{L-1}\leq \frac{2\sqrt{2}n\sqrt{R_\fS(\vtheta^0)}}{\kappa \lambda_\fS}\left(512^L\right)^{L-1},
\end{align*}
to sum up,  we obtain that for  any time $t\in[0, t^*)$
\begin{align*}
 \Abs{\overline{\mH}_{ij}^{[L]}   (\vtheta(t))-\overline{\mH}_{ij}^{[L]}   (\vtheta(0))}  
&\leq  \frac{16\sqrt{2}n\sqrt{R_\fS(\vtheta^0)}}{\kappa \lambda_\fS}4^L+\frac{2\sqrt{2}n\sqrt{R_\fS(\vtheta^0)}}{\kappa \lambda_\fS}\left(512^L\right)^{L-1}\\
&\leq \frac{2\sqrt{2}n\sqrt{R_\fS(\vtheta^0)}}{\kappa \lambda_\fS}\left(1024^{L}\right)^L.
\end{align*}

We assume that   \eqref{B-Thm-proof...eq...l<L-th-Infinity-Norm} holds for $l+1$,  and we proceed to demonstrate that \eqref{B-Thm-proof...eq...l<L-th-Infinity-Norm} holds for $l$. Since   we have
\begin{align*}
&\Abs{\overline{\mH}_{ij}^{[l]}   (\vtheta(t))-\overline{\mH}_{ij}^{[l]}   (\vtheta(0))}=\Abs{
\left< 
\vlambda_i^{[l]}(t),  
 \vlambda_j^{[l]}(t)
 \right>
-
\left< 
\vlambda_i^{[l]}(0),  
 \vlambda_j^{[l]}(0)
 \right>
 }\\
\leq &\Bigg|
\left< 
 \vsigma_{[l]}^{(1)}\left({\vx}_{i}(t)\right)\left(\frac{\overline{\mW}^{[l+1]}(t)}{\sqrt{m}}\right)^\T\vlambda_i^{[l+1]}(t),  
 \vsigma_{[l]}^{(1)}\left({\vx}_{j}(t)\right)\left(\frac{\overline{\mW}^{[l+1]}(t)}{\sqrt{m}}\right)^\T\vlambda_j^{[l+1]}(t)
 \right>
 \\
&~~-
\left< 
 \vsigma_{[l]}^{(1)}\left({\vx}_{i}(t)\right)\left(\frac{\overline{\mW}^{[l+1]}(0)}{\sqrt{m}}\right)^\T\vlambda_i^{[l+1]}(t),  
 \vsigma_{[l]}^{(1)}\left({\vx}_{j}(t)\right)\left(\frac{\overline{\mW}^{[l+1]}(0)}{\sqrt{m}}\right)^\T\vlambda_j^{[l+1]}(t)
 \right>\Bigg|\\
 &+\Bigg| 
\left< 
 \vsigma_{[l]}^{(1)}\left({\vx}_{i}(t)\right)\left(\frac{\overline{\mW}^{[l+1]}(0)}{\sqrt{m}}\right)^\T\vlambda_i^{[l+1]}(t),  
 \vsigma_{[l]}^{(1)}\left({\vx}_{j}(t)\right)\left(\frac{\overline{\mW}^{[l+1]}(0)}{\sqrt{m}}\right)^\T\vlambda_j^{[l+1]}(t)
 \right>\\
 &~~-
\left< 
 \vsigma_{[l]}^{(1)}\left({\vx}_{i}(t)\right)\left(\frac{\overline{\mW}^{[l+1]}(0)}{\sqrt{m}}\right)^\T\vlambda_i^{[l+1]}(0),  
 \vsigma_{[l]}^{(1)}\left({\vx}_{j}(t)\right)\left(\frac{\overline{\mW}^{[l+1]}(0)}{\sqrt{m}}\right)^\T\vlambda_j^{[l+1]}(0)
 \right>\Bigg|\\
&+\Bigg| 
\left< 
 \vsigma_{[l]}^{(1)}\left({\vx}_{i}(t)\right)\left(\frac{\overline{\mW}^{[l+1]}(0)}{\sqrt{m}}\right)^\T\vlambda_i^{[l+1]}(0),  
 \vsigma_{[l]}^{(1)}\left({\vx}_{j}(t)\right)\left(\frac{\overline{\mW}^{[l+1]}(0)}{\sqrt{m}}\right)^\T\vlambda_j^{[l+1]}(0)
 \right> \\
 &~~-
\left< 
 \vsigma_{[l]}^{(1)}\left({\vx}_{i}(0)\right)\left(\frac{\overline{\mW}^{[l+1]}(0)}{\sqrt{m}}\right)^\T\vlambda_i^{[l+1]}(0),  
 \vsigma_{[l]}^{(1)}\left({\vx}_{j}(0)\right)\left(\frac{\overline{\mW}^{[l+1]}(0)}{\sqrt{m}}\right)^\T\vlambda_j^{[l+1]}(0)
 \right>\Bigg|,
\end{align*}
so there are three terms to be analyzed, then for the first term, we obtain that 
\begin{align*}
&\Bigg|
\left< 
 \vsigma_{[l]}^{(1)}\left({\vx}_{i}(t)\right)\left(\frac{\overline{\mW}^{[l+1]}(t)}{\sqrt{m}}\right)^\T\vlambda_i^{[l+1]}(t),  
 \vsigma_{[l]}^{(1)}\left({\vx}_{j}(t)\right)\left(\frac{\overline{\mW}^{[l+1]}(t)}{\sqrt{m}}\right)^\T\vlambda_j^{[l+1]}(t)
 \right>
 \\
&~~ -
\left< 
 \vsigma_{[l]}^{(1)}\left({\vx}_{i}(t)\right)\left(\frac{\overline{\mW}^{[l+1]}(0)}{\sqrt{m}}\right)^\T\vlambda_i^{[l+1]}(t),  
 \vsigma_{[l]}^{(1)}\left({\vx}_{j}(t)\right)\left(\frac{\overline{\mW}^{[l+1]}(0)}{\sqrt{m}}\right)^\T\vlambda_j^{[l+1]}(t)
 \right>\Bigg|
\\
\leq & \Norm{\left(\frac{\overline{\mW}^{[l+1]}(t)}{\sqrt{m}}-\frac{\overline{\mW}^{[l+1]}(0)}{\sqrt{m}}\right)^\T}_{2\to 2}\Norm{\vlambda_i^{[l+1]}(t)}_2 \Norm{\left(\frac{\overline{\mW}^{[l+1]}(t)}{\sqrt{m}}\right)^\T}_{2\to 2}\Norm{\vlambda_j^{[l+1]}(t)}_2\\
&~~+\Norm{\left(\frac{\overline{\mW}^{[l+1]}(0)}{\sqrt{m}}\right)^\T}_{2\to 2}\Norm{\vlambda_j^{[l+1]}(t)}_2 \Norm{\left(\frac{\overline{\mW}^{[l+1]}(t)}{\sqrt{m}}-\frac{\overline{\mW}^{[l+1]}(0)}{\sqrt{m}}\right)^\T}_{2\to 2}\Norm{\vlambda_i^{[l+1]}(t)}_2\\
\leq & 8\Norm{\frac{\overline{\mW}^{[l+1]}(t)}{\sqrt{m}}-\frac{\overline{\mW}^{[l+1]}(0)}{\sqrt{m}}}_{2\to 2}16^{L-l+1}  \leq \frac{2\sqrt{2}n\sqrt{R_\fS(\vtheta^0)}}{\kappa \lambda_\fS}64^L,
\end{align*}
and for the second term, we obtain that
\begin{align*}
 & \Bigg| 
\left< 
 \vsigma_{[l]}^{(1)}\left({\vx}_{i}(t)\right)\left(\frac{\overline{\mW}^{[l+1]}(0)}{\sqrt{m}}\right)^\T\vlambda_i^{[l+1]}(t),  
 \vsigma_{[l]}^{(1)}\left({\vx}_{j}(t)\right)\left(\frac{\overline{\mW}^{[l+1]}(0)}{\sqrt{m}}\right)^\T\vlambda_j^{[l+1]}(t)
 \right>\\
 &~~-
\left< 
 \vsigma_{[l]}^{(1)}\left({\vx}_{i}(t)\right)\left(\frac{\overline{\mW}^{[l+1]}(0)}{\sqrt{m}}\right)^\T\vlambda_i^{[l+1]}(0),  
 \vsigma_{[l]}^{(1)}\left({\vx}_{j}(t)\right)\left(\frac{\overline{\mW}^{[l+1]}(0)}{\sqrt{m}}\right)^\T\vlambda_j^{[l+1]}(0)
 \right>\Bigg|\\
 \leq & \Bigg| 
\left< 
 \left(\frac{\overline{\mW}^{[l+1]}(0)}{\sqrt{m}}\right)^\T\vlambda_i^{[l+1]}(t),  
\left(\frac{\overline{\mW}^{[l+1]}(0)}{\sqrt{m}}\right)^\T\vlambda_j^{[l+1]}(t)
 \right>\\
 &~~-
\left< 
\left(\frac{\overline{\mW}^{[l+1]}(0)}{\sqrt{m}}\right)^\T\vlambda_i^{[l+1]}(0),  
 \left(\frac{\overline{\mW}^{[l+1]}(0)}{\sqrt{m}}\right)^\T\vlambda_j^{[l+1]}(0)
 \right>\Bigg|,
\end{align*}
since the entries in  $\left(\frac{\overline{\mW}^{[l+1]}(0)}{\sqrt{m}}\right)^\T$  reads 
\[
\left(\frac{\overline{\mW}^{[l+1]}(0)}{\sqrt{m}}\right)^\T=\left[\frac{w_{p,q}}{\sqrt{m}}\right]_{m\times m},
\] 
and the entries in $\vlambda_i^{[l+1]}(t)$ and $\vlambda_j^{[l+1]}(t)$ read 
\[  \vlambda_i^{[l+1]}(t)=\left[\frac{\lambda_r^i(t)}{\sqrt{m}}\right]_{m\times 1},~~\vlambda_j^{[l+1]}(t)=\left[\frac{\lambda_r^j(t)}{\sqrt{m}}\right]_{m\times 1},\]
then
\begin{align*}
\left< 
 \left(\frac{\overline{\mW}^{[l+1]}(0)}{\sqrt{m}}\right)^\T\vlambda_i^{[l+1]}(t),  
\left(\frac{\overline{\mW}^{[l+1]}(0)}{\sqrt{m}}\right)^\T\vlambda_j^{[l+1]}(t)
 \right>=\frac{1}{m}\sum_{p,q,r=1}^m w_{p,q}\frac{\lambda_q^i(t)}{\sqrt{m}}w_{p,r}\frac{\lambda_r^j(t)}{\sqrt{m}},    
\end{align*}
and 
\begin{align*}
\left< 
 \left(\frac{\overline{\mW}^{[l+1]}(0)}{\sqrt{m}}\right)^\T\vlambda_i^{[l+1]}(0),  
\left(\frac{\overline{\mW}^{[l+1]}(0)}{\sqrt{m}}\right)^\T\vlambda_j^{[l+1]}(0)
 \right>=\frac{1}{m}\sum_{p,q,r=1}^m w_{p,q}\frac{\lambda_q^i(0)}{\sqrt{m}}w_{p,r}\frac{\lambda_r^j(0)}{\sqrt{m}}.   
\end{align*}
If $q\neq r$, then $w_{p,q}, w_{p,r}$ are independent with each other  with expectation  zero. Therefore, we   focus on the case where $q=r,$ and with high probability, the coefficients of $\frac{\lambda_q^i(t)}{\sqrt{m}}\frac{\lambda_q^j(t)}{\sqrt{m}}$ converges to  \[\frac{1}{m}\sum_{p=1}^m w_{p,q}^2\to 1, \] 
 therefore,  with high probability,
\begin{align*}
 & \Bigg| 
\left< 
 \vsigma_{[l]}^{(1)}\left({\vx}_{i}(t)\right)\left(\frac{\overline{\mW}^{[l+1]}(0)}{\sqrt{m}}\right)^\T\vlambda_i^{[l+1]}(t),  
 \vsigma_{[l]}^{(1)}\left({\vx}_{j}(t)\right)\left(\frac{\overline{\mW}^{[l+1]}(0)}{\sqrt{m}}\right)^\T\vlambda_j^{[l+1]}(t)
 \right>\\
 &~~-
\left< 
 \vsigma_{[l]}^{(1)}\left({\vx}_{i}(t)\right)\left(\frac{\overline{\mW}^{[l+1]}(0)}{\sqrt{m}}\right)^\T\vlambda_i^{[l+1]}(0),  
 \vsigma_{[l]}^{(1)}\left({\vx}_{j}(t)\right)\left(\frac{\overline{\mW}^{[l+1]}(0)}{\sqrt{m}}\right)^\T\vlambda_j^{[l+1]}(0)
 \right>\Bigg|\\
 \leq & \Abs{\sum_{q=1}^m\left(\frac{\lambda_q^i(t)}{\sqrt{m}}\frac{\lambda_q^j(t)}{\sqrt{m}}-\frac{\lambda_q^i(0)}{\sqrt{m}}\frac{\lambda_q^j(0)}{\sqrt{m}}\right)} = \Abs{\overline{\mH}_{ij}^{[l+1]}   (\vtheta(t))-\overline{\mH}_{ij}^{[l+1]}   (\vtheta(0))},
\end{align*}
and for the third term,   as we set  $\eps_l:= \prod_{k=1}^l\alpha_k$,  we only need to focus on the case where $\eps_l>1$,   and since  the entries in $\vsigma^{(1)}_{[l]}(\vx_i(t))$ and $\vsigma^{(1)}_{[l]}(\vx_j(t))$ read
\[
\vsigma^{(1)}_{[l]}(\vx_i(t))=\mathrm{diag}\left([\mu_p^i(t)]_{m\times 1}\right),~~\vsigma^{(1)}_{[l]}(\vx_j(t))=\mathrm{diag}\left([\mu_p^j(t)]_{m\times 1}\right),
\]
then the third term reads
\begin{align*}
& \Bigg| 
\left< 
 \vsigma_{[l]}^{(1)}\left({\vx}_{i}(t)\right)\left(\frac{\overline{\mW}^{[l+1]}(0)}{\sqrt{m}}\right)^\T\vlambda_i^{[l+1]}(0),  
 \vsigma_{[l]}^{(1)}\left({\vx}_{j}(t)\right)\left(\frac{\overline{\mW}^{[l+1]}(0)}{\sqrt{m}}\right)^\T\vlambda_j^{[l+1]}(0)
 \right> \\
 &~~-
\left< 
 \vsigma_{[l]}^{(1)}\left({\vx}_{i}(0)\right)\left(\frac{\overline{\mW}^{[l+1]}(0)}{\sqrt{m}}\right)^\T\vlambda_i^{[l+1]}(0),  
 \vsigma_{[l]}^{(1)}\left({\vx}_{j}(0)\right)\left(\frac{\overline{\mW}^{[l+1]}(0)}{\sqrt{m}}\right)^\T\vlambda_j^{[l+1]}(0)
 \right>\Bigg|\\
=&\Abs{\sum_{p,q,r=1}^m\left( \mu_p^i(t)\frac{w_{p,q}}{\sqrt{m}}\frac{\lambda_q^i(0)}{\sqrt{m}}\mu_p^j(t)\frac{w_{p,r}}{\sqrt{m}}\frac{\lambda_r^j(0)}{\sqrt{m}}-\mu_p^i(0)\frac{w_{p,q}}{\sqrt{m}}\frac{\lambda_q^i(0)}{\sqrt{m}}\mu_p^j(0)\frac{w_{p,r}}{\sqrt{m}}\frac{\lambda_r^j(0)}{\sqrt{m}}\right)},
\end{align*}
if $q\neq r$, the quantity converges to zero 
with high probability, i.e.,
\[
\frac{1}{m}\sum_{p=1}^m\left( \mu_p^i(t)\mu_p^j(t)-\mu_p^i(0)  \mu_p^j(0) \right)w_{p,q}w_{p,r}  \frac{\lambda_q^i(0)}{\sqrt{m}} \frac{\lambda_r^j(0)}{\sqrt{m}}\to 0.
\]
Therefore,  we   focus on the case where $q=r,$
\begin{align*}
& \Abs{\frac{1}{m}\sum_{p,q=1}^m\left( \mu_p^i(t)\mu_p^j(t)-\mu_p^i(0)  \mu_p^j(0) \right)w_{p,q}^2  \frac{\lambda_q^i(0)}{\sqrt{m}} \frac{\lambda_q^j(0)}{\sqrt{m}}}\\
\leq & \frac{1}{m}\sum_{p,q=1}^m \Abs{ \frac{\lambda_q^i(0)}{\sqrt{m}} \frac{\lambda_q^j(0)}{\sqrt{m}}} w_{p,q}^2\Abs{ \mu_p^i(t)\mu_p^j(t)-\mu_p^i(0)  \mu_p^j(0)},
\end{align*}
then for any $p\in[m]$,
\[\Norm{\sum_{q=1}^m \Abs{ \frac{\lambda_q^i(0)}{\sqrt{m}} \frac{\lambda_q^j(0)}{\sqrt{m}}} w_{p,q}^2\Abs{ \mu_p^i(t)\mu_p^j(t)-\mu_p^i(0)  \mu_p^j(0)}}_{\psi}\leq \Norm{\sum_{q=1}^m w_{p,q}^2}_{\psi}=C_{\psi, m},\]
is a sub-exponential random variable, and  as  we   write   $\frac{\overline{\mW}^{[l]}(t)}{\sqrt{m}}$ into  
\[\frac{\overline{\mW}^{[l]}(t)}{\sqrt{m}}:=\begin{pmatrix}
\left(\frac{\Bar{\vw}_{l,1}(t)}{\sqrt{m}}\right)^\T \\
\left(\frac{\Bar{\vw}_{l,2}(t)}{\sqrt{m}}\right)^\T \\
\vdots\\
\left(\frac{\Bar{\vw}_{l,m}(t)}{\sqrt{m}} \right)^\T 
\end{pmatrix}. \] 
and we define the following events: For any  $i\in[n]$,   and $k'\in[m]$,
\begin{equation}\label{B-Thm-proof...eq...Events-AR-BR-l<L}
\begin{aligned} 
\sA_{i,l,k'}(R_{k'}):=\Bigg\{&\Abs{\left<\frac{\Bar{\vw}_{l,k'}(t)}{\sqrt{m}}, \bar{\vx}_{i}^{[l-1]}(t)\right>-\left<\frac{\Bar{\vw}_{l,k'}(0)}{\sqrt{m}},\bar{\vx}_{i}^{[l-1]}(0)\right>}=R_{k'} ,\\
&~\mathbf{1}_{\left<\frac{\Bar{\vw}_{l,k'}(t)}{\sqrt{m}}, ~\bar{\vx}_{i}^{[l-1]}(t)\right>>0}\neq \mathbf{1}_{\left<\frac{\Bar{\vw}_{l,k'}(0)}{\sqrt{m}},~\bar{\vx}_{i}^{[l-1]}(0)\right>>0}~~~~~~~~~\Bigg\}.
\end{aligned}
\end{equation}
As we notice that
\begin{align*}
&\Exp\left[ \frac{1}{m}\sum_{p,q=1}^m \Abs{ \frac{\lambda_q^i(0)}{\sqrt{m}} \frac{\lambda_q^j(0)}{\sqrt{m}}} w_{p,q}^2\Abs{ \mu_p^i(t)\mu_p^j(t)-\mu_p^i(0)  \mu_p^j(0)}\right]\\
\leq&2\Abs{b-a}\frac{\Abs{\overline{\mH}_{ij}^{[l+1]}(\vtheta(0))}}{m}\sum_{p=1}^m\Exp\left[ w_{p,q}^2\right]\Exp\left[\Abs{\mathbf{1}_{\left<\frac{\Bar{\vw}_{l,p}(t)}{\sqrt{m}}, ~\bar{\vx}_{i}^{[l-1]}(t)\right>>0}- \mathbf{1}_{\left<\frac{\Bar{\vw}_{l,p}(0)}{\sqrt{m}},~\bar{\vx}_{i}^{[l-1]}(0)\right>>0}}\right]\\
\leq&16^{L-l+1}\frac{2\Abs{b-a}}{m}\sum_{p=1}^m\Prob\left(\sA_{i,l,p}(R_{p})\right)\leq \frac{2\Abs{b-a}}{m}\sum_{p=1}^mR_{p}\sqrt{m}16^{L-l+1},
\end{align*}
we omit the term $\Abs{b-a}$  for simplicity, and as the following estimate holds,
\begin{align*}
& \frac{1}{m}\sum_{p=1}^mR_{p}\sqrt{m}=\frac{1}{m}\sum_{p=1}^m\sqrt{m}\Abs{\left<\frac{\Bar{\vw}_{l,p}(t)}{\sqrt{m}}, \bar{\vx}_{i}^{[l-1]}(t)\right>-\left<\frac{\Bar{\vw}_{l,p}(0)}{\sqrt{m}},\bar{\vx}_{i}^{[l-1]}(0)\right>}\\
\leq & \frac{1}{m}\sum_{p=1}^m \sqrt{m}\left(\Norm{\frac{\Bar{\vw}_{l,p}(t)}{\sqrt{m}}- \frac{\Bar{\vw}_{l,p}(0)}{\sqrt{m}}}_2  \Norm{\bar{\vx}_{i}^{[l-1]}(t)}_2+\Abs{ \left<\frac{\Bar{\vw}_{l,p}(0)}{\sqrt{m}}, \bar{\vx}_{i}^{[l-1]}(t)-\bar{\vx}_{i}^{[l-1]}(0)\right>}\right)\\
\leq & 8^{l-1} \left(\Norm{\frac{\overline{\mW}^{[l]}(t)}{\sqrt{m}}-\frac{\overline{\mW}^{[l]}(0)}{\sqrt{m}}}_{\mathrm{F}}+\Norm{\bar{\vx}_{i}^{[l-1]}(t)-\bar{\vx}_{i}^{[l-1]}(0)}_2\right)\\
\leq& 8^{l-1}\left(\frac{2\sqrt{2}n\sqrt{R_\fS(\vtheta^0)}}{\kappa \lambda_\fS}4^L+\frac{2\sqrt{2}n\sqrt{R_\fS(\vtheta^0)}}{\kappa \lambda_\fS}\left(8^L\right)^{l-1}\right)\leq \frac{2\sqrt{2}n\sqrt{R_\fS(\vtheta^0)}}{\kappa \lambda_\fS}\left(64^L\right)^{l-1},
\end{align*}
then with high probability, the following holds  
\begin{align*}
& \Bigg| 
\left< 
 \vsigma_{[l]}^{(1)}\left({\vx}_{i}(t)\right)\left(\frac{\overline{\mW}^{[l+1]}(0)}{\sqrt{m}}\right)^\T\vlambda_i^{[l+1]}(0),  
 \vsigma_{[l]}^{(1)}\left({\vx}_{j}(t)\right)\left(\frac{\overline{\mW}^{[l+1]}(0)}{\sqrt{m}}\right)^\T\vlambda_j^{[l+1]}(0)
 \right> \\
 &~~-
\left< 
 \vsigma_{[l]}^{(1)}\left({\vx}_{i}(0)\right)\left(\frac{\overline{\mW}^{[l+1]}(0)}{\sqrt{m}}\right)^\T\vlambda_i^{[l+1]}(0),  
 \vsigma_{[l]}^{(1)}\left({\vx}_{j}(0)\right)\left(\frac{\overline{\mW}^{[l+1]}(0)}{\sqrt{m}}\right)^\T\vlambda_j^{[l+1]}(0)
 \right>\Bigg|\\
\leq &\frac{2}{m}\sum_{p=1}^mR_{p}\sqrt{m}16^{L-l+1}\leq \frac{4\sqrt{2}n\sqrt{R_\fS(\vtheta^0)}}{\kappa \lambda_\fS}\left(64^L\right)^{l-1}16^{L-l+1}\leq \frac{2\sqrt{2}n\sqrt{R_\fS(\vtheta^0)}}{\kappa \lambda_\fS}1024^{L},
\end{align*}
to sum up,  we obtain that for  any time $t\in[0, t^*)$,
\begin{align*}
 \Abs{\overline{\mH}_{ij}^{[l]}   (\vtheta(t))-\overline{\mH}_{ij}^{[l]}   (\vtheta(0))} 
&\leq  \frac{2\sqrt{2}n\sqrt{R_\fS(\vtheta^0)}}{\kappa \lambda_\fS}64^L+\Abs{\overline{\mH}_{ij}^{[l+1]}   (\vtheta(t))-\overline{\mH}_{ij}^{[l+1]}   (\vtheta(0))} \\
&~~+\frac{2\sqrt{2}n\sqrt{R_\fS(\vtheta^0)}}{\kappa \lambda_\fS}1024^{L}  \\
&\leq  \frac{2\sqrt{2}n\sqrt{R_\fS(\vtheta^0)}}{\kappa \lambda_\fS}64^L+\frac{2\sqrt{2}n\sqrt{R_\fS(\vtheta^0)}}{\kappa \lambda_\fS}\left(1024^{L}\right)^{2L-l-1}\\
&~~+\frac{2\sqrt{2}n\sqrt{R_\fS(\vtheta^0)}}{\kappa \lambda_\fS}1024^{L}  \\
&\leq \frac{2\sqrt{2}n\sqrt{R_\fS(\vtheta^0)}}{\kappa \lambda_\fS}\left(1024^{L}\right)^{2L-l}.
\end{align*}
(iii). Based on the relations \eqref{B-Thm-proof...eq...l<L+1-th-2-Norm} and 
\eqref{B-Thm-proof...eq...l<L-th-Infinity-Norm}, we obtain that for any $l=L+1$,
\begin{equation*}
\begin{aligned}
&\Abs{\overline{\mG}_{ij}^{[L+1]}   (\vtheta(t))-\overline{\mG}_{ij}^{[L+1]}   (\vtheta(0))}\\
\leq& \Norm{\bar{\vx}_{i}^{[L]}(t)-\bar{\vx}_{i}^{[L]}(0)}_2\Norm{\bar{\vx}_{j}^{[L]}(t)}_2+\Norm{\bar{\vx}_{i}^{[L]}(0)}_2\Norm{\bar{\vx}_{j}^{[L]}(t) -\bar{\vx}_{j}^{[L]}(0)}_2\\ 
\leq & \frac{2\sqrt{2}n\sqrt{R_\fS(\vtheta^0)}}{\kappa \lambda_\fS}\left(8^L\right)^L4^L+\frac{2\sqrt{2}n\sqrt{R_\fS(\vtheta^0)}}{\kappa \lambda_\fS}\left(8^L\right)^L4^L \leq  \frac{2\sqrt{2}n\sqrt{R_\fS(\vtheta^0)}}{\kappa \lambda_\fS}\left(64^L\right)^L, 
\end{aligned}
\end{equation*}
and for any $l\in[L]$,
\begin{equation*}
\begin{aligned}
&\Abs{\overline{\mG}_{ij}^{[l]}   (\vtheta(t))-\overline{\mG}_{ij}^{[l]}   (\vtheta(0))}\\
 \leq &\Abs{\overline{\mH}_{ij}^{[l]}(\vtheta(t))-\overline{\mH}_{ij}^{[l]}   (\vtheta(0))}\Norm{\bar{\vx}_{i}^{[l-1]}(t)}_2\Norm{\bar{\vx}_{j}^{[l-1]}(t)}_2\\
&+\Abs{ \overline{\mH}_{ij}^{[l]}   (\vtheta(0))}\Abs{\left< \bar{\vx}_{i}^{[l-1]}(t),  \bar{\vx}_{j}^{[l-1]}(t)  \right>-\left< \bar{\vx}_{i}^{[l-1]}(0),  \bar{\vx}_{j}^{[l-1]}(0)  \right>}\\
 \leq&\frac{2\sqrt{2}n\sqrt{R_\fS(\vtheta^0)}}{\kappa \lambda_\fS}\left(1024^{L}\right)^{2L-l}16^{l-1}+16^{L-l+1}\frac{4\sqrt{2}n\sqrt{R_\fS(\vtheta^0)}}{\kappa \lambda_\fS}\left(8^L\right)^{l-1}4^{l-1}\\
\leq & \frac{2\sqrt{2}n\sqrt{R_\fS(\vtheta^0)}}{\kappa \lambda_\fS}\left(\left(1024^2\right)^{L}\right)^{L}.
\end{aligned}
\end{equation*}
To sum up,  for any $l\in[L+1]$ and   time $t\in[0, t^*)$,
\begin{equation}
\Norm{\overline{\mG}^{[l]}   (\vtheta(t))-\overline{\mG}^{[l]}   (\vtheta(0))}_{\infty}  \leq\frac{2\sqrt{2}n\sqrt{R_\fS(\vtheta^0)}}{\kappa \lambda_\fS}\left(\left(1024^2\right)^{L}\right)^{L},  
\end{equation}
then for  any time $t\in[0, t^*)$,   
\begin{align*}
&\Norm{\mG(\vtheta(t)) - \mG\left(\vtheta(0)\right)}_\mathrm{F} \\
\leq& \Norm{\sum_{l=1}^{L+1} \left(\mG^{[l]} (\vtheta(t))-\mG^{[l]} (\vtheta(0))\right)}_\mathrm{F} \leq \sum_{l=1}^{L+1}\Norm{ \mG^{[l]} (\vtheta(t))-\mG^{[l]} (\vtheta(0))}_\mathrm{F}\\
=&\sum_{l=1}^{L+1}\frac{\kappa^2}{\alpha_l^2}\Norm{\overline{\mG}^{[l]}   (\vtheta(t))-\overline{\mG}^{[l]}   (\vtheta(0))}_\mathrm{F} \leq \sum_{l=1}^{L+1}\frac{\kappa^2}{\alpha_l^2} n\Norm{\overline{\mG}^{[l]}   (\vtheta(t))-\overline{\mG}^{[l]}   (\vtheta(0))}_{\infty}\\
\leq &\sum_{l=1}^{L+1}\frac{\kappa^2}{\alpha_l^2}\frac{2\sqrt{2}n^2\sqrt{R_\fS(\vtheta^0)}}{\kappa \lambda_\fS}\left(\left(1024^2\right)^{L}\right)^{L}.
\end{align*}
As we   choose  $m$ large enough, i.e., if
\[
m\geq \left(\left(\left(1024^2\right)^{L}\right)^{L}{8\sqrt{2}\sqrt{R_\fS(\vtheta^0)}}\left(\frac{n}{\lambda_\fS}\right)^2\right)^{\frac{1}{\frac{L+1}{2}-\sum_{k=1}^{L+1}\gamma_k}},
\]
then 
\[
 \frac{2\sqrt{2}n^2\sqrt{R_\fS(\vtheta^0)}}{\kappa \lambda_\fS}\left(\left(1024^2\right)^{L}\right)^{L}\leq \frac{\lambda_\fS}{4},
\]
hence  for  any time $t\in[0, t^*)$,    we have
\begin{equation}
\Norm{\mG(\vtheta(t)) - \mG\left(\vtheta(0)\right)}_\mathrm{F} \leq \Norm{\sum_{l=1}^{L+1} \left(\mG^{[l]} (\vtheta(t))-\mG^{[l]} (\vtheta(0))\right)}_\mathrm{F} \leq \left(\sum_{l=1}^{L+1}\frac{\kappa^2}{\alpha_l^2}\right)  \frac{\lambda_\fS}{4}.    
\end{equation}
Suppose we have $t^*<+\infty$, then by sending  $t\to t^*$  would  lead to  contradiction with the definition of $t^*$. Therefore $t^*=+\infty$, and  with high probability, for any time $t\geq 0$,
\begin{equation}\label{B-Thm-proof...eq...Loss-all-time-Decay}
R_\fS(\vtheta(t)) \leq \exp\left(- \frac{1}{n}\left[\left(\sum_{l=1}^{L+1} \frac{\kappa^2}{\alpha_l^2}\right) {\lambda_\fS}\right]t\right)R_\fS(\vtheta(0)).
\end{equation} 

We only need to prove relation \eqref{Thm...eq...Theta-Lazy...Part-One}. Based on  Proposition \ref{A-prop..Upper-Bound-and-Lower-Bound-Initial-Parameter},   for any   $l\in[2:L]$, and time $t\geq 0$,
\begin{align*}
\frac{\Norm{\Bar{\vtheta}_{\mW^{[1]}}(t)-\Bar{\vtheta}_{\mW^{[1]}}(0)}_2}{\Norm{\Bar{\vtheta}_{\mW^{[1]}}(0)}_2}&=\frac{\Norm{\frac{\overline{\mW}^{[1]}(t)}{\sqrt{m}}-\frac{\overline{\mW}^{[1]}(0)}{\sqrt{m}}}_{\mathrm{F}}}{\Norm{\frac{\overline{\mW}^{[1]}(0)}{\sqrt{m}}}_{\mathrm{F}}}  \leq\sqrt{\frac{2}{d}}\frac{2\sqrt{2}n\sqrt{R_\fS(\vtheta^0)}}{\kappa \lambda_\fS}4^L,\\
\frac{\Norm{\Bar{\vtheta}_{\mW^{[l]}}(t)-\Bar{\vtheta}_{\mW^{[l]}}(0)}_2}{\Norm{\Bar{\vtheta}_{\mW^{[l]}}(0)}_2}&=\frac{\Norm{\frac{\overline{\mW}^{[l]}(t)}{\sqrt{m}}-\frac{\overline{\mW}^{[l]}(0)}{\sqrt{m}}}_{\mathrm{F}}}{\Norm{\frac{\overline{\mW}^{[l]}(0)}{\sqrt{m}}}_{\mathrm{F}}} \leq\sqrt{\frac{2}{m}}\frac{2\sqrt{2}n\sqrt{R_\fS(\vtheta^0)}}{\kappa \lambda_\fS}4^L,\\
\frac{\Norm{\Bar{\vtheta}_{\va}(t)-\Bar{\vtheta}_{\va}(0)}_2}{\Norm{\Bar{\vtheta}_{\va}(0)}_2}&=\frac{\Norm{\frac{\bar{\va}(t)}{\sqrt{m}}-\frac{\bar{\va}(0)}{\sqrt{m}}}_2}{\Norm{\frac{\bar{\va}(0)}{\sqrt{m}}}_2}  \leq\sqrt{2}\frac{2\sqrt{2}n\sqrt{R_\fS(\vtheta^0)}}{\kappa \lambda_\fS}4^L,
\end{align*}
as $\sum_{k=1}^{L+1}\gamma_k<\frac{L+1}{2}$, then for any time $t\geq 0$,
\begin{align*}
 \frac{\Norm{\Bar{\vtheta}_{\mW^{[1]}}(t)-\Bar{\vtheta}_{\mW^{[1]}}(0)}_2}{\Norm{\Bar{\vtheta}_{\mW^{[1]}}(0)}_2}&\lesssim  \frac{1}{m^{\frac{L+1}{2}-\sum_{k=1}^{L+1}\gamma_k}}\frac{4n}{\lambda_\fS}\sqrt{\frac{R_\fS(\vtheta^0)}{d}}4^L,\\
 \frac{\Norm{\Bar{\vtheta}_{\mW^{[l]}}(t)-\Bar{\vtheta}_{\mW^{[l]}}(0)}_2}{\Norm{\Bar{\vtheta}_{\mW^{[l]}}(0)}_2}&\lesssim   \frac{1}{m^{\frac{L+2}{2}-\sum_{k=1}^{L+1}\gamma_k}}\frac{4n\sqrt{{R_\fS(\vtheta^0)}}}{\lambda_\fS}
 4^L,\\
\frac{\Norm{\Bar{\vtheta}_{\va}(t)-\Bar{\vtheta}_{\va}(0)}_2}{\Norm{\Bar{\vtheta}_{\va}(0)}_2}&\lesssim   \frac{1}{m^{\frac{L+1}{2}-\sum_{k=1}^{L+1}\gamma_k}}\frac{4n\sqrt{{R_\fS(\vtheta^0)}}}{\lambda_\fS}4^L,
\end{align*}
which finishes the proof. 

\end{document}